\documentclass{article}

\PassOptionsToPackage{numbers,sort,compress}{natbib}

    \usepackage[final]{neurips_2025}

\usepackage[utf8]{inputenc}
\usepackage[T1]{fontenc}

\usepackage[hidelinks]{hyperref} 
\usepackage{url}
\usepackage{booktabs} 
\usepackage{amsfonts}
\usepackage{amsmath}
\usepackage{amssymb}
\usepackage{mathtools}
\usepackage{microtype}   
\usepackage{subcaption}
\usepackage{amsthm}
\usepackage[capitalize,noabbrev]{cleveref}
\usepackage{xcolor}
\usepackage{bbm}
\usepackage{stackrel}
\usepackage{csquotes}
\usepackage{enumitem}
\usepackage[splitrule]{footmisc}
\usepackage{microtype}
\usepackage{graphicx}
\usepackage{booktabs} 
\usepackage{caption}

\usepackage{makecell, tabularx, multirow, xcolor}
\usepackage{colortbl}
\usepackage{etoc}
\usepackage{titletoc}
\usepackage{adjustbox}
\usepackage{wrapfig}
\usepackage{etoolbox}
\usepackage{siunitx}

\usepackage[noend]{algpseudocode}
\usepackage{algorithm}

\definecolor{darkorange25512714}{RGB}{255,127,14}
\definecolor{steelblue31119180}{RGB}{31,119,180}

\setlist[itemize]{leftmargin=*,itemsep=0.25pt,topsep=0.25pt,partopsep=0pt}
\setlist[enumerate]{leftmargin=*,itemsep=0.25pt,topsep=0.25pt,partopsep=0pt}

\newcommand{\R}{\mathbb{R}}
\newcommand{\N}{\mathbb{N}}

\newcommand{\Z}{\mathcal{Z}}

\newcommand{\E}{\mathbbm{E}}
\renewcommand{\P}{\mathbbm{P}}
\newcommand{\Q}{\mathbbm{Q}}
\newcommand{\V}{\operatorname{Var}}
\newcommand{\Var}{\operatorname{Var}}
\newcommand{\KL}{D_{\mathrm{KL}}}

\newcommand{\VV}{V}

\DeclareMathOperator*{\argmin}{arg\,min}
\DeclareMathOperator*{\argmax}{arg\,max}
\makeatletter
\DeclareRobustCommand{\cev}[1]{%
  {\mathpalette\do@cev{#1}}%
}
\newcommand{\do@cev}[2]{%
  \vbox{\offinterlineskip
    \sbox\z@{$\m@th#1 x$}%
    \ialign{##\cr
      \hidewidth\reflectbox{$\m@th#1\vec{}\mkern4mu$}\hidewidth\cr
      \noalign{\kern-\ht\z@}
      $\m@th#1#2$\cr
    }%
  }%
}

\makeatother

\newcommand{\dd}{\mathrm{d}}

\theoremstyle{plain}
\newtheorem{theorem}{Theorem}[section]
\newtheorem{proposition}[theorem]{Proposition}
\newtheorem{lemma}[theorem]{Lemma}
\newtheorem{corollary}[theorem]{Corollary}
\theoremstyle{definition}
\newtheorem{definition}[theorem]{Definition}

\newtheorem{remark}[theorem]{Remark}

\usepackage{comment}
\usepackage[color=white,disable]{todonotes}

\newcommand{\highprio}[1]{}
\newcommand{\lowprio}[1]{}

\usepackage{caption}
\renewcommand{\paragraph}[1]{\textbf{#1}}
\crefname{section}{Sec.}{Secs.}
\crefname{subsection}{Sec.}{Secs.}
\crefname{proposition}{Prop.}{Props.}
\crefname{figure}{Fig.}{Figs.}
\crefname{appendix}{App.}{Apps.}
\crefname{table}{Tab.}{Tabs.}
\crefname{corollary}{Cor.}{Cors.}
\crefname{theorem}{Thm.}{Thms.}
\crefname{equation}{Eq.}{Eqs.}

\makeatletter

\newcommand{\@setequationsize}[1]{%
  \AtBeginEnvironment{equation}{#1}%
  \AtBeginEnvironment{equation*}{#1}%
  \AtBeginEnvironment{align}{#1}%
  \AtBeginEnvironment{align*}{#1}%
  \AtBeginEnvironment{gather}{#1}%
  \AtBeginEnvironment{gather*}{#1}%
  \AtBeginEnvironment{multline}{#1}%
  \AtBeginEnvironment{multline*}{#1}%
  \AtBeginEnvironment{flalign}{#1}%
  \AtBeginEnvironment{flalign*}{#1}%
  \AtBeginEnvironment{displaymath}{#1}
}

\newcommand{\MakeAllEquationsSmall}{%
  \@setequationsize{\small}%
}
\newcommand{\MakeAllEquationsNormal}{%
  \@setequationsize{\normalsize}%
}

\MakeAllEquationsSmall

\newcommand{\zerodisplayskips}{%
  \setlength{\abovedisplayskip}{1mm}%
  \setlength{\belowdisplayskip}{1mm}%
  \setlength{\abovedisplayshortskip}{0.5mm}%
  \setlength{\belowdisplayshortskip}{0.5mm}}
\appto{\normalsize}{\zerodisplayskips}
\appto{\small}{\zerodisplayskips}
\appto{\footnotesize}{\zerodisplayskips}

\def\section{\@startsection{section}{1}{\z@}{-0.05in}{0.02in}
             {\large\bf\raggedright}}
\def\subsection{\@startsection{subsection}{2}{\z@}{-0.05in}{0.01in}
                {\normalsize\bf\raggedright}}
\def\subsubsection{\@startsection{subsubsection}{3}{\z@}{-0.05in}{0.01in}
                {\normalsize\sc\raggedright}}
\makeatother

\title{Trust Region Constrained Measure Transport in Path Space for Stochastic Optimal Control and Inference}

\makeatletter
\renewcommand*{\@fnsymbol}[1]{%
  \ifcase#1\or \dagger\or \ddagger\or
    \mathsection\or \mathparagraph\or \|\or **\or \dagger\dagger \or \ddagger\ddagger \else\@ctrerr\fi}
\makeatother

\author{%
Denis Blessing\thanks{Correspondence to \texttt{denis.blessing@kit.edu}. \quad $^*$Equal contribution.}\textsuperscript{\,\,$,1$} \ 
    Julius Berner\textsuperscript{$*,2$} \ Lorenz Richter\textsuperscript{$*,3,4$} \ Carles Domingo-Enrich\textsuperscript{$*,5$} \\
    \textbf{Yuanqi Du\textsuperscript{$6$} \ Arash Vahdat\textsuperscript{$2$} \ Gerhard Neumann\textsuperscript{$1$}} \vspace{0.2cm} \\
    $^{1}$Karlsruhe Institute of Technology \ 
    $^{2}$NVIDIA \
    $^{3}$Zuse Institute Berlin \
    $^{4}$dida \\
    $^{5}$Microsoft Research New England \
    $^{6}$Cornell University
}

\begin{document}

\maketitle

\begin{abstract}
Solving stochastic optimal control problems with quadratic control costs can be viewed as approximating a target path space measure, e.g. via gradient-based optimization. In practice, however, this optimization is challenging in particular if the target measure differs substantially from the prior. In this work, we therefore approach the problem by iteratively solving constrained problems incorporating trust regions that aim for approaching the target measure gradually in a systematic way. It turns out that this trust region based strategy can be understood as a geometric annealing from the prior to the target measure, where, however, the incorporated trust regions lead to a principled and educated way of choosing the time steps in the annealing path. We demonstrate in multiple optimal control applications that our novel method can improve performance significantly, including tasks in diffusion-based sampling, transition path sampling, and fine-tuning of diffusion models.
\end{abstract}%
\section{Introduction}
\lowprio{Commenting this out since the NeurIPS email said we should not deviate from the submitted abstract. \\[1em]
JB: Many problems in the context of stochastic optimal control (SOC) can be cast as approximating a target distribution starting from an initial base distribution, e.g., via gradient-based optimization. However, this optimization is challenging if the target differs substantially from the base distribution. In this work, we propose trust region methods as a principled way of defining well-behaved geometric annealings between the base and target distributions. In particular, the trust region bound of the constrained optimization problem directly translates to a bound on the desired variance of the importance weights among adjacent steps. If the target distribution is a path space measure, we show that these constrained optimization problems lead to a sequence of intermediate SOC problems that can be efficiently solved using buffers and versions of SOC matching or log-variance divergences. We demonstrate that our novel method can significantly improve performance across a wide range of applications, ranging from classical SOC and sampling problems, to transition path sampling of molecular systems and reward fine-tuning of text-to-image diffusion models.
\\[1em]}

Even though the theory of stochastic optimal control (SOC) dates back several decades \cite{bellman57,fleming1975deterministic}, it has recently attracted renewed interest within the machine learning community. Building on novel formulations that are well-suited for gradient-based optimization (see~\cite{domingo2024taxonomy} for an overview) and drawing connections to diffusion models~\cite{de2021diffusion,berner2022optimal,pavon2022local}, recent work has led to significant progress in the numerical approximation of high-dimensional control problems using neural networks \cite{nuesken2021solving,domingoenrich2024stochastic}. Related problems are crucial in many practical applications, ranging from sampling problems (e.g., in  statistical physics~\cite{henin2022enhanced,faulkner2024sampling}, Bayesian statistics~\cite{neal1993probabilistic,gelman2013bayesian}, and reinforcement learning~\cite{celik2025dime}) to fine-tuning of diffusion models~\cite{didi2023framework,domingoenrich2025adjoint,venkatraman2024amortizing}.
In this work, we aim to further advance SOC approximation methods by taking inspiration from trust region methods used in optimization~\cite{peters2010relative,abdolmaleki2015model,schulman2015trust,thalmeier2020adaptive,otto2021differentiable}, resulting in a principled framework from the perspective of measure transport in path space.

\paragraph{Stochastic optimal control.} SOC problems (with quadratic control costs) describe optimization problems of the form
\begin{align}
\label{eq:soc_intro}
    \min_{u \in \mathcal{U}}\  \E \left[ \int_0^T \left( \tfrac{1}{2} \|u\|^2 + f \right)(X^u_s,s) \, \dd s + g(X^u_T)\right] \ \ \text{with} \ \  \begin{cases} \dd X^u_s = \left(b + \sigma u\right)(X^u_s,s) \dd s + \sigma(s) \dd W_s \\ X_0 \sim p_0, \end{cases}
\end{align}
where one optimizes the control $u$ of the stochastic differential equation (SDE).
Since the law of the SDE solution $X^u$
induces a so-called \emph{path measure} $\P^u$ on the space of continuous trajectories (specifying how likely a certain trajectory is), finding the optimal control is equivalent to finding an optimal target path space measure $\Q$. From the SOC literature it is known that the likelihood of $\Q$ w.r.t.\@ $\P^u$ can be expressed in closed-form (see \cite{dai1996connections} and~\eqref{eq:opt_change_of_measure} below), which allows to minimize divergences\footnote{Note that the cost functional~\eqref{eq:soc_intro} corresponds (up to the normalizing constant) to the reverse Kullback-Leibler (KL) divergence $D=\KL$.} $D(\P^u,\Q)$ via gradient-based optimization (also termed \textit{iterative diffusion optimization}).

\paragraph{Trust region methods.} However, if the target $\Q$ is rather different from the initialization $\P^{u_0}$ (typically the uncontrolled process with $u_0=\mathbf{0}$), many algorithms face challenges with high variances or mode discovery when directly minimizing $D(\P^u,\Q)$, especially in high dimensions. To this end, we propose to approach the target measure gradually by a sequence $(\P^{u_i})_i$, where in the $i$-th step we add the constraint $\KL(\P^u | \P^{{u_{i-1}}}) \le \varepsilon$ to the cost functional~\eqref{eq:soc_intro}, with $u_{i-1}$ being the approximated optimal control from the previous iteration and $\varepsilon > 0$ a chosen trust region bound. We prove that the intermediate measures $\P^{u_i}$ define a geometric annealing between the prior $\P^{u_0}$ and target measure $\Q$, where the annealing step-sizes are chosen optimally, in the sense of having an approximately constant change in Fisher-Rao distance (\Cref{prop: Optimal change of measure as geometric annealing,prop:properties_of_annealing}). 
Finding an optimal annealing schedule is paramount for the convergence speed of many measure transport and sampling methods~\cite{syed2024optimised},
and understanding physical processes~\cite{salamon1983thermodynamic,crooks2007measuring}. While the direct computation of Fisher-Rao distances can be challenging, we show that trust region methods lead to a simple way of obtaining equidistant steps in an information-geometric sense.
Moreover, we show that the Lagrangian of the constrained problem can be written as another SOC problem and that the optimal Lagrangian multiplier can be obtained via a dual optimization problem without additional computational overhead (\Cref{sec: Constrained Stochastic Optimal Control}). Finally, we adapt successful approaches based on SOC matching~\cite{domingoenrich2024stochastic,domingoenrich2025adjoint} and log-variance divergences~\cite{nuesken2021solving} to the constrained SOC problem to get a practical algorithm (\Cref{sec:learning_constrained_control}). 

\paragraph{Applications.} The resulting \emph{trust region stochastic optimal control} method can be viewed as an extension of various existing algorithms, yielding significant improvements on a range of applications (\Cref{sec:applications}). In particular, we consider (i) deep learning approaches to classical SOC problems (extending~\cite{nuesken2021solving,domingoenrich2024stochastic}) enabling the usage of cross-entropy losses in high dimensions, (ii) diffusion-based sampling from unnormalized densities (extending~\cite{vargas2023denoising,richter2023improved}) enabling efficient sampling from high-dimensional, multimodal densities with substantially fewer target evaluations, (iii) transition path sampling in molecular dynamics (extending~\cite{holdijk2022path,seong2024transition}) yielding notably higher transition hit rates, and (iv) reward fine-tuning of text-to-image models (extending~\cite{domingoenrich2025adjoint}) achieving comparable performance while requiring significantly fewer simulations.

\paragraph{Contributions.} Our contributions can be summarized as follows:
\begin{itemize}
    \item We develop a general framework for solving measure transport with trust regions and apply it to SOC problems using iterative diffusion optimization. 
    \item We prove that our framework leads to a sequence of SOC problems whose solutions define an equispaced annealing between initialization and optimum w.r.t.\@ the Fisher-Rao distance.
    \item Relying on different loss functionals, we propose two practical instantiations of our framework and 
    demonstrate state-of-the-art performance
    on a series of applications, ranging from sampling from unnormalized densities to transition path sampling and reward fine-tuning of text-to-image models.
\end{itemize}

\paragraph{Notation.}
We denote by $\mathcal{U} \subset C(\R^d \times [0,T] ; \R^d)$ the set of admissible controls and by $\mathcal{P}$ the set of all probability measures on 
$C([0,T],\R^d)$. We define the path space measure $\P \in \mathcal{P}$ as the law of a $\R^d$-valued stochastic process $X = (X_t)_{t\in[0,T]}$ 
and we denote by $\P_s$ the marginal distribution at time $s$.
 We refer to~\Cref{app:assumptions_and_auxiliary} for further details on our notation and assumptions.

\section{Trust region constrained measure transport for optimal control}
\label{sec:trust_region_soc}
\highprio{JB: it would be much cleaner to write it for general measures, so we do not need to argue that $\mathcal{U}$ needs to be expressive enough to contain the solutions.}
The idea of \textit{iterative diffusion optimization} in optimal control based on path space measures is to consider loss functionals of the form
\begin{equation}
\label{eq: loss via divergence}
    \mathcal{L}(u) = D(\P^u, \Q)
\end{equation}
and minimize them with gradient-descent algorithms \cite{nuesken2021solving}. The loss functional \eqref{eq: loss via divergence} yields implementable algorithms for SOC problems since the optimal path measure $\Q$ of~\eqref{eq:soc_intro} can be stated explicitly via the Radon-Nikodym derivative
\begin{equation}
\label{eq:opt_change_of_measure}
    \frac{\dd \Q}{\dd \P}(X) = \frac{e^{-\mathcal{W}(X,0)}}{\Z(X_0)} \quad \text{with}\quad \mathcal{W}(X,t) = \int_t^T f(X_s,s) \, \dd s + g(X_T),
\end{equation}
where $\Z \coloneqq \E \big[e^{-\mathcal{W}(X,0)}|X_0\big]$ and $\P$ is the path measure of the uncontrolled process $X = X^\mathbf{0}$; see~\Cref{app:soc}. In this work, we extend this attempt by using trust regions that shall make sure that the optimization is conducted in a more ``regulated'' fashion, where the essential idea is to divide the global problem into smaller (reasonably chosen) chunks. We quantify this in~\Cref{prop:properties_of_annealing} below.
 To this end, we consider the iterative optimization scheme defined by
\begin{equation}
\label{eq: constrained optimization}
    u_{i+1} = \argmin_{u \in \mathcal{U}} \KL\left(\P^u | \Q \right) \quad \text{s.t.} \quad \KL(\P^u | \P^{{u_i}}) \le \varepsilon,
\end{equation}
for any $i \in \N$,
where $\varepsilon > 0$ defines a trust region w.r.t.\@ to the previous control iterate and where we often set $u_0 = \mathbf{0}$ (and thus $\P^{u_0}=\P$).
This corresponds to dividing the overall optimization problem into parts according to their distance measured in the KL divergence between the respective preceding and succeeding path measures. 
Due to the convexity of the KL divergence, we can show that in all but the last step we actually have an equality constraint in~\eqref{eq: constrained optimization}; see~\Cref{app:solution_to_tr_opt}. Thus, there exists an $I \in \N$ such that $u_I = u^*$ is the optimal control of the global control problem defined in \eqref{eq:soc_intro}.

\begin{remark}[Controlling the variance of importance weights]
\label{rem: controlling the variance}
    The constraint $\KL(\P^u | \P^{{u_i}}) \le \varepsilon$ can be motivated by the goal to control the variance of importance weights $\Var_{\P^{u_i}}\left(\dd \P^{u_{i+1}} / \dd \P^{u_i}\right)$, which can be explained by the inequality $\Var_{\P^{u_i}}\left( \dd \P^{u_{i+1}} /\dd \P^{u_{i}} \right) \ge e^{\KL(\P^{u_{i+1}} | \P^{u_{i}})}-1$, see, e.g., \cite{hartmann2024nonasymptotic}. For small $\varepsilon$ 
    (which is a common choice in practice)
    we typically observe $\Var_{\P^{u_i}}\left(\dd \P^{u_{i+1}} / \dd \P^{u_i} \right) \approx 2 \varepsilon$ (see \Cref{app:diffusion_exp}), which can be explained by a Taylor expansion and assuming that $\dd \P^{u_{i+1}}/ \dd \P^{u_i} \approx 1$. 
    Low variance of importance weights is directly related to efficiency of many measure transport methods and too high variance makes it practically impossible to obtain reliable results.
    Note also that the reverse KL divergence allows for explicit expressions for the resulting constrained problem (see \Cref{sec: Constrained Stochastic Optimal Control}) and we leave alternative divergences for future research. 
\end{remark}

\begin{wrapfigure}[18]{r}[0pt]{0.35\textwidth} 
        \centering
        \vspace*{-1.5em}
        \centering
        \includegraphics[width=0.35\textwidth,trim={0 0 0 10pt}, clip]{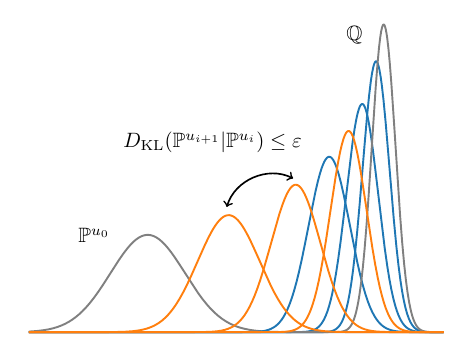}
        \vspace*{-1.9em}
        \caption{Illustration of a sequence of distributions $(\P^{u_i})_i$ resulting from our trust region method (orange) and a measure transport corresponding to non-equispaced geometric annealing (blue), leading to high variance in the importance weights for the initial steps.}
        \vspace{-0.4em}
        \label{fig:smoothed_ess}
\end{wrapfigure}

 In practice, under suitable regularity assumptions, we can approach the above constrained optimization problem using a relaxed Lagrangian formalism. To this end, we consider the loss functionals
\begin{equation}
\label{eq: Lagrangian loss}
    \mathcal{L}^{(i)}_\mathrm{TR}(u, \lambda) = \KL\left(\P^u | \Q \right) + \lambda\left( \KL(\P^u | \P^{u_i} ) - \varepsilon \right),
\end{equation}
where $\lambda \geq 0$ is a Lagrange multiplier, and solve the saddle point problems
\begin{equation}
\label{eq: Langrangian optimization}
  \max_{\lambda \ge 0} \, \min_{u \in\mathcal{U}}   \mathcal{L}_{\mathrm{TR}}^{(i)}(u, \lambda).
 \end{equation}
We note that $\mathcal{L}_\mathrm{TR}^{(i)}$ is convex in $u$ by convexity of the KL divergence (see~\Cref{app:solution_to_tr_opt}) and concave in $\lambda$ since it can be expressed as the pointwise minimum $\min_{u} \mathcal{L}_{\mathrm{TR}}^{(i)}(u, \lambda)$ among a family of linear functions of $\lambda$. 
Thus, \eqref{eq: Langrangian optimization} has unique optima which we denote by $u_{i+1}$ and $\lambda_{i}$, respectively. We can now show the following evolution of the optimal measures.

\begin{proposition}[Optimal change of measure as geometric annealing]
\label{prop: Optimal change of measure as geometric annealing}
Let $\Q$ be the optimal path measure defined in \eqref{eq:opt_change_of_measure}. The 
intermediate optimal path measures corresponding to \eqref{eq: constrained optimization} then satisfy\footnote{For notational convenience we assume an $X_0$-independent normalizing constant here and hereafter, which is possible whenever the optimal tilting of the initial density $p_0$ is known, cf. \Cref{app: initial value dependence Z}.}
   \begin{equation}
    \frac{\mathrm d \P^{u_{i+1}}}{\mathrm d \P^{u_i}} \propto \left(\frac{\mathrm d \Q}{\mathrm d \P^{u_i}}\right)^{\frac{1}{1 + \lambda_i}}
   \end{equation}
   and the optimal change of measure w.r.t.\@ the base measure $\P$ is given by\footnote{As usual, the empty product is defined as $1$ such that $\beta_0=0$.}
   \begin{equation}
   \label{eq: geometric annealing path measures}
    \frac{\dd {\P}^{u_{i}}}{\dd{\mathbb{P}}}(X) \propto \left(\frac{\dd \Q}{\dd{\mathbb{P}}}(X)\right)^{\beta_i}\left(\frac{\dd \P^{u_0}}{\dd{\mathbb{P}}}(X)\right)^{1-\beta_i} \quad \text{with} \quad \beta_i = 1-\prod_{j=0}^{i-1}\tfrac{\lambda_j}{1+\lambda_j}.
    \end{equation}
\end{proposition}
\begin{proof}\vspace{-0.8em}
    The first statement follows by the definition of the Lagrangian and the second follows by induction; see \Cref{app: proofs}.
\vspace{-0.75em}\end{proof}

Note that the sequence $(\beta_i)_i$ is monotonically increasing with values in $[0, 1]$, where 
we 
have $\beta_0 = 0$
and $\beta_{I} = 1$ (as $\lambda_{I-1}=0$ due to optimality). Thus, the formula in \eqref{eq: geometric annealing path measures} can be seen as a geometric annealing from the prior to the target measure. Note that when $u_0 = \mathbf{0}$, the second factor vanishes. 
Importantly, the step-size of the annealing is automatically chosen such that we obtain a well-behaved sequence of distributions; see also~\Cref{fig:smoothed_ess}.

\begin{proposition}[Equidistant steps on statistical manifold]
\label{prop:properties_of_annealing}
Up to higher order terms in $\varepsilon$, the sequence of measures $\P^{u_i}$, $i \in \{0, \dots, I-1\}$, are equispaced in the Fisher-Rao distance.
\end{proposition}
\begin{proof}\vspace{-0.8em}
    By \Cref{prop: Optimal change of measure as geometric annealing}, we obtain $\varepsilon = \KL(\P^{u_{i+1}}|\P^{u_i}) = \frac{\Delta_{i}^2}{2} \mathcal{I}(\beta_{i}) + O(\Delta_{i}^3 )$, where $\Delta_{i} = \beta_{i+1} - \beta_i$ and $\mathcal{I}(\beta_{i})$ is the Fisher information. The Fisher-Rao distance between $\P^{u_i}$ and $\P^{u_{i+1}}$ is then given by $\int_{\beta_i}^{\beta_{i+1}} \sqrt{\mathcal{I}(\tau)} \,\dd\tau = \sqrt{\mathcal{I}(\beta_i)} \Delta_{i} + O(\Delta_{i}^2) = \sqrt{2 \varepsilon} + O(\Delta_{i}^{3/2})$; see~\Cref{app:tr_limit} for details. 
\vspace{-0.2em}\end{proof}

\begin{remark}[Trust regions for general measures]
\label{rem: trust regions for general measures}
    The observant reader has likely noticed that so far all our arguments do not rely on the fact that we consider path space measures, but work for general probability measures. We could therefore as well write our trust region method stated in \eqref{eq: constrained optimization} as
    \begin{equation} \label{eq:p_i_plus_one_p_i}
    \P_{i+1} = \argmin_{\P \in \mathcal{P}} \KL\left(\P | \Q \right) \quad \text{s.t.} \quad \KL(\P | \P_i) \le \varepsilon.
\end{equation}
We refer to~\Cref{app:trust_region_densities} for a treatment when the measures admit densities on $\R^d$, which can, e.g., be considered for variational inference with normalizing flows.
\end{remark}

\subsection{Constrained stochastic optimal control}
\label{sec: Constrained Stochastic Optimal Control}

While the above formulation in principle works for arbitrary measures, in this work we focus on path space measures corresponding to optimal control problems. In this setting we can compute some of the objectives more explicitly and recover helpful relations. 

\paragraph{Lagrangian as SOC problem.}
First, note that, using the Girsanov theorem (see~\Cref{sec: technical assumptions}), it turns out that, for a fixed Lagrange multiplier $\lambda$, the Lagrangian in~\eqref{eq: Lagrangian loss} defines another SOC problem, i.e., 
\begin{equation}
\label{eq: connection_lagrangian_soc}
    \mathcal{L}_{\mathrm{TR}}^{(i)}(u, \lambda) = \mathcal{L}^{(i)}_{\mathrm{TRC}}(u, \lambda)- \lambda \varepsilon,
\end{equation}
where\footnote{The SOC problem is slightly more general than~\eqref{eq:soc_intro} due to the shift in the quadratic cost.}
\begin{align}
\label{eq: soc_tr}
    \mathcal{L}^{(i)}_{\mathrm{TRC}}(u, \lambda) =\E\left[ \int_0^T \left( \tfrac{1+\lambda}{2} \|u-\tfrac{\lambda}{1+\lambda}u_i\|^2 + \tfrac{\lambda}{2(1+\lambda)}\|u_i\|^2 +  f \right)(X^u_s,s) \, \dd s+ g(X^u_T) + \log \Z(X_0) \right]
\end{align}
and $X^u$ is still defined as in \eqref{eq:soc_intro}; see~\Cref{app:lagragian} for details. Note that this cost functional is more general than the one stated in \eqref{eq:soc_intro}, which one recovers when setting $\lambda = 0$. 
We can show that the corresponding SOC problem satisfies the following optimality conditions.

\begin{proposition}[Optimality for trust region SOC problems]
\label{prop: tr_soc_optimality}
For fixed $\lambda$, let us define by 
\begin{equation*}
    V^\lambda_{i+1}(x,t) \coloneqq \inf_{u \in \mathcal{U}} \E\left[ \int_t^T \left( \tfrac{1+\lambda}{2} \|u-\tfrac{\lambda}{1+\lambda}u_i\|^2 + \tfrac{\lambda}{2(1+\lambda)}\|u_i\|^2 +  f \right)(X^u_s,s) \, \dd s+ g(X^u_T) \Bigg| X_t=x \right]
\end{equation*}
the value function of the SOC problem $\inf_{u \in \mathcal{U}} \mathcal{L}^{(i)}_{\mathrm{TRC}}(u, \lambda)$ corresponding to~\eqref{eq: soc_tr} and by $u^\lambda_{i+1}$ its solution. Then it holds
\begin{enumerate}[label=(\roman*)]
    \item \textit{(Estimator for value function)} $V^{\lambda}_{i+1}(x,t)  = -(1+\lambda) \log \E\Big[e^{-\tfrac{1}{1+\lambda}\mathcal{W}_i(X^{u_i},t)}\Big|X^{u_i}_t=x\Big]$,
    \begin{equation*}
      \text{\normalsize where} \quad  \mathcal{W}_i(X^{u_i},t) = \int_t^T \tfrac{1}{2}\|u_i(X^{u_i}_s,s)\|^2\dd s + \int_t^T u_i(X^{u_i}_s,s) \cdot \dd W_s + \mathcal{W}(X^{u_i},t).
    \end{equation*}
    \item (Connection between solution and value function) It holds $u^\lambda_{i+1} = \frac{\lambda}{1 + \lambda} u_i - \frac{1}{1 + \lambda} \sigma^\top \nabla V^\lambda_{i+1}$.
\end{enumerate}
\end{proposition}

\begin{proof}\vspace{-0.6em} The statements can be proven using the verification theorem;
see~\Cref{app:lagragian} for details.
\vspace{-0.1em}\end{proof}

We note that~\Cref{prop: Optimal change of measure as geometric annealing}, the Girsanov theorem, and~\eqref{eq:opt_change_of_measure} relate the functional $\mathcal{W}_i$ in~\Cref{prop: tr_soc_optimality} to the importance weights 
\begin{align}
\label{eq:rnd_adjacent}
    \frac{\mathrm d \Q}{\mathrm d \P^{u_i}}(X^{u_i}) 
    \propto e^{-\mathcal{W}_i(X^{u_i},0)} \quad \text{and} \quad \frac{\dd \P^{u_{i+1}}}{\dd \P^{u_i}}(X^{u_i}) 
    \propto e^{-\tfrac{1}{1+\lambda_i}\mathcal{W}_i(X^{u_i},0)}.
\end{align}

\paragraph{Dual problem for Lagrange multiplier.} Next, we will outline how to find the optimal Lagrange multiplier $\lambda$ in~\eqref{eq: Langrangian optimization} in the SOC setting.
Plugging the optimal control $u^{\lambda}_{i+1}$ in the Lagrangian \eqref{eq: connection_lagrangian_soc} yields the dual function $\mathcal{L}_{\mathrm{Dual}}^{(i)} \in C(\R,\R)$ given by
\begin{align} \label{eq: lagrangian}
    \mathcal{L}_{\mathrm{Dual}}^{(i)}(\lambda) := \mathcal{L}_{\mathrm{TR}}^{(i)}(u^{\lambda}_{i+1}, \lambda) & = \mathcal{L}^{(i)}_{\mathrm{TRC}}(u^{\lambda}_{i+1})- \lambda \varepsilon.
\end{align}
We note that evaluating the SOC problem in~\eqref{eq: soc_tr} at the optimal control can be expressed via the value function given in~\Cref{prop: tr_soc_optimality}, which yields 
\begin{align}
\label{eq:dual}
   \mathcal{L}_{\mathrm{Dual}}^{(i)}(\lambda) =  
   \E\left[V^{\lambda}_{i+1}(X^{u_i}_0, 0)\right] - \lambda \varepsilon &=  -(1+\lambda) \E\left[\log \E\left[e^{-\tfrac{1}{1+\lambda}\mathcal{W}_i(X^{u_i},0)}\Big|X^{u_i}_0\right]\right]  - \lambda \varepsilon,
\end{align}
where we note that the expression in the expectation is proportional to the importance weights in~\eqref{eq:rnd_adjacent}.
Note that we can obtain a Monte Carlo estimate of the dual function using only simulations $X^{u_i}$ from the previous iterations.
As it turns out, these simulations are in most cases already required when learning the control $u_{i+1}$ and we can thus store them in a buffer. We can then obtain $\lambda_{i} = \argmax_{\lambda \in \R^+} \mathcal{L}^{(i)}_{\mathrm{Dual}}(\lambda)$ using any non-linear solver with minimal computational overhead.

In theory, we can define $u_{i+1}= u^{\lambda_{i}}_{i+1}$ using the representations in~\Cref{prop: tr_soc_optimality} and proceed with the next iteration of our trust region method in~\eqref{eq: constrained optimization}.
However, computing the optimal control $u_{i+1}$ using the representations in~\Cref{prop: tr_soc_optimality} requires gradients and Monte Carlo estimators of the value functions.
This is problematic since it relies on a large amount of samples \emph{for each state} $x$ due to the (typically) very high variance of the estimator; see~\Cref{app:broader_impact_limitations} for details. Thus, we propose versions of iterative diffusion optimization to learn parametrized approximations to $u_{i+1}$ in the next section.

\begin{figure}[t]
\vspace{-0.5em}
\begin{algorithm}[H]
\caption{\small Trust Region SOC with buffer (see~\Cref{app:implementation} for details)}
\label{alg:tr sampler}
\small
\begin{algorithmic}
\Require Initial path measure $\P^{u_0}$, target path measure $\Q$, divergence $D$, 
termination threshold $\delta$
\For{$i=0,1,\dots$}
\highprio{JB: write $X^{u_i}$}
\State Sample trajectories $X \sim \P^{u_i}$ by integrating the SDE in~\eqref{eq:soc_intro} with Brownian motion $W$ and control $u_i$
\State Compute importance weights $w=\frac{\dd\Q}{\dd\P^{u_i}}(X^{u_i}) \propto \exp(-\mathcal{W}_i(X^{u_i},0))$ as in~\eqref{eq:rnd_adjacent}
\State Initialize buffer $\mathcal{B} = \big\{ W, X, w \big\}$ 
\State Compute multiplier $\lambda_i = \argmax_{\lambda \in \R^+} \ \mathcal{L}_{\mathrm{Dual}}^{(i)}(\lambda)$ as in~\eqref{eq:dual} using $\mathcal{B}$ and a $1$-dim. non-linear solver
\State Compute $u^{i+1} = \argmin_u D(\P^u, \P^{u_{i+1}})$ using $\mathcal{B}$ and $\frac{\dd \P^{u_{i+1}}}{\dd \P^{u}} \propto w^{\frac{1}{1 + \lambda_i}} \frac{\dd \P^{u_i}}{\dd \P^{u}} $ as in~\Cref{sec:learning_constrained_control}
\vspace{0.1cm}
\If{$\lambda_i \le \delta$}
\State \Return control $u_{i+1}$ with $\P^{u_{i+1}} \approx \Q$
\EndIf
\EndFor
\end{algorithmic}
\end{algorithm}
\vspace{-2.8em}
\end{figure}

\subsection{Learning the constrained optimal control}
\label{sec:learning_constrained_control}
In this section we propose strategies to learn the optimal control for each iteration. As before, the general idea is to minimize loss functionals based on divergences between path space measures, namely
$
    \mathcal{L}(u) = D(\P^u, \P^{u_{i+1}})
$.
Such divergences often rely on the Radon-Nikodym derivative
\begin{align}
\begin{split}
\label{eq: rnd_tr_soc}
    \frac{\dd \P^{u_{i+1}}}{\dd \P^{u}}(X^{u_i}) &= \frac{\dd \P^{u_{i+1}}}{\dd \P^{u_i}}(X^{u_i}) \frac{\dd \P^{u_i}}{\dd \P^{u}} (X^{u_i}) \\
    &\propto \exp\Big(\int_0^T \tfrac{\| u_i - u \|^2}{2}  (X_s^{u_i}, s) \dd s + \int_0^T \big(u_i - u \big)(X_s^{u_i}, s) \cdot \dd W_s  - \tfrac{\mathcal{W}_i(X^{u_i}, 0)}{1+\lambda_i}\Big),
\end{split}
\end{align}
where we used Girsanov's theorem and~\eqref{eq:rnd_adjacent}. Note that the Radon-Nikodym derivative in~\eqref{eq: rnd_tr_soc} depends only on samples of the process with the already learned $u_i$. Let us now suggest two concrete divergences. Those divergences are desirable for high-dimensional problems since both do not rely on computing derivatives of the stochastic process and can be optimized \enquote{off-policy} using trajectories $X^{u_i}$ with the control $u_{i}$ of the previous iteration, which can be stored in a buffer; see~\Cref{alg:tr sampler}.

\textbf{Log-variance divergence.} This divergence can be considered w.r.t.\@ an arbitrary reference measure, where we choose $\P^{u_i}$ for convenience \cite{nuesken2021solving,richter2020vargrad}. We can then define the loss functional
\begin{equation}
\label{eq:trust_region_lv}
    \mathcal{L}_{\mathrm{LV}}(u) := \V
    \left[
    \log \left(\frac{\dd \P^{u_{i+1}}}{\dd \P^{u}}(X^{u_i})  \right)
    \right], 
\end{equation}
where the Radon-Nikodym derivative can be explicitly computed as in \eqref{eq: rnd_tr_soc}. 
Note that for $\lambda_{i}=0$, this loss reduces to the on-policy log-variance loss typically used in the literature~\cite{richter2023improved}. While this loss has beneficial theoretical properties \cite{nuesken2021solving}, it requires to keep the full trajectory in memory for the gradient computation. 
    
\textbf{Cross-entropy divergence and SOC matching.} 
Alternatively, we can consider the cross-entropy loss (i.e., the forward KL divergence computed using reweighting)
\begin{align} \label{eq:cross_entropy_TR}
    \mathcal{L}_{\mathrm{CE}}(u) := \KL(\P^{u_{i+1}} | \P^u) =
     \E\left[\left(\log \frac{\mathrm d \P^{u_{i+1}}}{\mathrm d \P^u}(X^{u_{i}}) \right) \frac{\mathrm d \P^{u_{i+1}}}{\mathrm d \P^{u_i}}(X^{u_{i}})\right],
\end{align}
where the Radon-Nikodym derivative is again given by \eqref{eq: rnd_tr_soc}. Contrary to the log-variance loss, the reweighting $\frac{\mathrm d \P^{u_{i+1}}}{\mathrm d \P^{u_i}}$ in~\eqref{eq:rnd_adjacent} induces exponential terms. Our trust region constraint makes sure, however, that the variance of those weights stays bounded, see~\Cref{rem: controlling the variance}.

To efficiently compute this loss, we define the so-called (\emph{lean}\footnote{Instead of the uncontrolled process $X$, we could also express the adjoint state w.r.t.\@ the process $X^u$; however, this relies on more costly vector-Jacobian products; see~\Cref{app:tr_adjoint_matching}.}) \emph{adjoint state} $a$ as in \cite{domingoenrich2025adjoint} 
via
\begin{equation}
\label{eq:lean_adjoint_state}
    \frac{\dd}{\dd s} a_{i+1}(X_s, s) = - \left[
    \left(\nabla  b(X_s,s\right)^{\top} a_{i+1}(X_s, s) + \beta_{i+1}\nabla f(X_s,s)
    \right]
\end{equation}
with $ a_{i+1}(X_T, T) = \beta_{i+1}\nabla g(X_T)$, satisfying $ a_{i+1}(X_s, s) = \nabla_{X_s}\beta_{i+1} \mathcal{W}(X,s) $; see~\cite[Lemma 5]{domingoenrich2025adjoint} and observe that it differs from the standard lean adjoint by the factor $\beta_i$ defined in~\Cref{prop: Optimal change of measure as geometric annealing}. 
Similar to~\cite{domingoenrich2024stochastic}, we can use the expression for the optimal control in~\Cref{prop: tr_soc_optimality} and the Girsanov theorem to arrive at the \emph{SOC matching loss}\footnote{The loss is similar to the SOCM-Adjoint loss in 
\cite{domingoenrich2024stochastic}, which, however, involves matrix-valued functions.}, a simple regression objective given by
\begin{equation}
\label{eq:lean_adjoint_matching}
   \mathcal{L}_{\mathrm{SOCM}}(u) \coloneqq  \E\left[\tfrac{1}{2}\int_0^T  \| \sigma^{\top} a_{i+1}(X^{u_i}_s,s) - u(X^{u_i}_s,s)\|^2 \dd s \, \frac{\mathrm d \P^{u_{i+1}}}{\mathrm d \P^{u_i}}(X^{u_i})\right],
\end{equation}
see~\Cref{app:tr_lean_adjoint_matching} for details. Contrary to the log-variance divergence above, this objective does not require to keep the whole trajectory $X^{u_i}$ in memory for backpropagation but can be computed at times $t \sim \operatorname{Unif}([0,T])$ using a Monte Carlo approximation. We summarize our algorithm in~\eqref{alg:tr sampler} and compare the different losses against existing approaches for SOC problems in the next section.

\section{Applications}
\label{sec:applications}

In this section, we explore several applications of SOC, comparing our novel trust-region-based optimization algorithm against existing methods. Specifically, we consider the three tasks sampling from unnormalized densities, transition path sampling, and fine-tuning text-to-image models. For background information, detailed experimental setups, and additional results, we refer to \Cref{app:sampling,appendix:transitionpath,appendix:finetuning}, respectively. We also include further experiments on classical SOC problems in \Cref{appendix:classicalSOC}.

\subsection{Diffusion-based sampling} \label{sec:application sampling}

\begin{figure*}[t!]
        \centering
        \begin{minipage}[t!]{\textwidth}
            \begin{minipage}[t!]{0.9\textwidth}
            \centering
            \includegraphics[width=0.8\textwidth]{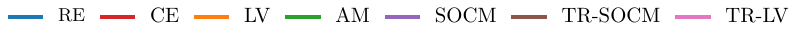}
            \end{minipage}
            \centering
            \begin{minipage}[t!]{0.24\textwidth}
            \includegraphics[width=\textwidth]{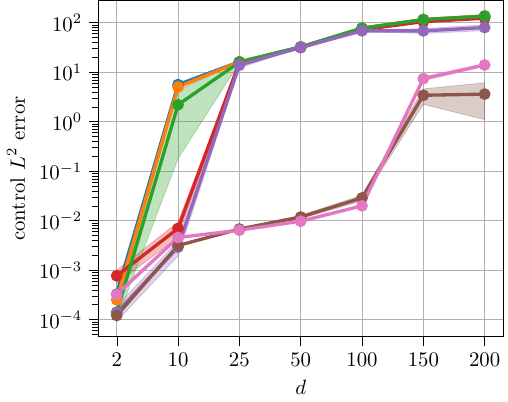}
            \end{minipage}
            \begin{minipage}[t!]{0.24\textwidth}
            \includegraphics[width=0.95\textwidth]{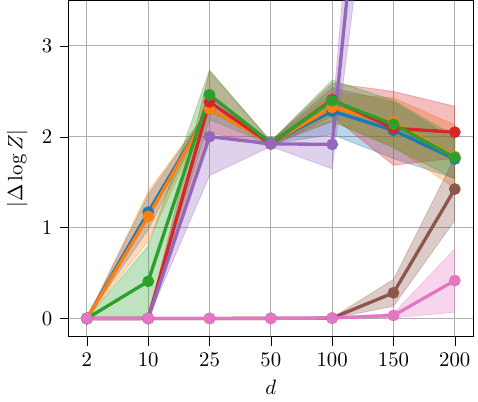}
            \end{minipage}
            \begin{minipage}[t!]{0.24\textwidth}
            \includegraphics[width=\textwidth]{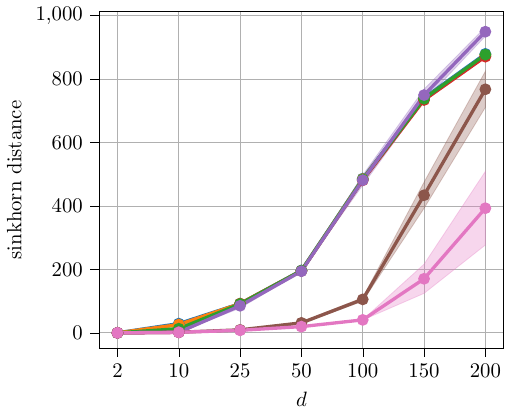}
            \end{minipage}
            \begin{minipage}[t!]{0.24\textwidth}
            \includegraphics[width=0.95\textwidth]{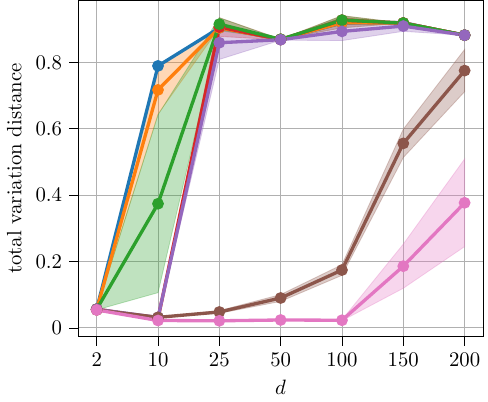}
            \end{minipage}
        \end{minipage}
        \vspace{-0.4em}
        \caption[ ]
        {
        Performance criteria for a Gaussian mixture target density with varying dimension $d$, averaged across four seeds. 
        We show the errors of estimating the optimal control, the log-normalizing constant, as well as the Sinkhorn and total variation distances over different dimensions (from left to right). We observe that our trust
region methods (TR-SOCM and TR-LV) are the only methods that perform well in high dimensions.
        }
        \vspace{-0.5em}
        \label{fig:gmm}
    \end{figure*}

Using~\eqref{eq:opt_change_of_measure}, we can show that sampling problems can be reformulated as SOC problems. To this end, we leverage the following corollary showing that the terminal distributions $\Q_T$ and $\P_T$ of the optimally controlled and uncontrolled processes differ by a tilting. 

\begin{corollary}[Sampling from tilted distributions]
\label{cor:tilting}
    Let us set $f=0$ and assume that the terminal distribution of the uncontrolled process $X$ is independent of $p_0$ and admits a density denoted by $\P_T$. Then it holds that
    $ 
        \Q_T \propto  \P_T e^{-g}.
    $ 
\end{corollary}%
\begin{proof}\vspace*{-1em}
    Using~\eqref{eq:opt_change_of_measure} it holds that
    $\frac{\dd \Q}{\dd \P}(X) = \frac{e^{-g(X_T)}}{\Z(X_0)}$ with $\Z(X_0) = \E \big[e^{-g(X_T)}|X_0\big]$. The results follows from the independence of $X_T$ and $X_0$; see~\cite{domingoenrich2025adjoint} and~\Cref{app:sampling} for details.
\vspace*{-0.7em}\end{proof}%
\Cref{cor:tilting} shows that the optimally controlled process $X^{u^*}$ samples from a given unnormalized density $\rho_\mathrm{target}$ when using an uncontrolled process with known terminal distribution $\P_T$ and setting 
$ 
    g=\log \frac{\P_T}{\rho_\mathrm{target}};
$ 
see~\cite{dai1991stochastic,tzen2019theoretical,richter2021thesis,zhang2021path,vargas2023bayesian,vargas2023denoising,zhangdiffusion,richter2023improved,vargas2024transport} and~\Cref{app:sampling} for details.
Such sampling problems are of immense practical interest, with numerous applications in the natural sciences \cite{zhang2023artificial, schopmans2025temperature}, in Bayesian statistics \cite{gelman2013bayesian}, and reinforcement learning~\cite{celik2025dime}. 

\paragraph{Numerical experiments.} 
Here, we compare existing methods for solving SOC problems with our trust region method on challenging multimodal sampling problems. We use the \emph{Denoising Diffusion Sampler} (DDS)~\cite{vargas2023denoising} method, which leverages an ergodic Ornstein–Uhlenbeck process initialized at its equilibrium measure as uncontrolled process $X$. We consider five baselines, specifically, reverse and (importance weighted) forward KL, also known as \textit{relative entropy (RE)} and \textit{cross entropy (CE)} method, respectively. Additionally, we consider the \textit{log-variance loss \cite{richter2020vargrad}}, \textit{adjoint matching (AM)} \cite{domingoenrich2025adjoint}, and \textit{stochastic optimal control matching (SOCM)} \cite{domingoenrich2024stochastic}, for the unconstrained problem in~\eqref{eq: loss via divergence}; see \cite{domingo2024taxonomy} for a comprehensive overview of SOC losses. In all experiments, we deliberately avoid using gradient guidance from the target density in the diffusion process, often referred to as Langevin preconditioning (LP) \cite{he2025no}. Prior work has shown that LP is essential for preventing mode collapse in neural samplers \cite{blessing2024beyond, he2025no}. However, LP is computationally expensive, as it requires querying the target distribution at every discretization step, making such approaches impractical for many problems where evaluating the target gradient is costly.

\begin{figure*}[t!]
    \centering
    \begin{minipage}{0.48\textwidth}
        \centering
        \resizebox{\textwidth}{!}{%
        \renewcommand{\arraystretch}{1.2}
        \begin{tabular}{l|rrrr}
        \toprule
        & \multicolumn{1}{c}{$d=2$}
        & \multicolumn{1}{c}{$d=50$}
        & \multicolumn{1}{c}{$d=100$}
        & \multicolumn{1}{c}{$d=200$}
        \\
        \midrule
        RE
        & $1.364 \scriptstyle \pm 0.002$
        & $3.443 \scriptstyle \pm 0.004$
        & $3.077 \scriptstyle \pm 0.669$
        & $2.908 \scriptstyle \pm 0.679$
        \\
        CE
        & $0.001 \scriptstyle \pm 0.000$
        & $0.202 \scriptstyle \pm 0.159$
        & $0.526 \scriptstyle \pm 0.181$
        & $0.641 \scriptstyle \pm 0.527$
        \\
        LV
        & \multicolumn{1}{c}{\textit{diverged}}
        & $1.363 \scriptstyle \pm 0.325$
        & $1.809 \scriptstyle \pm 0.737$
        & $1.958 \scriptstyle \pm 0.698$
        \\
        AM 
        & $1.364 \scriptstyle \pm 0.002$
        & $3.432 \scriptstyle \pm 0.020$
        & $3.457 \scriptstyle \pm 0.019$
        & $3.322 \scriptstyle \pm 0.307$
        \\
        SOCM
        & $0.001 \scriptstyle \pm 0.000$
        & $2.958 \scriptstyle \pm 0.831$
        & $2.971 \scriptstyle \pm 0.846$
        & $3.504 \scriptstyle \pm 0.005$
        \\
        TR-LV
        & $\mathbf{0.000 \scriptstyle \pm 0.000}$
        & $\mathbf{0.000 \scriptstyle \pm 0.000}$
        & $\mathbf{0.002 \scriptstyle \pm 0.002}$
        & $\mathbf{0.002 \scriptstyle \pm 0.001}$
        \\
        \bottomrule
        \end{tabular}
        }
        \vspace{0.4cm}
    \end{minipage}\hfill
    \begin{minipage}{0.5\textwidth}
        \centering
        \begin{minipage}[t!]{0.49\textwidth}
        \includegraphics[width=\textwidth]{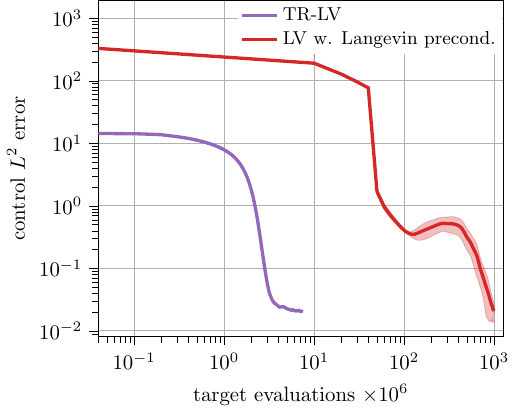}
        \end{minipage}
        \begin{minipage}[t!]{0.49\textwidth}
        \includegraphics[width=\textwidth]{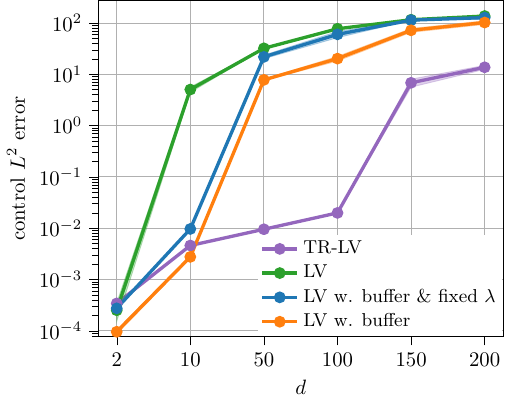}
        \end{minipage}
    \end{minipage}
    \vspace{-0.5em}
    \caption{The left table reports $|\Delta \log \mathcal{Z}|$ values for the \textit{Many Well} target across different dimensions $d$. The middle plot compares the log-variance loss of our trust region method (TR-LV) with that of Langevin preconditioning on the GMM target in dimension $d = 100$. The rightmost figure presents an ablation analysis of key components in our method, highlighting the importance of trust regions in preventing mode collapse and achieving low control error. All results are averaged across four seeds.
    }
    \label{fig:combined_manywell}
    \vspace{-1em}
\end{figure*}

First, we consider a \textit{Gaussian Mixture Model (GMM)} comprising 10 components and randomized mixing weights. GMMs are particularly compelling as they admit an analytical solution for the optimal control, which enables direct computation of the \textit{$L^2$} error between the learned and optimal controls, a reliable metric for detecting mode collapse. In addition, we assess the \textit{Sinkhorn distance} \cite{cuturi2013sinkhorn} between samples from the target and the model, and the absolute error in estimating the log-normalizing constant, denoted $|\Delta \log \mathcal{Z}|$. Finally, we evaluate the \textit{total variation distance} between the true mixing weights and the model’s estimated weights.
The results, shown in \Cref{fig:gmm}, indicate that for $d = 2$, all methods closely approximate the optimal control. However, for dimensions beyond $d = 10$, most methods suffer from mode collapse, as reflected by increased control errors, except for those employing trust region updates. Trust region methods maintain robustness across a wide range of dimensions and only begin to show signs of mode collapse in high dimensions ($d \ge 150$).

We additionally evaluate our method on the \textit{Many Well} target \cite{wu2020stochastic} with 32 modes. For quantitative analysis, we report the log-normalization error $|\Delta \log \mathcal{Z}|$, as other ground-truth quantities are unavailable. Additionally, for the high-dimensional case $d = 200$, we visualize pairs of marginal distributions in \Cref{app:sampling}. The results, presented in \Cref{fig:combined_manywell}, demonstrate that our method significantly outperforms competing approaches in estimating the normalizing constant. Furthermore, the visualizations in \Cref{app:sampling} illustrate that trust region updates effectively prevent mode collapse, even in high dimensions. In contrast, baseline methods either suffer from mode collapse or fail to converge.

Finally, we perform an ablation study on the GMM target, analyzing key components of our proposed method. Specifically, we investigate the effects of incorporating a replay buffer and applying trust region optimization. To this end, we compare a variant using a fixed Lagrangian multiplier $\lambda$, selected via hyperparameter tuning, with one in which $\lambda$ is dynamically optimized using our trust region approach. Additionally, we evaluate the log-variance loss both with and without using a replay buffer. Moreover, we compare our method to LV with Langevin preconditioning on the GMM target with dimensionality $d = 100$. The results, shown in Figure \ref{fig:combined_manywell}, demonstrate that trust region optimization significantly reduces control error and decreases the number of target evaluations by several orders of magnitude.

\subsection{Transition path sampling}
\label{sec:tps}

Transition path sampling is of great importance for studying phase transitions and chemical reactions. The key challenge comes from the energy barrier that connects two sets $A$ and $B$ along the energy landscape, which makes direct sampling of transition paths extremely unlikely. These problems can also be formulated as SOC problems~\cite{singh2025variational,hartmann2012efficient,hartmann2019variational}. Specifically, we set $b=-\nabla U$, where $U:\mathbb{R}^{N\times 3} \rightarrow \mathbb{R}$ is the potential function, and $g = -\log \boldsymbol{1}_B$ as well as $p_0\propto \boldsymbol{1}_A$, which constraints the initial and target states in the sets $A$ and $B$. As in \eqref{eq:opt_change_of_measure}, it holds that $\frac{\dd \mathbb{Q}}{\dd\mathbb{P}} = \frac{\boldsymbol{1}_B(X_T)}{\Z(X_0)}$. 
Recent work has leveraged neural networks to parameterize a bias force to solve the corresponding SOC problem, employing objectives such as the KL~\cite{holdijk2022path,yan2022learning,dudoob}, or log-variance divergence~\cite{seong2024transition}.

\paragraph{Numerical experiments.} 
We evaluate the performance of the trust-region-based log-variance loss (TR-LV) on two transition path sampling problems: Alanine Dipeptide isomerization and Chignolin folding, with 22 and 138 atoms, respectively. 

Our evaluation includes three metrics:\textit{ Kabsch-aligned root mean squared distance (RMSD)} between the final states of the sampled paths and the target state, \textit{transition hit percentage (THP)} measuring the proportion of final states hitting within the target region, and \textit{energy of transition state (ETS)} identifying the highest energy values along paths that reach the target.

We compare our method to standard molecular dynamics (MD) with increased temperature (UMD), steered MD (SMD)~\cite{izrailev1999steered} with force applied to collective variables, and PIPS~\cite{holdijk2022path} which uses the cross-entropy loss. We also include TPS-DPS~\cite{seong2024transition} as a key baseline, which employs an (unconstrained) log-variance loss to formulate TPS as a stochastic optimal control (SOC) problem. Further experimental details are provided in \Cref{appendix:transitionpath}.

Table \ref{tab:tps} shows that TR-LV achieves superior target state RMSD and transition hit percentage compared to the standard log-variance objective (TPS-DPS) for both molecular systems. Notably, SMD performs well due to its use of collective variables with biased force guiding the sampling process. Figure \ref{fig:tps} illustrates that the trust region constraint leads to significantly more robust training compared to TPS-DPS as indicated by low standard deviations across different seeds.  Moreover, on Alanine Dipeptide, the trust region constraint initially regularizes optimization and accelerates convergence thereafter. Across both systems, the trust region constraint significantly enhances training stability and performance.

\begin{table}[t]
    \centering
    \caption{Quantitative evaluation on transition path sampling problems. $\dagger$ denotes that results are taken from \cite{seong2024transition}. The results for TPS-DPS and TR-LV are averaged across three seeds.}
    \label{tab:tps}
    \resizebox{\textwidth}{!}{%
        \begin{tabular}{lcccc}
            \toprule
            \multirow{1}{*}{Method} & RMSD (\r{A}, $\downarrow$) & THP (\%, $\uparrow$) & ETS ($\mathrm{kJ/mol}$) \\
            \midrule
            \multicolumn{4}{c}{Alanine Dipeptide} \\
            \midrule
            UMD (3600K)$\dagger$ & 1.19 $\pm$ 0.32 & \phantom{0}6.25 & 812.47 $\pm$ 148.80  \\
            SMD$\dagger$ & 0.56 $\pm$ 0.27 & 54.69 & 78.40 $\pm$ 12.76  \\
            PIPS$\dagger$ & 0.66 $\pm$ 0.15 & 43.75 & 28.17 $\pm$ 10.86  \\
            TPS-DPS & $0.47 \pm 0.18$ & $39.58 \pm 28.13$ &  $46.34 \pm 10.16$ \\ 
            TR-LV & $\mathbf{0.29 \pm 0.03}$ & $\mathbf{61.25 \pm 4.05}$ & $49.11 \pm 5.84$ \\
            \bottomrule
        \end{tabular}
        \hspace{.1in}
        \begin{tabular}{lcccc}
            \toprule
            \multirow{1}{*}{Method} & RMSD (\r{A}, $\downarrow$) & THP (\%, $\uparrow$) & ETS ($\mathrm{kJ/mol}$)  \\
            \midrule
            \multicolumn{4}{c}{Chignolin} \\
            \midrule
            UMD (1200K)$\dagger$ & 7.23 $\pm$ 0.93 & 1.56 & 388.17  \\
            SMD$\dagger$ & \textbf{0.85 $\pm$ 0.24} & 34.38 & 179.52 $\pm$ 138.87  \\
            PIPS$\dagger$ &  4.66 $\pm$ 0.17 & 0.00 & --  \\
            TPS-DPS  & $1.06 \pm 0.08$ & $25.00 \pm 10.69$ & $-189.91 \pm 23.01$ \\
            TR-LV  & $0.90 \pm 0.01$ & $\mathbf{43.95 \pm 5.64}$ & $-303.98 \pm 28.65$ \\
            \bottomrule
        \end{tabular}%
    }
\vspace{-1em}
\end{table}

\begin{figure*}[t!]
        \centering
            \begin{minipage}[t!]{0.9\textwidth}
            \centering
            \includegraphics[width=0.3\textwidth]{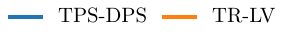}
            \end{minipage}
        \begin{minipage}[t!]{\textwidth}
            \centering
            \begin{subfigure}[t]{0.48\textwidth}
                \centering
                \begin{minipage}[t!]{0.49\textwidth}
                \includegraphics[width=\textwidth]{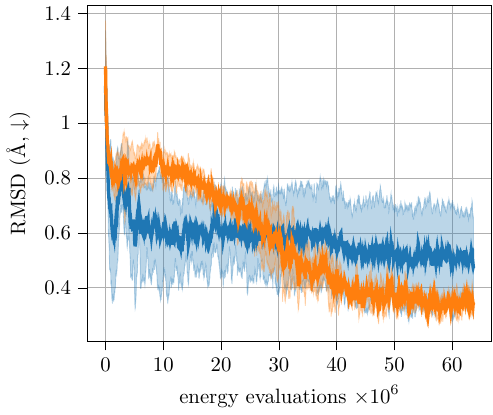}
                \end{minipage}
                \begin{minipage}[t!]{0.49\textwidth}
                \includegraphics[width=\textwidth]{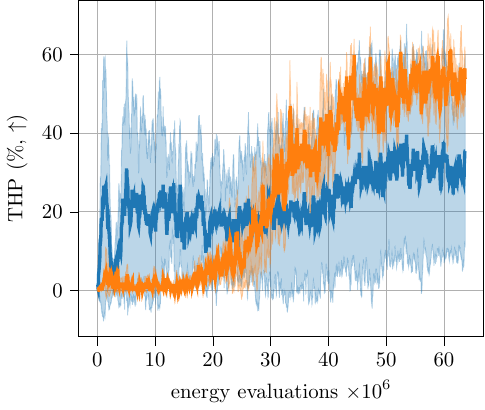}
                \end{minipage}
                \caption{Alanine Dipeptide}
            \end{subfigure}
            \begin{subfigure}[t]{0.48\textwidth}
                \centering
                \begin{minipage}[t!]{0.49\textwidth}
                \includegraphics[width=\textwidth]{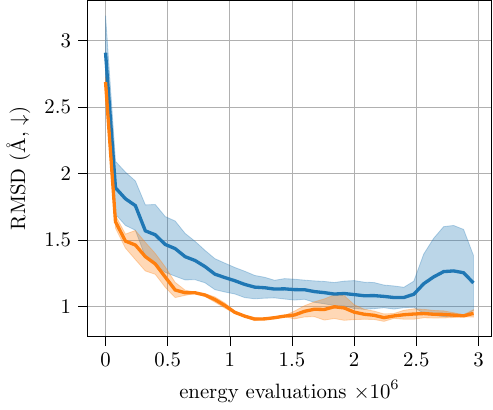}
                \end{minipage}
                \begin{minipage}[t!]{0.49\textwidth}
                \includegraphics[width=\textwidth]{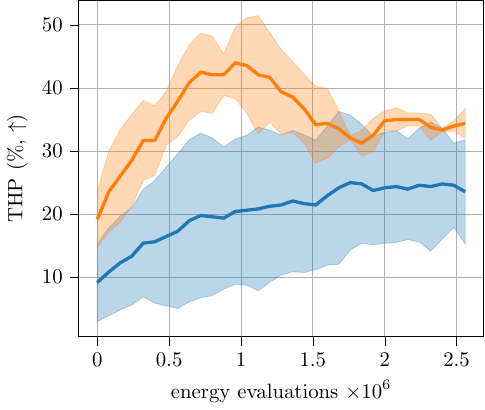}
                \end{minipage}
                \caption{Chignolin}
            \end{subfigure}
        \end{minipage}
        \vspace{-0.25em}
        \caption[ ]
        {
        We compare our trust region method (TR-LV) with Diffusion Path Sampler (TPS-DPS) \cite{seong2024transition} on Alanine Dipeptide and Chignolin. All results are averaged over three random seeds, with both the mean and standard deviation reported. Our method identifies transition paths more consistently and robustly, as evidenced by higher THP values and lower standard deviations.
}
        \vspace{-1.5em}
        \label{fig:tps}
    \end{figure*}

\subsection{Fine-tuning of diffusion models}

Interpreting $-g$ as a \emph{reward} and the uncontrolled process $X$ as a pretrained diffusion model (i.e., $b$ includes the pretrained neural network),~\Cref{cor:tilting} shows that we can perform reward fine-tuning by solving the SOC problem in~\eqref{eq:soc_intro}; see also~\cite{didi2023framework,venkatraman2024amortizing,domingoenrich2025adjoint}. Reward fine-tuning has recently shown impressive results, e.g., in image~\cite{domingoenrich2025adjoint,clark2024directly} and molecule generation~\cite{didi2023framework}, and SOC provides a principled framework. A special case is given by posterior sampling~\cite{didi2023framework}. Setting $g = -\log p(y|x)$, where $p(y|x)$ is the likelihood and we interpret $\P_T$
as a learned (\emph{diffusion}) prior $p(x)$, Bayes' theorem shows that the optimally controlled process samples from the posterior $p(x|y)$.

\paragraph{Numerical experiments.} We perform reward fine-tuning on Stable Diffusion 1.5~\cite{rombach2022high}, using ImageReward \cite{xu2023imagereward}, which is a reward model designed to capture prompt alignment and image quality according to human preferences. We take the adjoint matching (AM) method as baseline and compare it against our TR-SOCM loss \eqref{eq:lean_adjoint_matching}, keeping all other hyperparameters fixed. Our TR-SOCM allows the principled use of buffers, and we perform three passes on each buffer of size $500$, leading to three times fewer trajectories for a fixed number of model updates. For faster convergence, we use a modified version of TR-SOCM with annealing factor $\beta_i=1$. For each algorithm, we evaluate 5 checkpoints during fine-tuning (with ODE and SDE inference) on ImageReward and three additional metrics: CLIP-Score \cite{hessel2021clipscore}, which measures prompt alignment, Human Preference Score \cite{wu2023humanpreferencescorev2}, which measures human-perceived image quality, and Dreamsim diversity \cite{fu2023learning}, which measures per-prompt diversity. We observe that TR-SOCM achieves similar performance metrics to AM at a fraction of the cost, as sampling the trajectories and solving the lean adjoint ODE, which dominates the computational costs, is amortized over the buffer passes; see~\Cref{fig:finetune,fig:images_diff_finetuning} as well as \Cref{appendix:finetuning} for more details.

\begin{figure*}[t!]
        \centering
        \begin{minipage}[t!]{\textwidth}
            \begin{minipage}[t!]{0.9\textwidth}
            \centering
            \includegraphics[width=0.8\textwidth]{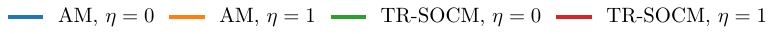}
            \end{minipage}
            \centering
            \begin{minipage}[t!]{0.24\textwidth}
            \includegraphics[width=\textwidth]{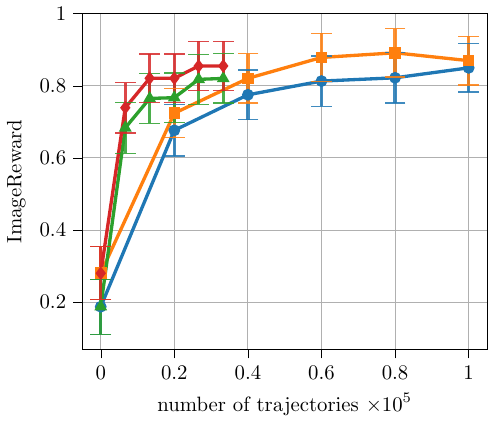}
            \end{minipage}
            \begin{minipage}[t!]{0.24\textwidth}
            \includegraphics[width=\textwidth]{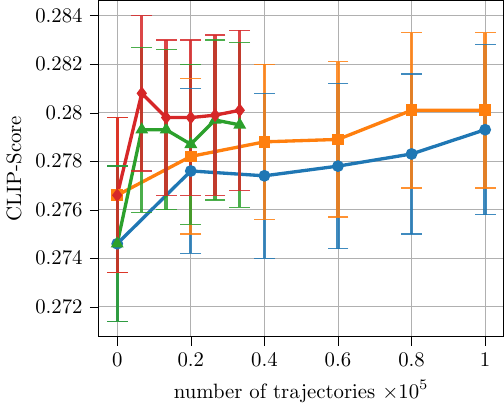}
            \end{minipage}
            \begin{minipage}[t!]{0.24\textwidth}
            \includegraphics[width=\textwidth]{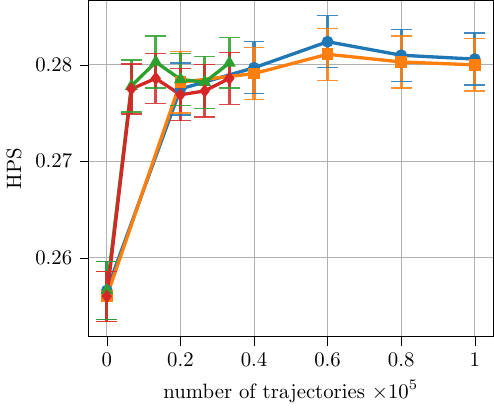}
            \end{minipage}
            \begin{minipage}[t!]{0.24\textwidth}
            \includegraphics[width=\textwidth]{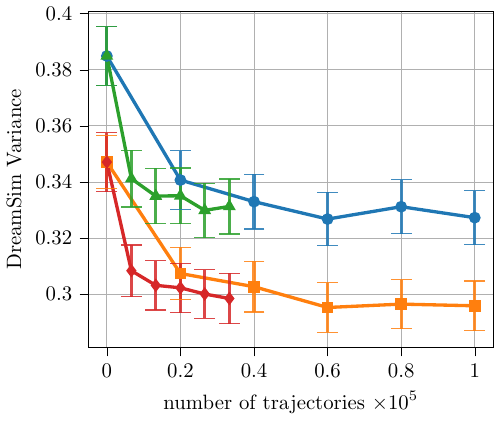}
            \end{minipage}
        \end{minipage}
        \vspace{-0.25em}
        \caption[ ]
        {Comparison of Adjoint Matching against Trust Region SOCM for Stable Diffusion 1.5 fine-tuning w.r.t.\@ four quality metrics, where $\eta = 0$ and $\eta = 1$ refer to ODE (DDIM) and SDE (DDPM) inference, respectively.
        }
        \vspace{-0.75em}
        \label{fig:finetune}
    \end{figure*}

\begin{figure}
    \centering
    \includegraphics[width=0.95\linewidth]{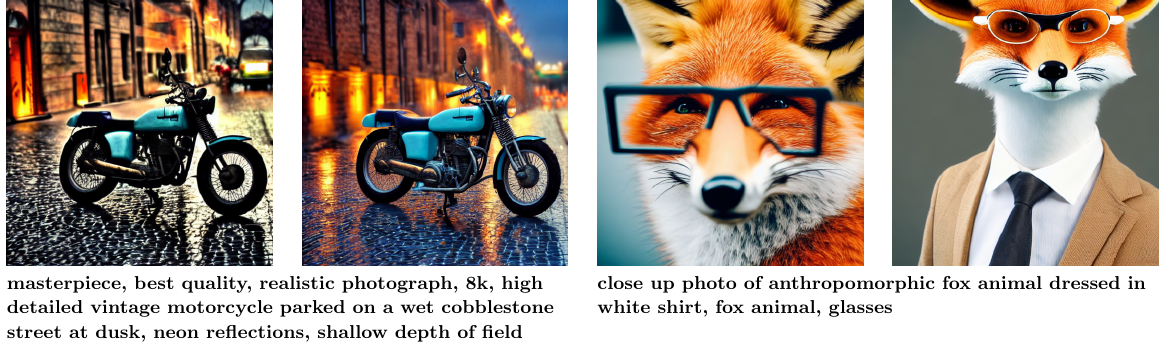}
    \caption{Comparison between images generated by the base Stable Diffusion 1.5 model (left) and its version fine-tuned with TR-SOCM (right), using the same prompts and random seeds. The fine-tuned model generates higher quality images (bike) with better prompt alignment (fox).}
    \label{fig:images_diff_finetuning}
    \vspace{-1.5em}
\end{figure}
\section{Related works} \label{sec: related works}
In this section, we discuss the most related works, comparing our approach to existing methods for solving SOC problems. We provide a more extensive comparison in~\Cref{app:broader_impact_limitations}.

\paragraph{Iterative diffusion optimization.}
Many recently developed methods approach SOC problems by simulating the (diffusion) process $X^u$, computing a suitable cost function, and optimizing the parameters of the control function $u$ using variants of stochastic gradient methods. These techniques are collectively referred to as \emph{iterative diffusion optimization} (IDO) methods \cite{nuesken2021solving}.
While the underlying theory dates back to~\cite{pavon1989stochastic,dai1991stochastic}, combinations with deep learning in the context of SOC have been explored by~\cite{nuesken2021solving,zhang2021path,richter2021thesis,zhou2021actor,vargas2023denoising,berner2022optimal,richter2023improved,vargas2024transport}.
One can derive most of the related objectives starting from the Radon-Nikodym derivative $\frac{\dd \P^{u}}{\dd \Q}(X^u)$
as in~\eqref{eq:rnd_adjacent} (with $u=u_i$).
One can then minimize a loss based on a suitable divergence as in~\eqref{eq: loss via divergence}. Previous works have, e.g., proposed the log-variance divergence~\cite{richter2023improved,seong2024transition} or the forward KL divergence
(corresponding to the cross-entropy loss \cite{hartmann2017variational,kappen2016adaptive,rubinstein2013cross,zhang2014applications,holdijk2022path}), for which we develop corresponding trust region versions in~\eqref{eq:trust_region_lv} and~\eqref{eq:cross_entropy_TR}. The SOC matching loss \cite{domingoenrich2024stochastic}, which we extended to trust regions in~\eqref{eq:lean_adjoint_matching}, is equal to the cross entropy loss in expectation but exhibits lower variance empirically.\highprio{JB: Replace this citation with the correct one after the submission and delete the bibentry} We refer to~\cite{domingoenrich2025adjoint} for more IDO losses. {However, all existing methods have either directly tackled the target measure $\Q$ 
or relied on a form of hand-tuned annealing.}

\paragraph{Trust region methods.} We show how IDO methods can generally be extended to trust region methods, enabling (1) automatic control on the variance of the importance weights and (2) principled usage of buffers, leading to faster and more stable convergence, in particular avoiding mode collapse in high dimensions. Trust region methods have a long history as robust optimization algorithms that iteratively minimize an objective within an adaptively sized \enquote{trust region}; see~\cite{conn2000trust} for an overview. These methods have also been extended to optimize over spaces of probability distributions, particularly in reinforcement learning~\cite{peters2010relative,schulman2015trust,schulman2017proximal,achiam2017constrained,pajarinen2019compatible,akrour2019projections,yang2020projection,otto2021differentiable,xu2024trust,wu2017scalable,abdolmaleki2018maximum,abdolmaleki2018relative,meng2021off,becker2025troll}, black-box optimization~\cite{sun2009efficient,wierstra2014natural,abdolmaleki2015model}, variational inference~\cite{arenz2020trust,arenz2022unified,hertrich2024importance,von2025learning} and path integral control~\cite{thalmeier2020adaptive}. To the best of our knowledge, these methods have not yet been extended to path measures or inference problems. Moreover, the connection between trust-region iterates and geometric annealing has not previously been established.

\section{Conclusion}
In this work, we develop a novel framework for solving SOC problems using deep learning. Our framework builds on the fact that we can reformulate specific problems as finding an optimal path space measure induced by a controlled SDE. Instead of finding this optimal measure at once, we divide the unconstrained problem into a sequence of constrained optimization problems by bounding the KL divergence to the measure from the previous iteration. We show that this defines a well-behaved geometric annealing between the prior and the target path measure, resulting in equidistant steps on the Fisher-Rao information manifold. Crucially, each intermediate problem turns out to be an altered SOC problem that can be efficiently solved without simulations by using a buffer of trajectories with the control from the previous iteration. In our experiments, we show that our method significantly improves the learning of the optimal control, including applications in diffusion-based sampling and transition path sampling in molecular dynamics.
Further, we show that our method can be scaled to improve the efficiency of reward fine-tuning for text-to-image diffusion models. In the future, we expect our framework to improve even more applications of SOC, potentially including the use of divergences other than the KL divergence for the trust region constraint. Finally, our results for general measures motivate the use of trust region methods for other learned measure transports, e.g., normalizing flows.

\section*{Acknowledgements}
D.B. acknowledges support by funding from a Google PhD fellowship in Machine Learning and ML Foundations, the pilot program Core Informatics of the Helmholtz Association (HGF) and the state of Baden-Württemberg through bwHPC, as well as the HoreKa supercomputer funded by the Ministry of Science, Research and the Arts Baden-Württemberg and by the German Federal Ministry of Education and Research. The research of L.R.~was partially funded by Deutsche Forschungsgemeinschaft (DFG) through the grant CRC 1114 ``Scaling Cascades in Complex Systems'' (project A05, project number $235221301$).  

\newpage

\newpage
\bibliographystyle{abbrv}
\bibliography{references}

@article{midgley2022flow,
  title={Flow annealed importance sampling bootstrap},
  author={Midgley, Laurence Illing and Stimper, Vincent and Simm, Gregor NC and Sch{\"o}lkopf, Bernhard and Hern{\'a}ndez-Lobato, Jos{\'e} Miguel},
  journal={arXiv preprint arXiv:2208.01893},
  year={2022}
}

@article{schopmans2025temperature,
  title={Temperature-Annealed Boltzmann Generators},
  author={Schopmans, Henrik and Friederich, Pascal},
  journal={arXiv preprint arXiv:2501.19077},
  year={2025}
}

@article{bolhuis2002transition,
  title={Transition path sampling: Throwing ropes over rough mountain passes, in the dark},
  author={Bolhuis, Peter G and Chandler, David and Dellago, Christoph and Geissler, Phillip L},
  journal={Annual review of physical chemistry},
  volume={53},
  number={1},
  pages={291--318},
  year={2002},
  publisher={Annual Reviews 4139 El Camino Way, PO Box 10139, Palo Alto, CA 94303-0139, USA}
}

@article{singh2023variational,
  title={Variational deep learning of equilibrium transition path ensembles},
  author={Singh, Aditya N and Limmer, David T},
  journal={The Journal of Chemical Physics},
  volume={159},
  number={2},
  year={2023},
  publisher={AIP Publishing}
}

@article{das2021reinforcement,
  title={Reinforcement learning of rare diffusive dynamics},
  author={Das, Avishek and Rose, Dominic C and Garrahan, Juan P and Limmer, David T},
  journal={The Journal of Chemical Physics},
  volume={155},
  number={13},
  year={2021},
  publisher={AIP Publishing}
}

@article{chetrite2015variational,
  title={Variational and optimal control representations of conditioned and driven processes},
  author={Chetrite, Rapha{\"e}l and Touchette, Hugo},
  journal={Journal of Statistical Mechanics: Theory and Experiment},
  volume={2015},
  number={12},
  pages={P12001},
  year={2015},
  publisher={IOP Publishing}
}

@inproceedings{dudoob2024,
  title={Doob's Lagrangian: A Sample-Efficient Variational Approach to Transition Path Sampling},
  author={Du, Yuanqi and Plainer, Michael and Brekelmans, Rob and Duan, Chenru and Noe, Frank and Gomes, Carla P and Aspuru-Guzik, Alan and Neklyudov, Kirill},
  booktitle={The Thirty-eighth Annual Conference on Neural Information Processing Systems},
  year={2024}
}

@article{brekelmans2020all,
  title={All in the exponential family: Bregman duality in thermodynamic variational inference},
  author={Brekelmans, Rob and Masrani, Vaden and Wood, Frank and Steeg, Greg Ver and Galstyan, Aram},
  journal={arXiv preprint arXiv:2007.00642},
  year={2020}
}

@article{singh2025variational,
  title={Variational path sampling of rare dynamical events},
  author={Singh, Aditya N and Das, Avishek and Limmer, David T},
  journal={Annual Review of Physical Chemistry},
  volume={76},
  year={2025},
  publisher={Annual Reviews}
}

@article{dellago1998efficient,
  title={Efficient transition path sampling: Application to Lennard-Jones cluster rearrangements},
  author={Dellago, Christoph and Bolhuis, Peter G and Chandler, David},
  journal={The Journal of chemical physics},
  volume={108},
  number={22},
  pages={9236--9245},
  year={1998},
  publisher={American Institute of Physics}
}

@article{seong2024transition,
  title={Transition Path Sampling with Improved Off-Policy Training of Diffusion Path Samplers},
  author={Seong, Kiyoung and Park, Seonghyun and Kim, Seonghwan and Kim, Woo Youn and Ahn, Sungsoo},
  journal={arXiv preprint arXiv:2405.19961},
  year={2024}
}

@article{he2025feat,
  title={{FEAT}: Free energy Estimators with Adaptive Transport},
  author={He, Jiajun and Du, Yuanqi and Vargas, Francisco and Wang, Yuanqing and Gomes, Carla P and Hern{\'a}ndez-Lobato, Jos{\'e} Miguel and Vanden-Eijnden, Eric},
  journal={arXiv preprint arXiv:2504.11516},
  year={2025}
}

@article{hartmann2012efficient,
  title={Efficient rare event simulation by optimal nonequilibrium forcing},
  author={Hartmann, Carsten and Sch{\"u}tte, Christof},
  journal={Journal of Statistical Mechanics: Theory and Experiment},
  volume={2012},
  number={11},
  pages={P11004},
  year={2012},
  publisher={IOP Publishing}
}

@inproceedings{izrailev1999steered,
  title={Steered molecular dynamics},
  author={Izrailev, Sergei and Stepaniants, Sergey and Isralewitz, Barry and Kosztin, Dorina and Lu, Hui and Molnar, Ferenc and Wriggers, Willy and Schulten, Klaus},
  booktitle={Computational Molecular Dynamics: Challenges, Methods, Ideas: Proceedings of the 2nd International Symposium on Algorithms for Macromolecular Modelling, Berlin, May 21--24, 1997},
  pages={39--65},
  year={1999},
  organization={Springer}
}

@inproceedings{holdijk2022path,
  title={Path integral stochastic optimal control for sampling transition paths},
  author={Holdijk, Lars and Du, Yuanqi and Jaini, Priyank and Hooft, Ferry and Ensing, Bernd and Welling, Max},
  booktitle={ICML 2022 2nd AI for Science Workshop},
  year={2022}
}

@inproceedings{clark2024directly,
  title={Directly Fine-Tuning Diffusion Models on Differentiable Rewards},
  author={Clark, Kevin and Vicol, Paul and Swersky, Kevin and Fleet, David J},
  booktitle={The Twelfth International Conference on Learning Representations},
  year={2024},
}

@article{hartmann2019variational,
  title={Variational approach to rare event simulation using least-squares regression},
  author={Hartmann, Carsten and Kebiri, Omar and Neureither, Lara and Richter, Lorenz},
  journal={Chaos: An Interdisciplinary Journal of Nonlinear Science},
  volume={29},
  number={6},
  year={2019},
  publisher={AIP Publishing}
}

@article{rose2021reinforcement,
  title={A reinforcement learning approach to rare trajectory sampling},
  author={Rose, Dominic C and Mair, Jamie F and Garrahan, Juan P},
  journal={New Journal of Physics},
  volume={23},
  number={1},
  pages={013013},
  year={2021},
  publisher={IOP Publishing}
}

@article{yan2022learning,
  title={Learning nonequilibrium control forces to characterize dynamical phase transitions},
  author={Yan, Jiawei and Touchette, Hugo and Rotskoff, Grant M},
  journal={Physical Review E},
  volume={105},
  number={2},
  pages={024115},
  year={2022},
  publisher={APS}
}

@article{salamon1983thermodynamic,
  title={Thermodynamic length and dissipated availability},
  author={Salamon, Peter and Berry, R Stephen},
  journal={Physical Review Letters},
  volume={51},
  number={13},
  pages={1127},
  year={1983},
  publisher={APS}
}

@book{conn2000trust,
  title={Trust region methods},
  author={Conn, Andrew R and Gould, Nicholas IM and Toint, Philippe L},
  year={2000},
  publisher={SIAM}
}

@article{crooks2007measuring,
  title={Measuring thermodynamic length},
  author={Crooks, Gavin E},
  journal={Physical Review Letters},
  volume={99},
  number={10},
  pages={100602},
  year={2007},
  publisher={APS}
}

@article{syed2024optimised,
  title={Optimised annealed sequential {M}onte {C}arlo samplers},
  author={Syed, Saifuddin and Bouchard-C{\^o}t{\'e}, Alexandre and Chern, Kevin and Doucet, Arnaud},
  journal={arXiv preprint arXiv:2408.12057},
  year={2024}
}

@inproceedings{zhangdiffusion,
  title={Diffusion Generative Flow Samplers: Improving learning signals through partial trajectory optimization},
  author={Zhang, Dinghuai and Chen, Ricky TQ and Liu, Cheng-Hao and Courville, Aaron and Bengio, Yoshua},
  booktitle={The Twelfth International Conference on Learning Representations},
  year={2024},
}

@inproceedings{dudoob,
  title={Doob's Lagrangian: A Sample-Efficient Variational Approach to Transition Path Sampling},
  author={Du, Yuanqi and Plainer, Michael and Brekelmans, Rob and Duan, Chenru and Noe, Frank and Gomes, Carla P and Aspuru-Guzik, Alan and Neklyudov, Kirill},
  booktitle={The Thirty-eighth Annual Conference on Neural Information Processing Systems}
}

@article{celik2025dime,
  title={{DIME}: Diffusion-Based Maximum Entropy Reinforcement Learning},
  author={Celik, Onur and Li, Zechu and Blessing, Denis and Li, Ge and Palanicek, Daniel and Peters, Jan and Chalvatzaki, Georgia and Neumann, Gerhard},
  journal={arXiv preprint arXiv:2502.02316},
  year={2025}
}

@article{venkatraman2024amortizing,
  title={Amortizing intractable inference in diffusion models for vision, language, and control},
  author={Venkatraman, Siddarth and Jain, Moksh and Scimeca, Luca and Kim, Minsu and Sendera, Marcin and Hasan, Mohsin and Rowe, Luke and Mittal, Sarthak and Lemos, Pablo and Bengio, Emmanuel and others},
  journal={arXiv preprint arXiv:2405.20971},
  year={2024}
}

@inproceedings{rombach2022high,
  title={High-resolution image synthesis with latent diffusion models},
  author={Rombach, Robin and Blattmann, Andreas and Lorenz, Dominik and Esser, Patrick and Ommer, Bj{\"o}rn},
  booktitle={Proceedings of the IEEE/CVF conference on computer vision and pattern recognition},
  pages={10684--10695},
  year={2022}
}

@article{didi2023framework,
  title={A framework for conditional diffusion modelling with applications in motif scaffolding for protein design},
  author={Didi, Kieran and Vargas, Francisco and Mathis, Simon V and Dutordoir, Vincent and Mathieu, Emile and Komorowska, Urszula J and Lio, Pietro},
  journal={arXiv preprint arXiv:2312.09236},
  year={2023}
}

@book{pham2009continuous,
  title={Continuous-time Stochastic Control and Optimization with Financial Applications},
  author={Pham, H.},
  series={Stochastic Modelling and Applied Probability},
  year={2009},
  publisher={Springer Berlin Heidelberg}
}

@book{fleming2006controlled,
  title={Controlled {M}arkov processes and viscosity solutions},
  author={Fleming, Wendell H and Soner, Halil Mete},
  volume={25},
  year={2006},
  publisher={Springer Science \& Business Media}
}

@article{von2025learning,
  title={Learning Boltzmann Generators via Constrained Mass Transport},
  author={von Klitzing, Christopher and Blessing, Denis and Schopmans, Henrik and Friederich, Pascal and Neumann, Gerhard},
  journal={arXiv preprint arXiv:2510.18460},
  year={2025}
}

@article{richter2020vargrad,
  title={Var{G}rad: {A} low-variance gradient estimator for variational inference},
  author={Richter, Lorenz and Boustati, Ayman and N{\"u}sken, Nikolas and Ruiz, Francisco and Akyildiz, Omer Deniz},
  journal={Advances in Neural Information Processing Systems},
  volume={33},
  pages={13481--13492},
  year={2020}
}

@article{hertrich2024importance,
  title={Importance corrected neural JKO sampling},
  author={Hertrich, Johannes and Gruhlke, Robert},
  journal={arXiv preprint arXiv:2407.20444},
  year={2024}
}

@inproceedings{
domingoenrich2025adjoint,
title={Adjoint Matching: Fine-tuning Flow and Diffusion Generative Models with Memoryless Stochastic Optimal Control},
author={Carles Domingo-Enrich and Michal Drozdzal and Brian Karrer and Ricky T. Q. Chen},
booktitle={The Thirteenth International Conference on Learning Representations},
year={2025},
}

@article{han2018solving,
  title={Solving high-dimensional partial differential equations using deep learning},
  author={Han, Jiequn and Jentzen, Arnulf and E, Weinan},
  journal={Proceedings of the National Academy of Sciences},
  volume={115},
  number={34},
  pages={8505--8510},
  year={2018},
  publisher={National Acad Sciences}
}

@article{richter2023continuous,
  title={From continuous-time formulations to discretization schemes: tensor trains and robust regression for BSDEs and parabolic PDEs},
  author={Richter, Lorenz and Sallandt, Leon and N{\"u}sken, Nikolas},
  journal={Journal of Machine Learning Research},
  volume={25},
  number={248},
  pages={1--40},
  year={2024}
}

@inproceedings{richter2021solving,
  title={Solving high-dimensional parabolic {PDE}s using the tensor train format},
  author={Richter, Lorenz and Sallandt, Leon and N{\"u}sken, Nikolas},
  booktitle={International Conference on Machine Learning},
  pages={8998--9009},
  year={2021},
  organization={PMLR}
}

@inproceedings{richter2022robust,
  title={Robust {SDE}-Based Variational Formulations for Solving Linear {PDE}s via Deep Learning},
  author={Richter, Lorenz and Berner, Julius},
  booktitle={International Conference on Machine Learning},
  pages={18649--18666},
  year={2022},
  organization={PMLR}
}

@article{hartmann2024nonasymptotic,
  title={Nonasymptotic bounds for suboptimal importance sampling},
  author={Hartmann, Carsten and Richter, Lorenz},
  journal={SIAM/ASA Journal on Uncertainty Quantification},
  volume={12},
  number={2},
  pages={309--346},
  year={2024},
  publisher={SIAM}
}

@article{berner2020numerically,
  title={Numerically solving parametric families of high-dimensional {K}olmogorov partial differential equations via deep learning},
  author={Berner, Julius and Dablander, Markus and Grohs, Philipp},
  journal={Advances in Neural Information Processing Systems},
  volume={33},
  pages={16615--16627},
  year={2020}
}

@article{beck2021solving,
  title={Solving the {K}olmogorov {PDE} by means of deep learning},
  author={Beck, Christian and Becker, Sebastian and Grohs, Philipp and Jaafari, Nor and Jentzen, Arnulf},
  journal={Journal of Scientific Computing},
  volume={88},
  pages={1--28},
  year={2021},
  publisher={Springer}
}

@article{faulkner2024sampling,
  title={Sampling algorithms in statistical physics: a guide for statistics and machine learning},
  author={Faulkner, Michael F and Livingstone, Samuel},
  journal={Statistical Science},
  volume={39},
  number={1},
  pages={137--164},
  year={2024},
  publisher={Institute of Mathematical Statistics}
}

@article{becker2025troll,
  title={Troll: Trust regions improve reinforcement learning for large language models},
  author={Becker, Philipp and Freymuth, Niklas and Thilges, Serge and Otto, Fabian and Neumann, Gerhard},
  journal={arXiv preprint arXiv:2510.03817},
  year={2025}
}

@article{henin2022enhanced,
  title={Enhanced sampling methods for molecular dynamics simulations},
  author={H{\'e}nin, J{\'e}r{\^o}me and Leli{\`e}vre, Tony and Shirts, Michael R and Valsson, Omar and Delemotte, Lucie},
  journal={arXiv preprint arXiv:2202.04164},
  year={2022}
}

@article{neal1993probabilistic,
  title={Probabilistic inference using {M}arkov chain {M}onte {C}arlo methods},
  author={Neal, Radford M},
  year={1993},
  publisher={Department of Computer Science, University of Toronto Toronto, ON, Canada}
}

@article{de2021diffusion,
  title={Diffusion {Schr{\"o}dinger} bridge with applications to score-based generative modeling},
  author={De Bortoli, Valentin and Thornton, James and Heng, Jeremy and Doucet, Arnaud},
  journal={Advances in Neural Information Processing Systems},
  volume={34},
  pages={17695--17709},
  year={2021}
}

@article{pavon2022local,
  title={On local entropy, stochastic control and deep neural networks},
  author={Pavon, Michele},
  journal={arXiv preprint arXiv:2204.13049},
  year={2022}
}

@article{tan2023noise,
  title={Noise-free sampling algorithms via regularized {W}asserstein proximals},
  author={Tan, Hong Ye and Osher, Stanley and Li, Wuchen},
  journal={arXiv preprint arXiv:2308.14945},
  year={2023}
}

@inproceedings{domingoenrich2024stochastic,
title={Stochastic Optimal Control Matching},
author={Carles Domingo-Enrich and Jiequn Han and Brandon Amos and Joan Bruna and Ricky T. Q. Chen},
booktitle={The Thirty-eighth Annual Conference on Neural Information Processing Systems},
year={2024},
}

@article{sabate2021unbiased,
  title={Unbiased deep solvers for linear parametric {PDEs}},
  author={Sabate Vidales, Marc and {\v{S}}i{\v{s}}ka, David and Szpruch, Lukasz},
  journal={Applied Mathematical Finance},
  volume={28},
  number={4},
  pages={299--329},
  year={2021},
  publisher={Taylor \& Francis}
}

@article{zhang2023artificial,
  title={Artificial intelligence for science in quantum, atomistic, and continuum systems},
  author={Zhang, Xuan and Wang, Limei and Helwig, Jacob and Luo, Youzhi and Fu, Cong and Xie, Yaochen and Liu, Meng and Lin, Yuchao and Xu, Zhao and Yan, Keqiang and others},
  journal={arXiv preprint arXiv:2307.08423},
  year={2023}
}

@article{brent1971algorithm,
  title={An algorithm with guaranteed convergence for finding a zero of a function},
  author={Brent, Richard P.},
  journal={The computer journal},
  volume={14},
  number={4},
  pages={422--425},
  year={1971},
  publisher={Oxford University Press}
}

@book{gelman2013bayesian,
  title={Bayesian Data Analysis, Third Edition},
  author={Gelman, A. and Carlin, J.B. and Stern, H.S. and Dunson, D.B. and Vehtari, A. and Rubin, D.B.},
  series={Chapman \& Hall/CRC Texts in Statistical Science},
  year={2013},
  publisher={Taylor \& Francis}
}

@inproceedings{tzen2019theoretical,
  title={Theoretical guarantees for sampling and inference in generative models with latent diffusions},
  author={Tzen, Belinda and Raginsky, Maxim},
  booktitle={Conference on Learning Theory},
  pages={3084--3114},
  year={2019},
  organization={PMLR}
}

@article{zhou2021actor,
  title={Actor-critic method for high dimensional static {Hamilton--Jacobi--Bellman} partial differential equations based on neural networks},
  author={Zhou, Mo and Han, Jiequn and Lu, Jianfeng},
  journal={SIAM Journal on Scientific Computing},
  volume={43},
  number={6},
  pages={A4043--A4066},
  year={2021},
  publisher={SIAM}
}

@phdthesis{richter2021thesis,
  title={Solving high-dimensional {PDE}s, approximation of path space measures and importance sampling of diffusions},
  author={Richter, Lorenz},
  year={2021},
  school={BTU Cottbus-Senftenberg}
}

@book{bellman57,
  title={Dynamic programming},
  author={Bellman, Richard},
  year={1957},
  publisher={Princeton University Press}
}

@article{dai1996connections,
  title={Connections between stochastic control and dynamic games},
  author={Dai Pra, Paolo and Meneghini, Lorenzo and Runggaldier, Wolfgang J},
  journal={Mathematics of Control, Signals and Systems},
  volume={9},
  pages={303--326},
  year={1996},
  publisher={Springer}
}

@book{fleming1975deterministic,
  title={Deterministic and Stochastic Optimal Control},
  author={Fleming, W.H. and Rishel, R.W.},
  series={Applications of mathematics},
  year={1975},
  publisher={Springer}
}

@article{nuesken2021solving,
  title={Solving high-dimensional {H}amilton--{J}acobi--{B}ellman {PDE}s using neural networks: perspectives from the theory of controlled diffusions and measures on path space},
  author={N{\"u}sken, Nikolas and Richter, Lorenz},
  journal={Partial differential equations and applications},
  volume={2},
  pages={1--48},
  year={2021},
  publisher={Springer}
}

@article{domingo2024taxonomy,
  title={A taxonomy of loss functions for stochastic optimal control},
  author={Domingo-Enrich, Carles},
  journal={arXiv preprint arXiv:2410.00345},
  year={2024}
}

@article{wu2020stochastic,
  title={Stochastic normalizing flows},
  author={Wu, Hao and K{\"o}hler, Jonas and No{\'e}, Frank},
  journal={Advances in neural information processing systems},
  volume={33},
  pages={5933--5944},
  year={2020}
}

@article{blessing2024beyond,
  title={Beyond {ELBO}s: A large-scale evaluation of variational methods for sampling},
  author={Blessing, Denis and Jia, Xiaogang and Esslinger, Johannes and Vargas, Francisco and Neumann, Gerhard},
  journal={arXiv preprint arXiv:2406.07423},
  year={2024}
}

@Article{hartmann2017variational,
AUTHOR = {Hartmann, Carsten and Richter, Lorenz and Schütte, Christof and Zhang, Wei},
TITLE = {Variational Characterization of Free Energy: Theory and Algorithms},
JOURNAL = {Entropy},
VOLUME = {19},
YEAR = {2017},
NUMBER = {11},
ARTICLE-NUMBER = {626},
}

@article{kappen2016adaptive,
  title={Adaptive importance sampling for control and inference},
  author={Kappen, Hilbert Johan and Ruiz, Hans Christian},
  journal={Journal of Statistical Physics},
  volume={162},
  number={5},
  pages={1244--1266},
  year={2016},
  publisher={Springer}
}

@book{rubinstein2013cross,
  title={The cross-entropy method: a unified approach to combinatorial optimization, {M}onte-{C}arlo simulation and machine learning},
  author={Rubinstein, Reuven Y and Kroese, Dirk P},
  year={2013},
  publisher={Springer Science \& Business Media}
}

@article{zhang2014applications,
  title={Applications of the cross-entropy method to importance sampling and optimal control of diffusions},
  author={Zhang, Wei and Wang, Han and Hartmann, Carsten and Weber, Marcus and Schütte, Christof},
  journal={SIAM Journal on Scientific Computing},
  volume={36},
  number={6},
  pages={A2654--A2672},
  year={2014},
  publisher={SIAM}
}

@inproceedings{
holdijk2023stochastic,
title={Stochastic Optimal Control for Collective Variable Free Sampling of Molecular Transition Paths},
author={Lars Holdijk and Yuanqi Du and Ferry Hooft and Priyank Jaini and Bernd Ensing and Max Welling},
booktitle={Thirty-seventh Conference on Neural Information Processing Systems},
year={2023},
}

@article{cuturi2013sinkhorn,
  title={Sinkhorn distances: Lightspeed computation of optimal transport},
  author={Cuturi, Marco},
  journal={Advances in neural information processing systems},
  volume={26},
  year={2013}
}

@inproceedings{xu2023imagereward,
title={ImageReward: Learning and Evaluating Human Preferences for Text-to-Image Generation},
author={Jiazheng Xu and Xiao Liu and Yuchen Wu and Yuxuan Tong and Qinkai Li and Ming Ding and Jie Tang and Yuxiao Dong},
booktitle={Thirty-seventh Conference on Neural Information Processing Systems},
year={2023},
}

@article{hessel2021clipscore,
  title={Clipscore: A reference-free evaluation metric for image captioning},
  author={Hessel, Jack and Holtzman, Ari and Forbes, Maxwell and Bras, Ronan Le and Choi, Yejin},
  journal={arXiv preprint arXiv:2104.08718},
  year={2021}
}

@article{wu2023humanpreferencescorev2,
      title={Human Preference Score v2: A Solid Benchmark for Evaluating Human Preferences of Text-to-Image Synthesis}, 
      author={Xiaoshi Wu and Yiming Hao and Keqiang Sun and Yixiong Chen and Feng Zhu and Rui Zhao and Hongsheng Li},
      year={2023},
      journal={arXiv preprint arXiv:2306.09341},
}

@article{fu2023learning,
title={DreamSim: Learning New Dimensions of Human Visual Similarity using Synthetic Data},
author={Stephanie Fu and Netanel Tamir and Shobhita Sundaram and Lucy Chai and Richard Zhang and Tali Dekel and Phillip Isola},
journal={arXiv preprint arXiv:2306.09344},
year={2023}
}

@article{hendrycks2016gaussian,
  title={Gaussian error linear units (gelus)},
  author={Hendrycks, Dan and Gimpel, Kevin},
  journal={arXiv preprint arXiv:1606.08415},
  year={2016}
}

@article{tancik2020fourier,
  title={Fourier features let networks learn high frequency functions in low dimensional domains},
  author={Tancik, Matthew and Srinivasan, Pratul and Mildenhall, Ben and Fridovich-Keil, Sara and Raghavan, Nithin and Singhal, Utkarsh and Ramamoorthi, Ravi and Barron, Jonathan and Ng, Ren},
  journal={Advances in neural information processing systems},
  volume={33},
  pages={7537--7547},
  year={2020}
}

@article{van2007stochastic,
  title={Stochastic calculus, filtering, and stochastic control},
  author={Van Handel, Ramon},
  journal={Course notes., URL http://www. princeton. edu/rvan/acm217/ACM217. pdf},
  volume={14},
  year={2007}
}

@article{zhang2024improving,
      title={Improving {GF}low{N}ets for Text-to-Image Diffusion Alignment}, 
      author={Dinghuai Zhang and Yizhe Zhang and Jiatao Gu and Ruixiang Zhang and Josh Susskind and Navdeep Jaitly and Shuangfei Zhai},
      year={2024},
      journal={arXiv preprint arXiv:2406.00633},
}

@misc{liu2025efficient,
      title={Efficient Diversity-Preserving Diffusion Alignment via Gradient-Informed {GF}low{N}ets}, 
      author={Zhen Liu and Tim Z. Xiao and Weiyang Liu and Yoshua Bengio and Dinghuai Zhang},
      year={2025},
      eprint={2412.07775},
      archivePrefix={arXiv},
      primaryClass={cs.LG},
      url={https://arxiv.org/abs/2412.07775}, 
}

@misc{liu2025flow,
      title={Flow-GRPO: Training Flow Matching Models via Online RL}, 
      author={Jie Liu and Gongye Liu and Jiajun Liang and Yangguang Li and Jiaheng Liu and Xintao Wang and Pengfei Wan and Di Zhang and Wanli Ouyang},
      year={2025},
      eprint={2505.05470},
      archivePrefix={arXiv},
      primaryClass={cs.CV},
      url={https://arxiv.org/abs/2505.05470}, 
}

@article{vanden2010transition,
  title={Transition-path theory and path-finding algorithms for the study of rare events.},
  author={Vanden-Eijnden, Eric and others},
  journal={Annual review of physical chemistry},
  volume={61},
  pages={391--420},
  year={2010}
}

@inproceedings{black2024training,
title={Training Diffusion Models with Reinforcement Learning},
author={Kevin Black and Michael Janner and Yilun Du and Ilya Kostrikov and Sergey Levine},
booktitle={The Twelfth International Conference on Learning Representations},
year={2024},
}

@article{fan2023dpok,
      title={DPOK: Reinforcement Learning for Fine-tuning Text-to-Image Diffusion Models}, 
      author={Ying Fan and Olivia Watkins and Yuqing Du and Hao Liu and Moonkyung Ryu and Craig Boutilier and Pieter Abbeel and Mohammad Ghavamzadeh and Kangwook Lee and Kimin Lee},
      year={2023},
      journal={arXiv preprint arXiv:2305.16381},
}

@article{uehara2024finetuning,
      title={Fine-Tuning of Continuous-Time Diffusion Models as Entropy-Regularized Control}, 
      author={Masatoshi Uehara and Yulai Zhao and Kevin Black and Ehsan Hajiramezanali and Gabriele Scalia and Nathaniel Lee Diamant and Alex M Tseng and Tommaso Biancalani and Sergey Levine},
      year={2024},
      journal={arXiv preprint arXiv:2402.15194},
}

@article{bradbury2021jax,
  title={Jax: Autograd and xla},
  author={Bradbury, James and Frostig, Roy and Hawkins, Peter and Johnson, Matthew James and Leary, Chris and Maclaurin, Dougal and Necula, George and Paszke, Adam and VanderPlas, Jake and Wanderman-Milne, Skye and others},
  journal={Astrophysics Source Code Library},
  pages={ascl--2111},
  year={2021}
}

@article{kingma2014adam,
  title={Adam: A method for stochastic optimization},
  author={Kingma, Diederik P},
  journal={arXiv preprint arXiv:1412.6980},
  year={2014}
}

@article{cuturi2022optimal,
  title={Optimal transport tools ({OTT}): A jax toolbox for all things {Wasserstein}},
  author={Cuturi, Marco and Meng-Papaxanthos, Laetitia and Tian, Yingtao and Bunne, Charlotte and Davis, Geoff and Teboul, Olivier},
  journal={arXiv preprint arXiv:2201.12324},
  year={2022}
}

@article{schulman2017proximal,
  title={Proximal policy optimization algorithms},
  author={Schulman, John and Wolski, Filip and Dhariwal, Prafulla and Radford, Alec and Klimov, Oleg},
  journal={arXiv preprint arXiv:1707.06347},
  year={2017}
}

@inproceedings{akrour2019projections,
  title={Projections for approximate policy iteration algorithms},
  author={Akrour, Riad and Pajarinen, Joni and Peters, Jan and Neumann, Gerhard},
  booktitle={International Conference on Machine Learning},
  pages={181--190},
  year={2019},
  organization={PMLR}
}

@article{yang2020projection,
  title={Projection-based constrained policy optimization},
  author={Yang, Tsung-Yen and Rosca, Justinian and Narasimhan, Karthik and Ramadge, Peter J},
  journal={arXiv preprint arXiv:2010.03152},
  year={2020}
}

@article{meng2021off,
  title={An off-policy trust region policy optimization method with monotonic improvement guarantee for deep reinforcement learning},
  author={Meng, Wenjia and Zheng, Qian and Shi, Yue and Pan, Gang},
  journal={IEEE Transactions on Neural Networks and Learning Systems},
  volume={33},
  number={5},
  pages={2223--2235},
  year={2021},
  publisher={IEEE}
}

@article{xu2024trust,
  title={Trust region policy optimization via entropy regularization for Kullback--Leibler divergence constraint},
  author={Xu, Haotian and Xuan, Junyu and Zhang, Guangquan and Lu, Jie},
  journal={Neurocomputing},
  volume={589},
  pages={127716},
  year={2024},
  publisher={Elsevier}
}

@article{pajarinen2019compatible,
  title={Compatible natural gradient policy search},
  author={Pajarinen, Joni and Thai, Hong Linh and Akrour, Riad and Peters, Jan and Neumann, Gerhard},
  journal={Machine Learning},
  volume={108},
  pages={1443--1466},
  year={2019},
  publisher={Springer}
}

@inproceedings{peters2010relative,
  title={Relative entropy policy search},
  author={Peters, Jan and Mulling, Katharina and Altun, Yasemin},
  booktitle={Proceedings of the AAAI Conference on Artificial Intelligence},
  volume={24},
  number={1},
  pages={1607--1612},
  year={2010}
}

@article{abdolmaleki2015model,
  title={Model-based relative entropy stochastic search},
  author={Abdolmaleki, Abbas and Lioutikov, Rudolf and Peters, Jan R and Lau, Nuno and Pualo Reis, Luis and Neumann, Gerhard},
  journal={Advances in Neural Information Processing Systems},
  volume={28},
  year={2015}
}

@article{thalmeier2020adaptive,
  title={Adaptive smoothing for path integral control},
  author={Thalmeier, Dominik and Kappen, Hilbert J and Totaro, Simone and G{\'o}mez, Vicen{\c{c}}},
  journal={Journal of Machine Learning Research},
  volume={21},
  number={191},
  pages={1--37},
  year={2020}
}

@inproceedings{schulman2015trust,
  title={Trust region policy optimization},
  author={Schulman, John and Levine, Sergey and Abbeel, Pieter and Jordan, Michael and Moritz, Philipp},
  booktitle={International conference on machine learning},
  pages={1889--1897},
  year={2015},
  organization={PMLR}
}

@article{otto2021differentiable,
  title={Differentiable trust region layers for deep reinforcement learning},
  author={Otto, Fabian and Becker, Philipp and Vien, Ngo Anh and Ziesche, Hanna Carolin and Neumann, Gerhard},
  journal={arXiv preprint arXiv:2101.09207},
  year={2021}
}

@inproceedings{sun2009efficient,
  title={Efficient natural evolution strategies},
  author={Sun, Yi and Wierstra, Daan and Schaul, Tom and Schmidhuber, J{\"u}rgen},
  booktitle={Proceedings of the 11th Annual conference on Genetic and evolutionary computation},
  pages={539--546},
  year={2009}
}

@article{wierstra2014natural,
  title={Natural evolution strategies},
  author={Wierstra, Daan and Schaul, Tom and Glasmachers, Tobias and Sun, Yi and Peters, Jan and Schmidhuber, J{\"u}rgen},
  journal={The Journal of Machine Learning Research},
  volume={15},
  number={1},
  pages={949--980},
  year={2014},
  publisher={JMLR. org}
}

@article{abdolmaleki2018relative,
  title={Relative entropy regularized policy iteration},
  author={Abdolmaleki, Abbas and Springenberg, Jost Tobias and Degrave, Jonas and Bohez, Steven and Tassa, Yuval and Belov, Dan and Heess, Nicolas and Riedmiller, Martin},
  journal={arXiv preprint arXiv:1812.02256},
  year={2018}
}

@article{abdolmaleki2018maximum,
  title={Maximum a posteriori policy optimisation},
  author={Abdolmaleki, Abbas and Springenberg, Jost Tobias and Tassa, Yuval and Munos, Remi and Heess, Nicolas and Riedmiller, Martin},
  journal={arXiv preprint arXiv:1806.06920},
  year={2018}
}

@article{arenz2020trust,
  title={Trust-region variational inference with {G}aussian mixture models},
  author={Arenz, Oleg and Zhong, Mingjun and Neumann, Gerhard},
  journal={Journal of Machine Learning Research},
  volume={21},
  number={163},
  pages={1--60},
  year={2020}
}

@inproceedings{achiam2017constrained,
  title={Constrained policy optimization},
  author={Achiam, Joshua and Held, David and Tamar, Aviv and Abbeel, Pieter},
  booktitle={International conference on machine learning},
  pages={22--31},
  year={2017},
  organization={PMLR}
}

@article{arenz2022unified,
  title={A unified perspective on natural gradient variational inference with {G}aussian mixture models},
  author={Arenz, Oleg and Dahlinger, Philipp and Ye, Zihan and Volpp, Michael and Neumann, Gerhard},
  journal={arXiv preprint arXiv:2209.11533},
  year={2022}
}

@article{wu2017scalable,
  title={Scalable trust-region method for deep reinforcement learning using {Kronecker}-factored approximation},
  author={Wu, Yuhuai and Mansimov, Elman and Grosse, Roger B and Liao, Shun and Ba, Jimmy},
  journal={Advances in neural information processing systems},
  volume={30},
  year={2017}
}

@inproceedings{zhang2021path,
  author    = {Qinsheng Zhang and Yongxin Chen},
  title     = {{Path Integral Sampler}: a stochastic control approach for sampling},
  booktitle = {International Conference on Learning Representations},
  year      = {2022}
}

@article{albergo2024nets,
  title={{NETS}: {A} non-equilibrium transport sampler},
  author={Albergo, Michael S and Vanden-Eijnden, Eric},
  journal={arXiv preprint arXiv:2410.02711},
  year={2024}
}

@article{holderrieth2025leaps,
  title={LEAPS: A discrete neural sampler via locally equivariant networks},
  author={Holderrieth, Peter and Albergo, Michael S and Jaakkola, Tommi},
  journal={arXiv preprint arXiv:2502.10843},
  year={2025}
}

@article{sanokowski2025scalable,
  title={Scalable Discrete Diffusion Samplers: Combinatorial Optimization and Statistical Physics},
  author={Sanokowski, Sebastian and Berghammer, Wilhelm and Ennemoser, Martin and Wang, Haoyu Peter and Hochreiter, Sepp and Lehner, Sebastian},
  journal={arXiv preprint arXiv:2502.08696},
  year={2025}
}

@article{choi2025non,
  title={Non-equilibrium Annealed Adjoint Sampler},
  author={Choi, Jaemoo and Chen, Yongxin and Tao, Molei and Liu, Guan-Horng},
  journal={arXiv preprint arXiv:2506.18165},
  year={2025}
}

@article{liu2025adjoint,
  title={Adjoint Schr\"odinger Bridge Sampler},
  author={Liu, Guan-Horng and Choi, Jaemoo and Chen, Yongxin and Miller, Benjamin Kurt and Chen, Ricky TQ},
  journal={arXiv preprint arXiv:2506.22565},
  year={2025}
}

@inproceedings{erivescontinuously,
  title={Continuously Tempered Diffusion Samplers},
  author={Erives, Ezra and Jing, Bowen and Holderrieth, Peter and Jaakkola, Tommi},
  booktitle={Frontiers in Probabilistic Inference: Learning meets Sampling},
  year={2025}
}

@inproceedings{blessing2025underdamped,
  title={Underdamped diffusion bridges with applications to sampling},
  author={Blessing, Denis and Berner, Julius and Richter, Lorenz and Neumann, Gerhard},
  booktitle={The Thirteenth International Conference on Learning Representations},
  year={2025}
}

@inproceedings{chen2024sequential,
  title={Sequential Controlled {L}angevin Diffusions},
  author={Chen, Junhua and Richter, Lorenz and Berner, Julius and Blessing, Denis and Neumann, Gerhard and Anandkumar, Anima},
  booktitle={The Thirteenth International Conference on Learning Representations},
  year={2025}
}

@article{blessing2025end,
  title={End-to-end learning of {G}aussian mixture priors for diffusion sampler},
  author={Blessing, Denis and Jia, Xiaogang and Neumann, Gerhard},
  journal={arXiv preprint arXiv:2503.00524},
  year={2025}
}

@article{berner2025discrete,
  title={From discrete-time policies to continuous-time diffusion samplers: Asymptotic equivalences and faster training},
  author={Berner, Julius and Richter, Lorenz and Sendera, Marcin and Rector-Brooks, Jarrid and Malkin, Nikolay},
  journal={arXiv preprint arXiv:2501.06148},
  year={2025}
}

@article{kim2024adaptive,
  title={Adaptive teachers for amortized samplers},
  author={Kim, Minsu and Choi, Sanghyeok and Yun, Taeyoung and Bengio, Emmanuel and Feng, Leo and Rector-Brooks, Jarrid and Ahn, Sungsoo and Park, Jinkyoo and Malkin, Nikolay and Bengio, Yoshua},
  journal={arXiv preprint arXiv:2410.01432},
  year={2024}
}

@article{rissanen2025progressive,
  title={Progressive Tempering Sampler with Diffusion},
  author={Rissanen, Severi and OuYang, RuiKang and He, Jiajun and Chen, Wenlin and Heinonen, Markus and Solin, Arno and Hern{\'a}ndez-Lobato, Jos{\'e} Miguel},
  journal={arXiv preprint arXiv:2506.05231},
  year={2025}
}

@article{sanokowski2025rethinking,
  title={Rethinking Losses for Diffusion Bridge Samplers},
  author={Sanokowski, Sebastian and Gruber, Lukas and Bartmann, Christoph and Hochreiter, Sepp and Lehner, Sebastian},
  journal={arXiv preprint arXiv:2506.10982},
  year={2025}
}

@article{kim2025scalable,
  title={On scalable and efficient training of diffusion samplers},
  author={Kim, Minkyu and Seong, Kiyoung and Woo, Dongyeop and Ahn, Sungsoo and Kim, Minsu},
  journal={arXiv preprint arXiv:2505.19552},
  year={2025}
}

@article{grenioux2025improving,
  title={Improving the evaluation of samplers on multi-modal targets},
  author={Grenioux, Louis and Noble, Maxence and Gabri{\'e}, Marylou},
  journal={arXiv preprint arXiv:2504.08916},
  year={2025}
}

@inproceedings{richter2023improved,
  title={Improved sampling via learned diffusions},
  author={Richter, Lorenz and Berner, Julius},
  booktitle={International Conference on Learning Representations},
  year={2024}
}

@article{dai1991stochastic,
  title={A stochastic control approach to reciprocal diffusion processes},
  author={Dai Pra, Paolo},
  journal={Applied mathematics and Optimization},
  volume={23},
  number={1},
  pages={313--329},
  year={1991},
  publisher={Springer}
}

@article{berner2022optimal,
  title={An optimal control perspective on diffusion-based generative modeling},
  author={Berner, Julius and Richter, Lorenz and Ullrich, Karen},
  journal={Transactions on Machine Learning Research},
  year={2024}
}

@article{havens2025adjoint,
  title={Adjoint sampling: Highly scalable diffusion samplers via adjoint matching},
  author={Havens, Aaron and Miller, Benjamin Kurt and Yan, Bing and Domingo-Enrich, Carles and Sriram, Anuroop and Wood, Brandon and Levine, Daniel and Hu, Bin and Amos, Brandon and Karrer, Brian and others},
  journal={arXiv preprint arXiv:2504.11713},
  year={2025}
}

@article{gritsaev2025adaptive,
  title={Adaptive Destruction Processes for Diffusion Samplers},
  author={Gritsaev, Timofei and Morozov, Nikita and Tamogashev, Kirill and Tiapkin, Daniil and Samsonov, Sergey and Naumov, Alexey and Vetrov, Dmitry and Malkin, Nikolay},
  journal={arXiv preprint arXiv:2506.01541},
  year={2025}
}

@article{akhound2024iterated,
  title={Iterated denoising energy matching for sampling from {B}oltzmann densities},
  author={Akhound-Sadegh, Tara and Rector-Brooks, Jarrid and Bose, Avishek Joey and Mittal, Sarthak and Lemos, Pablo and Liu, Cheng-Hao and Sendera, Marcin and Ravanbakhsh, Siamak and Gidel, Gauthier and Bengio, Yoshua and others},
  journal={arXiv preprint arXiv:2402.06121},
  year={2024}
}

@article{sendera2024improved,
  title={Improved off-policy training of diffusion samplers},
  author={Sendera, Marcin and Kim, Minsu and Mittal, Sarthak and Lemos, Pablo and Scimeca, Luca and Rector-Brooks, Jarrid and Adam, Alexandre and Bengio, Yoshua and Malkin, Nikolay},
  journal={Advances in Neural Information Processing Systems},
  volume={37},
  pages={81016--81045},
  year={2024}
}

@article{tan2025scalable,
  title={Scalable equilibrium sampling with sequential {B}oltzmann generators},
  author={Tan, Charlie B and Bose, Avishek Joey and Lin, Chen and Klein, Leon and Bronstein, Michael M and Tong, Alexander},
  journal={arXiv preprint arXiv:2502.18462},
  year={2025}
}

@inproceedings{zhang2025generalised,
  title={Generalised parallel tempering: Flexible replica exchange via flows and diffusions},
  author={Zhang, Leo and Potaptchik, Peter and Deligiannidis, George and Doucet, Arnaud and Dau, Hai-Dang and Syed, Saifuddin},
  booktitle={Frontiers in Probabilistic Inference: Learning meets Sampling},
  year={2025}
}

@article{zhang2025accelerated,
  title={Accelerated Parallel Tempering via Neural Transports},
  author={Zhang, Leo and Potaptchik, Peter and He, Jiajun and Du, Yuanqi and Doucet, Arnaud and Vargas, Francisco and Dau, Hai-Dang and Syed, Saifuddin},
  journal={arXiv preprint arXiv:2502.10328},
  year={2025}
}

@article{wang2025importance,
  title={Importance Weighted Score Matching for Diffusion Samplers with Enhanced Mode Coverage},
  author={Wang, Chenguang and Zhang, Xiaoyu and Cui, Kaiyuan and Zhao, Weichen and Guan, Yongtao and Yu, Tianshu},
  journal={arXiv preprint arXiv:2505.19431},
  year={2025}
}

@article{guo2025complexity,
  title={Complexity Analysis of Normalizing Constant Estimation: from {Jarzynski} Equality to Annealed Importance Sampling and beyond},
  author={Guo, Wei and Tao, Molei and Chen, Yongxin},
  journal={arXiv preprint arXiv:2502.04575},
  year={2025}
}

@inproceedings{vargas2024transport,
  title={Transport meets variational inference: Controlled {M}onte {C}arlo diffusions},
  author={Vargas, Francisco and Padhy, Shreyas and Blessing, Denis and N{\"u}sken, Nikolas},
  booktitle={The Twelfth International Conference on Learning Representations},
  year={2024}
}

@article{vargas2023denoising,
  title={Denoising diffusion samplers},
  author={Vargas, Francisco and Grathwohl, Will and Doucet, Arnaud},
  journal={arXiv preprint arXiv:2302.13834},
  year={2023}
}

@article{vargas2023bayesian,
  title={Bayesian learning via neural {Schr{\"o}dinger}--{F{\"o}llmer} flows},
  author={Vargas, Francisco and Ovsianas, Andrius and Fernandes, David and Girolami, Mark and Lawrence, Neil D and N{\"u}sken, Nikolas},
  journal={Statistics and Computing},
  volume={33},
  number={1},
  pages={3},
  year={2023},
  publisher={Springer}
}

@article{pavon1989stochastic,
  title={Stochastic control and nonequilibrium thermodynamical systems},
  author={Pavon, Michele},
  journal={Applied Mathematics and Optimization},
  volume={19},
  number={1},
  pages={187--202},
  year={1989},
  publisher={Springer}
}

@article{he2025no,
  title={No Trick, No Treat: Pursuits and Challenges Towards Simulation-free Training of Neural Samplers},
  author={He, Jiajun and Du, Yuanqi and Vargas, Francisco and Zhang, Dinghuai and Padhy, Shreyas and OuYang, RuiKang and Gomes, Carla and Hern{\'a}ndez-Lobato, Jos{\'e} Miguel},
  journal={arXiv preprint arXiv:2502.06685},
  year={2025}
}

@article{sun2024dynamical,
  title={Dynamical measure transport and neural {PDE} solvers for sampling},
  author={Sun, Jingtong and Berner, Julius and Richter, Lorenz and Zeinhofer, Marius and M{\"u}ller, Johannes and Azizzadenesheli, Kamyar and Anandkumar, Anima},
  journal={arXiv preprint arXiv:2407.07873},
  year={2024}
}

@article{noble2024learned,
  title={Learned Reference-based Diffusion Sampling for multi-modal distributions},
  author={Noble, Maxence and Grenioux, Louis and Gabri{\'e}, Marylou and Durmus, Alain Oliviero},
  journal={arXiv preprint arXiv:2410.19449},
  year={2024}
}

@article{huang2021schrodinger,
  title={{Schr{ö}dinger}-{F{ö}llmer} sampler: sampling without ergodicity},
  author={Huang, Jian and Jiao, Yuling and Kang, Lican and Liao, Xu and Liu, Jin and Liu, Yanyan},
  journal={arXiv preprint arXiv:2106.10880},
  year={2021}
}

@article{ding2023sampling,
  title={Sampling via {Föllmer} Flow},
  author={Ding, Zhao and Jiao, Yuling and Lu, Xiliang and Yang, Zhijian and Yuan, Cheng},
  journal={arXiv preprint arXiv:2311.03660},
  year={2023}
}

@article{huang2023monte,
  title={Monte {C}arlo Sampling without Isoperimetry: A Reverse Diffusion Approach},
  author={Huang, Xunpeng and Dong, Hanze and Hao, Yifan and Ma, Yian and Zhang, Tong},
  journal={arXiv preprint arXiv:2307.02037},
  year={2023}
}

@article{shi2024diffusion,
  title={Diffusion-{PINN} Sampler},
  author={Shi, Zhekun and Yu, Longlin and Xie, Tianyu and Zhang, Cheng},
  journal={arXiv preprint arXiv:2410.15336},
  year={2024}
}

@article{yoon2025value,
  title={Value Gradient Sampler: Sampling as Sequential Decision Making},
  author={Yoon, Sangwoong and Hwang, Himchan and Jeong, Hyeokju and Shin, Dong Kyu and Park, Che-Sang and Kweon, Sehee and Park, Frank Chongwoo},
  journal={arXiv preprint arXiv:2502.13280},
  year={2025}
}

@article{geffner2021mcmc,
  title={{MCMC} variational inference via uncorrected Hamiltonian annealing},
  author={Geffner, Tomas and Domke, Justin},
  journal={Advances in Neural Information Processing Systems},
  volume={34},
  pages={639--651},
  year={2021}
}

@inproceedings{geffner2023langevin,
  title={Langevin diffusion variational inference},
  author={Geffner, Tomas and Domke, Justin},
  booktitle={International Conference on Artificial Intelligence and Statistics},
  pages={576--593},
  year={2023},
  organization={PMLR}
}

@InProceedings{Thin:2021,
  title = 	 {Monte {C}arlo Variational Auto-Encoders},
  author =       {Thin, Achille and Kotelevskii, Nikita and Durmus, Alain and Moulines, Eric and Panov, Maxim and Doucet, Arnaud},
  booktitle = 	 {International Conference on Machine Learning},
  year = 	 {2021},
}

@inproceedings{ZhangAIS2021,
       author = {{Zhang}, Guodong and {Hsu}, Kyle and {Li}, Jianing and {Finn}, Chelsea  
        and {Grosse}, Roger},
        title = {Differentiable annealed importance sampling and the perils of gradient noise},
	    Booktitle = {Advances in Neural Information Processing Systems},
        year = 2021,
}

@inproceedings{doucet2022annealed,
  title={Score-based diffusion meets Annealed Importance Sampling},
  author={Doucet, Arnaud and Grathwohl, Will and Matthews, Alexander G de G and Strathmann, Heiko},
  booktitle={Advances in Neural Information Processing Systems},
  year={2022}
}

@article{ouyang2024bnem,
  title={{BNEM}: A {B}oltzmann Sampler Based on Bootstrapped Noised Energy Matching},
  author={OuYang, RuiKang and Qiang, Bo and Hern{\'a}ndez-Lobato, Jos{\'e} Miguel},
  journal={arXiv preprint arXiv:2409.09787},
  year={2024}
}

@article{he2024training,
  title={Training neural samplers with reverse diffusive kl divergence},
  author={He, Jiajun and Chen, Wenlin and Zhang, Mingtian and Barber, David and Hern{\'a}ndez-Lobato, Jos{\'e} Miguel},
  journal={arXiv preprint arXiv:2410.12456},
  year={2024}
}

@misc{zhu2025mdnsmaskeddiffusionneural,
      title={MDNS: Masked Diffusion Neural Sampler via Stochastic Optimal Control}, 
      author={Yuchen Zhu and Wei Guo and Jaemoo Choi and Guan-Horng Liu and Yongxin Chen and Molei Tao},
      year={2025},
      eprint={2508.10684},
      archivePrefix={arXiv},
      primaryClass={cs.LG},
      url={https://arxiv.org/abs/2508.10684}, 
}

@article{akhound2025progressive,
  title={Progressive Inference-Time Annealing of Diffusion Models for Sampling from {B}oltzmann Densities},
  author={Akhound-Sadegh, Tara and Lee, Jungyoon and Bose, Avishek Joey and De Bortoli, Valentin and Doucet, Arnaud and Bronstein, Michael M and Beaini, Dominique and Ravanbakhsh, Siamak and Neklyudov, Kirill and Tong, Alexander},
  journal={arXiv preprint arXiv:2506.16471},
  year={2025}
}

@article{wu2025reverse,
  title={Reverse Diffusion {S}equential {M}onte {C}arlo Samplers},
  author={Wu, Luhuan and Han, Yi and Naesseth, Christian A and Cunningham, John P},
  journal={arXiv preprint arXiv:2508.05926},
  year={2025}
}


\newpage
\appendix

\startcontents[appendix]

\printcontents[appendix]{}{1}{\section*{Appendix}}
\clearpage

\linespread{1}
\setlength{\parskip}{0.4\baselineskip plus 0.25pt}
\setlength{\abovedisplayskip}{12pt plus 3pt minus 9pt}%
\setlength{\belowdisplayskip}{12pt plus 3pt minus 9pt}%
\setlength{\abovedisplayshortskip}{0pt plus 3pt}%
\setlength{\belowdisplayshortskip}{7pt plus 3pt minus 4pt}%
\normalsize

\section{Assumptions and auxiliary results}
\label{app:assumptions_and_auxiliary}

\subsection{Additional notation}

For vectors $v_1, v_2 \in \R^d$, we denote by $\|v\|$ the Euclidean norm and by $v_1 \cdot v_2$ the Euclidean inner product. For a real-valued matrix $A$, we denote by $\operatorname{Tr}(A)$ and $A^\top$ its trace and transpose.

For a sufficiently smooth function $f \colon \R^d \times [0,T] \to \R$, we denote by $\nabla f = \nabla_x f$ its gradient w.r.t. the spatial variables $x$ and by $\partial_t f$ and $\partial_{x_i}f$ its partial derivatives w.r.t.\@ the time coordinate $t$ and the spatial coordinate $x_i$, respectively. 

We denote by $\mathcal{N}(\mu,\Sigma)$ a multivariate normal distribution with mean $\mu \in \R^d$ and covariance matrix $\Sigma \in \R^{d \times d}$.  
Moreover, we denote by $\mathrm{Unif}([0,T])$ the uniform distribution on $[0,T]$. For random variables $X_1$, $X_2$, we denote by $\E[X_1]$ and $\V[X_1]$ the expectation and variance of $X_1$ and by $\E[X_1|X_2]$ the conditional expectation of $X_1$ given $X_2$.

\subsection{Technical assumptions} \label{sec: technical assumptions}
Throughout our work, we make the same assumptions as in \cite{nuesken2021solving,domingoenrich2024stochastic}, which are needed for all the objects considered to be well-defined. Namely, we assume that:
\begin{enumerate}[label=(\roman*)]
    \item The set $\mathcal{U}$ of \textit{admissible controls} is given by
    \begin{align}
        \mathcal{U} = \{ u \in C^1(\R^d \times [0,T] ; \R^d) \, | \, \exists C > 0, \, \forall (x,s) \in \R^d \times [0,T], \, u(x,s) \leq C(1+\|x\|) \}. 
    \end{align}
    \item The coefficients $b$ and $\sigma$ are continuously differentiable, $\sigma$ has bounded first-order spatial derivatives, and $(\sigma \sigma^{\top})(x, s)$ is positive definite for all $(x, s) \in \R^d \times [0, T ]$. Furthermore, there exist constants $C, c_1, c_2 > 0$ such that
    \begin{align}
    \begin{split}
    \|b(x, s)\| \leq C (1 + \|x\|), \qquad &\text{(linear growth)} \\
    c_1\|\beta\|^2 \leq \beta^{\top} (\sigma \sigma^{\top})(x, s) \beta \leq c_2\|\beta\|^2, \qquad &\text{(ellipticity)}
    \end{split}
    \end{align}
    for all $(x, s) \in \R^d \times [0, T ]$ and $\beta \in \R^d$.
\end{enumerate}

\subsection{Useful identities} \label{sec: useful identities}

\begin{definition}[Controlled SDEs]
Let \( u \in \mathcal{U} \) be a control function. Throughout, we consider controlled and uncontrolled stochastic processes defined via the SDEs
\begin{align}
    \dd X^u_s &= \left(b + \sigma u\right)(X^u_s, s) \dd s + \sigma(s) \dd W_s, \quad && X^u_0 \sim p_0, \\
    \dd X_s &= b(X_s, s)\dd s + \sigma(s) \dd W_s, \quad && X_0 \sim p_0,
\end{align}
where \( X^u \sim \mathbb{P}^u \), \( X \sim \mathbb{P} \), with $\mathbb{P}^u$ and $\mathbb{P}$ denoting the respective path space measures, and \( W \) is a standard Brownian motion.
\end{definition}

\begin{theorem}[Girsanov’s theorem for path measures] \label{theorem: girsanov}
Let \( u, v, w \in \mathcal{U} \). Then the Radon-Nikodym derivative between \( \mathbb{P}^u \) and \( \mathbb{P}^v \), evaluated along \( X^w \), is given by:
\begin{align}
    \log \frac{\dd {\mathbb{P}}^{u}}{\dd {\mathbb{P}}^v}(X^w)  & = \int_0^T \sigma^{-1}(u-v)(X^w_s,s) \cdot \dd X^w_s- \frac{1}{2}\int_0^T \left(\|\sigma^{-1}b + u\|^2 - \|\sigma^{-1}b + v\|^2\right)(X^w_s,s) \dd s,
\end{align}
\end{theorem}

\begin{proof}
    See, e.g., Lemma A.1 in \cite{nuesken2021solving} or Appendix E in \cite{vargas2024transport}.
\end{proof}

\begin{corollary}[Change of measure identities]
Let \( u, v \in \mathcal{U} \). From~\Cref{theorem: girsanov}, we obtain the following useful identities:
\begin{enumerate}[label=(\roman*)]
    \item $\log \frac{\dd \mathbb{P}^u}{\dd \mathbb{P}}(X^u) = \int_0^T u(X^u_s, s) \cdot \dd W_s + \frac{1}{2} \int_0^T \|u(X^u_s, s)\|^2 \dd s$
    \item $\log \frac{\dd \mathbb{P}^u}{\dd \mathbb{P}}(X)  = \int_0^T u(X_s, s) \cdot \dd W_s - \frac{1}{2} \int_0^T \|u(X_s, s)\|^2 \dd s$
    \item $\log \frac{\dd \mathbb{P}^u}{\dd \mathbb{P}^v}(X^u) = \int_0^T (u - v)(X^u_s, s) \cdot \dd W_s + \frac{1}{2} \int_0^T \|u - v\|^2(X^u_s, s)\dd s$
    \item $\log \frac{\dd \mathbb{P}^u}{\dd \mathbb{P}^v}(X^v) = \int_0^T (u - v)(X^v_s, s) \cdot \dd W_s - \frac{1}{2} \int_0^T \|u - v\|^2(X^v_s, s)\dd s$
\end{enumerate}
\end{corollary}

\begin{lemma}[Itô's formula]
Let \( X_s \) solve the SDE
\[
\dd X_s = b(X_s, s) \dd s + \sigma(s) \dd W_s,
\]
and let \( f \colon \mathbb{R}^d \times [0, T] \to \mathbb{R} \) be a smooth function. Then
\[
\dd f(X_s, s) = (\partial_s + L)f(X_s, s)\dd s + \sigma^\top \nabla f(X_s, s) \cdot \dd W_s,
\]
where $L$ is the infinitesimal generator given by
\[
L \coloneqq \frac{1}{2} \sum_{i,j=1}^d (\sigma \sigma^\top)_{ij} \partial_{x_i}\partial_{x_j} + \sum_{i=1}^d b_i(x, t)\partial_{x_i}.
\]
\end{lemma}

\section{Proofs}
\label{app: proofs}

\begin{proof}[Proof of \Cref{prop: Optimal change of measure as geometric annealing}]

    Let $\widetilde{\mathbb{P}}$ be the measure defined by $\frac{\mathrm d \widetilde{\mathbb{P}}}{\mathrm d \P^{u_i}} = \left(\frac{\mathrm d \Q}{\mathrm d \P^{u_i}}\right)^{\frac{1}{1 + \lambda_i}}/\widetilde{\mathcal{Z}}$, where $\widetilde{\mathcal{Z}}$ is the normalizing constant. Then we have that 
    \begin{subequations}
    \begin{align}
        (1 + \lambda_i)\log \frac{{\mathrm d\mathbb{P}}^u}{\mathrm d\widetilde{\mathbb{P}}}  = (1 + \lambda_i)\log \left(\frac{{\mathrm d\mathbb{P}}^u}{\mathrm d\P^{u_i}} \frac{\mathrm d\P^{u_i}}{\mathrm d\widetilde{\mathbb{P}}}\right) &= (1 + \lambda_i)\log \frac{{\mathrm d\mathbb{P}}^u}{\mathrm d\P^{u_i}}  + \log \frac{\mathrm d\P^{u_i}}{\mathrm d\Q} +  (1 + \lambda_i) \log \widetilde{\mathcal{Z}}\\
        &= \lambda_i\log \frac{{\mathrm d\mathbb{P}}^u}{\mathrm d\P^{u_i}}  + \log \frac{\mathrm d\P^{u}}{\mathrm d\Q} +  (1 + \lambda_i) \log \widetilde{\mathcal{Z}}. 
    \end{align}
    \end{subequations}
    Using the definition of the Lagrangian in~\eqref{eq: Lagrangian loss}, this implies that
    \begin{subequations}
    \begin{align}
       (1 + \lambda_i) \KL({\mathbb{P}}^u|\widetilde{\mathbb{P}}) &=  \lambda_i\KL\left({\mathbb{P}}^u|{\P}^{u_i}\right)  + \KL\left({\mathbb{P}}^u|\Q\right) +  (1 + \lambda_i) \E\left[\log \widetilde{\mathcal{Z}}(X^u_0)\right] \\
       & = \mathcal{L}_{\mathrm{TR}}^{(i)}(u, \lambda_i) + (1 + \lambda_i) \E\left[\log \widetilde{\mathcal{Z}}(X^u_0)\right] +\lambda_i \varepsilon,
    \end{align}
    \end{subequations}
    Since we defined the minimizer of the Lagrangian (with optimal multiplier $\lambda_i$) in the last expression as $u_{i+1}$, we have that
    $u_{i+1} = \argmin_{u\in \mathcal{U}}\KL({\mathbb{P}}^u|\widetilde{\mathbb{P}})$. 
    This shows that $\widetilde{\mathbb{P}} = \P^{u_{i+1}}$ by the uniqueness of the Radon-Nikodym derivative. For the second statement, we introduce the unnormalized path measure $\widetilde{\P}^{u_{i+1}}$ such that
    \begin{equation}
    \label{eq: unnormalized optimal iterates and log Z}
        \frac{\dd \P^{u_{i+1}}}{\dd \P}(X) = \frac{1}{\widetilde{\Z}_{i+1}(X_0)}\frac{\dd \widetilde{\P}^{u_{i+1}}}{\dd \P}(X) \quad \text{with} \quad  \widetilde{\Z}_{i+1}(X_0) = \E\left[\frac{\dd \widetilde{\P}^{u_{i+1}}}{\dd \P}(X)\bigg|X_0\right]
    \end{equation}
    and 
    \begin{equation}
    \label{eq: unnormalized incremental}
        \frac{\dd \widetilde{\P}^{u_{i+1}}}{\dd \P}(X) = \left(\frac{\dd{\Q}}{\dd \P}(X)\right)^{\frac{1}{1+\lambda_i}}
\left(\frac{\dd{\widetilde{\P}}^{u_i}}{\dd \P}(X)\right)^{\frac{\lambda_i}{1+\lambda_i}}.    
\end{equation}
    Assuming $\widetilde{\Z}_0 = 1$, we have $\widetilde{\P}^{u_{0}}={\P}^{u_{0}}$ for $u_0 \in \mathcal{U}$. By induction, it follows that 
    \begin{equation}
        \frac{\dd \widetilde{\P}^{u_{i+1}}}{\dd \P}(X)=\left(\frac{\dd{\Q}}{\dd \P}(X)\right)^{\beta_{i+1}}\left(\frac{\dd{{\P}}^{u_0}}{\dd \P}(X)\right)^{1-\beta_{i+1}},
    \end{equation}
    with $\beta_{i+1}$ defined as in~\Cref{prop: Optimal change of measure as geometric annealing}, which proves the second statement.
    \highprio{JB: this needs more attention: one needs to show that $\widetilde{P}$ can be written as a controlled path measure, which probably can be shown that it again defines an SOC problem. We can avoid it by looking at abstract measures in general; see below.}
    \highprio{JB: we could omit the statement on the expressiveness by defining the optimization problem in~\eqref{eq: trust region} directly over path measures. Lorenz: True. I tend to leave it like it is in order to focus more on optimal control -- I think that would fir better to the story. We could remark though that everything works for abstract measures.}
    \highprio{Lorenz: This proof relies on the KL divergence, right? I'm wondering if it can also be done with a general divergence, but one would get different formulas then, right? (This might be relevant for \Cref{rem: trust regions for general measures}.}
    \highprio{Lorenz: Do we need to change something here due to $u_0 \neq 0$? JB: for the 1st or 2nd part?}
\vspace{-0.1em}\end{proof}
    \begin{remark} Using the left side of \eqref{eq: unnormalized optimal iterates and log Z}, we can rewrite the normalized version of \eqref{eq: unnormalized incremental} as
    \begin{align}
        \frac{\dd \P^{u_{i+1}}}{\dd \P}(X) & = \frac{1}{\widetilde{\Z}_{i+1}(X_0)}\left(\frac{\dd{\Q}}{\dd \P}(X)\right)^{\frac{1}{1+\lambda_i}}
        \left(\frac{\dd{\widetilde{\P}}^{u_i}}{\dd \P}(X)\right)^{\frac{\lambda_i}{1+\lambda_i}} \\ &= \frac{1}{\widehat{\Z}_{i+1}(X_0)}\left(\frac{\dd{\Q}}{\dd \P}(X)\right)^{\frac{1}{1+\lambda_i}}
        \left(\frac{\dd{{\P}}^{u_i}}{\dd \P}(X)\right)^{\frac{\lambda_i}{1+\lambda_i}}.
    \end{align}
    with $\widehat{\Z}_{i+1}={\widetilde{\Z}_{i+1}}/\widetilde{\Z}^{\frac{\lambda_i}{1+ \lambda_i}}_{i}$.        
    \end{remark}

\section{Further related works, broader impact, and limitations}
\label{app:broader_impact_limitations}

\subsection{Further related works}
\label{app:related_works}

\paragraph{Monte Carlo estimator.} In theory, one could directly compute the optimal control using the representations in~\Cref{prop: tr_soc_optimality} (for $\lambda=0$ and $i=0$; see~\Cref{it:soc_sol} in~\Cref{thm:soc_optimality}) combined with Monte Carlo estimates\footnote{One can obtain derivative estimates using adjoint states (as defined in~\Cref{sec:learning_constrained_control}) or using reparametrization tricks if the uncontrolled process has suitable, known marginals. 
For Gaussian marginals, one can also use Stein's lemma~\cite{huang2021schrodinger}. We further note that control variates for such estimators have been analyzed in~\cite{sabate2021unbiased,richter2022robust}.} of the value function in~\Cref{it:V} in~\Cref{thm:soc_optimality}~\cite{huang2021schrodinger,vargas2023bayesian,huang2023monte,ding2023sampling,tan2023noise}. However, in practice this can be problematic since it requires a large amount of samples \emph{for each state} $x$ due to the (typically) very high variance of the estimator for $V$~\cite{vargas2023bayesian}. In particular, we note that the variance translates to a bias in the control due to the logarithmic transformation. Moreover, for nonzero $f$ or general $b$ (e.g., in the fine-tuning setting), one needs to \emph{simulate} the uncontrolled process to obtain samples.
\highprio{JB: talk about dynamic programming/policy iteration?}

\paragraph{PDE solver.} One can also leverage the representation of the value function as the solution of an HJB equation (see~\Cref{it:hjb} in~\Cref{thm:soc_optimality}). While solving PDEs in high dimensions is very challenging, there exist scalable approaches based on tensor trains and neural networks\footnote{Note that some of these approaches correspond to regressions of the Monte Carlo estimators mentioned above~\cite{vargas2023bayesian} or to the IDO methods mentioned below~\cite{sun2024dynamical}.} that leverage backward stochastic differential equations or the Hopf-Cole transform in combination with the Feynman-Kac formula~\cite{han2018solving,berner2020numerically,richter2021thesis,beck2021solving,richter2022robust,richter2023continuous,richter2021solving,akhound2024iterated}. However, in practice, we only need the value function in the domain where the optimal path measure has sufficiently large values, which is typically not considered for PDE solvers. 

\paragraph{Iterative diffusion optimization.} 
To focus more on promising regions of the path space, methods for iterative diffusion optimization simulate {(sub-)trajectories} of the controlled SDE to compute a suitable loss and update the control. Typically, the control is parametrized as a neural network and optimized using variants of stochastic gradient descent. While such methods have been explored for general SOC problems with quadratic control costs~\cite{nuesken2021solving,richter2021thesis,domingoenrich2024stochastic,domingo2024taxonomy}, many recent works have focused on the special case of sampling from unnormalized densities as described in~\Cref{sec:application sampling}; see, e.g.,~\cite{zhang2021path,berner2022optimal,vargas2023bayesian,zhangdiffusion,vargas2024transport,sendera2024improved,noble2024learned,albergo2024nets}. From the perspective of path measures, all these works propose to minimize suitable divergences between measures induced by controlled SDEs. While we demonstrate the benefits of leveraging trust region methods for the \emph{Denoising Diffusion Sampler} (DDS)~\cite{vargas2023denoising}, our method could also be extended to other samplers. 

\paragraph{Transition path sampling.} Transition path sampling has been a longstanding problem in physics and chemistry to understand phase transitions and chemical reactions, with applications in energy, catalysis, and drug discovery~\cite{bolhuis2002transition,vanden2010transition}. Computationally, MCMC-based approaches have been extended to path space to mix the transition path distribution, pioneered by~\cite{dellago1998efficient}. As discussed in \Cref{sec:tps}, transition path sampling can be formulated as a stochastic optimal control problem and has been numerically solved using reverse KL divergence~\cite{yan2022learning}, cross-entropy divergence~\cite{holdijk2023stochastic}, and log-variance divergence~\cite{seong2024transition}; the optimal control is known to be the Doob's $h$-function~\cite{chetrite2015variational,singh2023variational,dudoob2024} (for a review, we refer to~\cite{singh2025variational}). To solve the Doob's $h$-function, \cite{singh2023variational} proposes a shooting-based method which requires MD simulation to reach the target state, while~\cite{dudoob2024} proposes a Gaussian approximation conditioned on both the initial and target state which satisfies boundary conditions by design and provides a simulation-free optimization algorithm. Similarly to SOC, transition path sampling can also naturally be formulated as a reinforcement learning problem, as shown in~\cite{das2021reinforcement,rose2021reinforcement}. 

\paragraph{Diffusion and flow matching reward fine-tuning.} Several of the early works on diffusion fine-tuning focused on directly optimizing the reward model making use of its differentiability \cite{xu2023imagereward,clark2024directly}, without any KL regularization, which can lead to ``reward hacking''. Some other works \cite{black2024training,fan2023dpok} framed reward fine-tuning as a reinforcement learning problem, but did not make the probabilistic connection to tilted distributions. \cite{uehara2024finetuning} provides a probabilistic view of the problem, but proposes an algorithm that is hard to scale. \cite{domingoenrich2025adjoint} gives a comprehensive view of flow matching reward fine-tuning, introducing memoryless noise schedules as the right ones, as well as a new scalable SOC algorithm that we use and adapt, namely adjoint matching. Using the memoryless noise schedule, a recent work \cite{liu2025flow} considers GRPO for flow matching fine-tuning. \cite{zhang2024improving,liu2025efficient} consider alternative algorithms that learn the value functions.

\paragraph{Diffusion-based sampling from unnormalized densities.}
Early work on sampling from unnormalized densities based on a Schr\"odinger-F\"ollmer diffusions dates back to~\cite{dai1991stochastic} and was later implemented using Monte Carlo \cite{huang2021schrodinger,ding2023sampling} and deep learning approaches  \cite{zhang2021path,richter2021thesis,vargas2023bayesian}. Another line of work is based on Langevin diffusions \cite{Thin:2021,ZhangAIS2021,doucet2022annealed,geffner2021mcmc,geffner2023langevin} and denoising diffusion models based on Ornstein-Uhlenbeck processes 
\cite{berner2022optimal,vargas2023denoising,huang2023monte}. A unifying perspective was proposed in \cite{richter2023improved,vargas2024transport}, which consider general diffusion bridges. An extension based on underdamped diffusion processes was later proposed by \cite{blessing2025underdamped}.
Recent developments have led to improved loss functions and training schemes  
\cite{zhangdiffusion,choi2025non,liu2025adjoint,havens2025adjoint,akhound2024iterated,he2024training,ouyang2024bnem,erivescontinuously,wang2025importance,yoon2025value,sanokowski2025rethinking,gritsaev2025adaptive,sun2024dynamical,shi2024diffusion,albergo2024nets,berner2025discrete}, exploration capabilities \cite{kim2024adaptive,blessing2025end,kim2025scalable}, or normalizing constant estimation \cite{guo2025complexity,he2025feat}. Other studies focus on the combination of MCMC and diffusion-based sampling methods 
\cite{chen2024sequential,zhang2025generalised,zhang2025accelerated,akhound2025progressive,sendera2024improved,rissanen2025progressive,wu2025reverse}.
Approaches for discrete state spaces have been proposed in \cite{holderrieth2025leaps,sanokowski2025scalable,zhu2025mdnsmaskeddiffusionneural}. Combinations of diffusion-based sampling with additional access to ground truth data have been studied in \cite{noble2024learned,tan2025scalable}. Lastly, \cite{grenioux2025improving, blessing2024beyond} study improved evaluation techniques.

\highprio{JB: describe different approaches for the KL case as in Carles' taxonomy paper (in particular adjoint matching).}

\subsection{Limitations}
While our method for solving stochastic optimal control problems exhibits strong sample efficiency, it relies on storing entire trajectories in the replay buffer during training. In large-scale settings -- such as fine-tuning text-to-image models -- this necessitates keeping the replay buffer in CPU memory while training occurs on the GPU. This separation introduces additional computational overhead due to data transfers between CPU and GPU; however, the buffer still significantly accelerates the fine-tuning since the main computational cost in such settings stems from the simulation of trajectories.

\subsection{Broader impact}
This paper proposes new methodologies and theories that find numerical solutions for stochastic optimal control problems ranging from equilibrium sampling, transition path sampling, to fine-tuning text-to-image generative models. Equilibrium sampling and transition path sampling are important in Bayesian statistics, physics and chemistry where they can be used to estimate free energy, understand phase transition and rare events, thus holding promises to accelerate drug and material discovery. More efficient fine-tuning of text-to-image models democratizes the generation of specialized high-quality visual content for creative applications. However, these capabilities also introduce risks such as the potential for generating convincing misinformation or deepfakes.

\section{Background on SOC}
\label{app:soc}

\subsection{Stochastic optimal control} \label{sec: stochastic optimal control}
In this work, we consider stochastic optimal control (SOC) problems of the form
\begin{align}
\label{eq:soc_obj app}
    & \min_{u \in \mathcal{U}} \ \mathcal{L}_{\mathrm{SOC}}(u) = \min_{u \in \mathcal{U}}\  \E \left[ \int_0^T \left( \tfrac{1}{2} \|u(X^u_s,s)\|^2 + f(X^u_s,s) \right) \, \dd s + g(X^u_T)\right],
\end{align}
with state costs $f$, terminal costs $g$ and control function $u \in \mathcal{U}$, where $\mathcal{U}$ denotes a set of admissible controls; see \Cref{sec: technical assumptions} for further details. Here, $X^u$ is a controlled SDE of the form
\begin{equation}
\label{eq: controlled SDE app}
    \dd X^u_s = (b + \sigma u)(X^u_s,s) \dd s + \sigma(s) \dd W_s, \quad X_0 \sim p_0,
\end{equation}
with base drift $b$, base distribution $p_0$ (typically a Gaussian or dirac delta distribution), and diffusion coefficient $\sigma$. We denote the path measure induced by \eqref{eq: controlled SDE app} by $\P^u \in \mathcal{P}$. Moreover, we simply write $\P$ for the path measure corresponding to the uncontrolled process, i.e.,
\begin{equation}
\label{eq:uncontrolledSDE app}
    \dd X_s = b(X_s,s) \dd s + \sigma(s) \dd W_s, \quad X_0 \sim p_0.
\end{equation}
Given a time $t$ and state $x$, the cost functional $J(u;x,t)$ is the expected cost-to-go for a control $u$ on the time interval $[t,T]$ and is defined as 
\begin{equation}
\label{eq: cost to go app}
    J(u;x,t) =  \E \left[ \int_t^T  \left( \tfrac{1}{2} \|u(X^u_s,s)\|^2 + f(X^u_s,s) \right)\, \dd s + g(X^u_T)\ \Big| \ X^u_t = x\right].
\end{equation}
The value function $V$, or, \textit{optimal cost-to-go} is obtained by taking the infimum over all controls in $\mathcal{U}$, that is,
\begin{equation}
\label{eq: value from J app}
    V(x,t) = \inf_{u \in \mathcal{U}} \ J(u;x,t).
\end{equation}
Then we have the following well-known results on representations of the value function $V$ and solution to the SOC problem $u^*$; see, e.g.,~\cite{nuesken2021solving,pavon1989stochastic,dai1991stochastic,fleming2006controlled,pham2009continuous} for details.
\highprio{JB: it would be good to present this result for more general HJB PDEs since we need this for the Lagrangian SOC problem later.}
\begin{theorem}[Optimality for SOC Problems]
\label{thm:soc_optimality}
Let us define the work functional as 
\begin{equation}
\label{eq: work functional}
    \mathcal{W}(X,t) = \int_t^T f(X_s,s) \, \dd s + g(X_T).
\end{equation}
Then we have the following representations of the value function $V$ in~\eqref{eq: value from J app} and the solution $u^*$ to the SOC problem in~\eqref{eq:soc_obj app}:
\begin{enumerate}
    \item \label{it:soc_sol} \textit{(Connection between solution and value function)} The solution can be written as
$ 
    u^* = - \sigma^{\top} \nabla V.
$ 
\item \label{it:opt_change_of_measure} \textit{(Optimal change of measure)} The Radon-Nikodym derivative of the optimal path measure $\Q$ w.r.t.\@ the uncontrolled path measure $\P$ satisfies
\begin{equation}
    \frac{\dd \Q}{\dd \P}(X) = \frac{e^{-\mathcal{W}(X,0)}}{\Z(X_0)} \quad \text{with} \quad \Z(X_0) = \E \big[e^{-\mathcal{W}(X,0)}|X_0\big].
\end{equation}
\item \label{it:hjb} \textit{(PDE for value function)} The value function $V$ is the solution to the Hamilton-Jacobi-Bellman (HJB) equation
\begin{align}
\label{eq: HJB equation}
    (\partial_t + L) V(x,t) - \tfrac{1}{2} \| (\sigma^{\top} \nabla V) (x,t) \|^2 + f(x,t) = 0, \quad V(x,T) = g(x),
\end{align}
where $L := \frac{1}{2} \sum_{i,j=1}^{d}  (\sigma \sigma^{\top})_{ij} \partial_{x_i} \partial_{x_j} + \sum_{i=1}^{d} b_i \partial_{x_i}$ denotes the infinitesimal generator of the uncontrolled SDE in~\eqref{eq:uncontrolledSDE app}.
\item \label{it:V} \textit{(Estimator for value function)} For every $(x,t)\in\R^d \times [0,T]$ the value function can be written as 
$
V(x,t) = - \log \mathbb{E} \big[ e^{-\mathcal{W}(X,t)}\big| X_t = x \big],
$
where $X$ is the solution of the uncontrolled SDE in~\eqref{eq:uncontrolledSDE app}.
\end{enumerate}
\end{theorem}
\highprio{JB: add proof
\begin{proof}
We provide a sketch of the main arguments and refer to~\cite{nuesken2021solving,pavon1989stochastic,dai1991stochastic,fleming2006controlled,pham2009continuous} for details.
The verification theorem states that if a function $V$ satisfies the HJB equation under appropriate regularity conditions, then $V$ corresponds to the value function \eqref{eq: value from J app} associated with the optimization problem \eqref{eq:soc_obj app}. Using Ito's formula for $V(X^u_s,s)$, a key consequence of this theorem is that for any control \( u \in \mathcal{U} \), it holds that
\begin{align} \label{eq:V_J appendix}
    V(x,t) + \mathbb{E} \left[ \frac{1}{2} \int_t^{T} \| \sigma^{\top} \nabla V + u \|^2 (X^u_s, s) \, \mathrm{d}s \, \bigg| \, X^u_t = x \right] = J(u;x,t).
\end{align}
This result implies that the optimal control is uniquely determined in terms of the value function and is given by  
\begin{equation}
\label{eq: optimal u from V appendix}
    u^*(x,t) = - \sigma(t)^{\top} \nabla V(x,t).
\end{equation}
Using forward-backward SDEs\footnote{Alternatively, one can use the Feynman-Kac formula for the Kolmogorov PDE arising from applying the Hopf-Cole transform $V \mapsto \exp(V)$ to the HJB equation.}, one can derive an explicit expression for the value function $V$ as
\begin{align} \label{eq:V appendix}
    V(x,t) = - \log \mathbb{E} \left[ e^{-\mathcal{W}(X,t)} \Big| X_t = x \right],
\end{align}
where $X_t$ is the solution of the uncontrolled SDE \eqref{eq:uncontrolledSDE app}.
\end{proof}}

Combining the expressions for $u^*$ and $V$ in~\Cref{thm:soc_optimality}, we directly obtain the path integral representation of the optimal control, i.e.,
\begin{align} \label{eq: path integral u appendix}
    u^*(x,t)  = \sigma(t)^{\top} \nabla_x \log \mathbb{E} \left[ e^{-\mathcal{W}(X,t)} \Big| X_t = x \right],
\end{align} 
In practice, computing the optimal control \eqref{eq: path integral u appendix} is typically impractical, as it requires running multiple simulations for each state $x$ to obtain a Monte Carlo approximation of the expectation; see~\Cref{app:related_works}. To address this challenge, many approaches instead learn a parameterized control function, optimized using stochastic gradient methods. These techniques are collectively referred to as iterative diffusion optimization (IDO) methods and are further discussed in the next section.

\subsection{Iterative diffusion optimization} \label{sec: iterative diffusion optimization}
An alternative view on problem \eqref{eq:soc_obj app} is obtained by considering loss functions on path measures \cite{nuesken2021solving}. By the Girsanov theorem (see~\Cref{sec: useful identities}) we have
\begin{equation}
\label{eq:change_of_measure_u_to_0 appendix}
    \frac{\dd \P}{\dd \P^{u}}(X^u)  = \exp\left({-\int_0^T u(X^u_s,s) \cdot \dd W_s - \frac{1}{2}\int_0^T \|u(X^u_s,s)\|^2 \dd s}\right).
\end{equation}
Combining this with the optimal change of measure $\dd \Q / \dd \P$ from~\Cref{thm:soc_optimality}, we obtain an expression for $\dd \Q /\dd \P^{u}$, from which we can compute the relative entropy $\mathcal{L}_{\mathrm{RE}}$, i.e., the reverse Kullback-Leibler (KL) divergence 
\begin{equation}
\label{eq:relative_entropy appendix}
    \mathcal{L}_{\mathrm{RE}}(u) = \KL(\P^u|\Q) = \E \left[ \int_0^T \left( \tfrac{1}{2} \|u(X^u_s,s)\|^2 + f(X^u_s,s) \right) \dd s + g(X^u_T) + \log \Z(X^u_0)\right].
\end{equation}
Note that minimizing the stochastic optimal control problem in \eqref{eq:soc_obj app} is equal to minimizing the KL divergence, that is,
\begin{equation}
\label{eq:re_equals_soc appendix}
    u^* = \argmin_{u \in \mathcal{U}} \ \mathcal{L}_{\mathrm{SOC}}(u) = \argmin_{u \in \mathcal{U}} \ \mathcal{L}_{\mathrm{RE}}(u),
\end{equation}
in the sense that both have the same unique optimal control $u^*$ as a minimizer. As such, we can consider an arbitrary divergence $D: \mathcal{P}\times \mathcal{P}\rightarrow \R^+$, for which $D(\P_1|\P_2) = 0$ holds if and only if $\P_1 = \P_2$, to solve stochastic optimal control problems. More generally, we can consider any loss function for which the unique minimizer is the optimal control $u^*$. Iterative diffusion optimization builds on this perspective and can be seen as a common framework for solving (potentially high-dimensional) SOC problems by leveraging parameterized control functions and stochastic gradient methods to minimize different loss functions. 

\subsection{On the initial value dependence of the normalizing constant}
\label{app: initial value dependence Z}

In general, the normalizing constant $\mathcal{Z}(X_0)$ in the optimal change of measure \eqref{eq:opt_change_of_measure} depends on the initial value $X_0$. Let us demonstrate in the following why this is the case. To this end, let us first assume a generic normalization constant $\Z$ that may or may not depend on $X_0$. As in \eqref{eq:opt_change_of_measure}, it then holds
\begin{equation}
    \frac{\dd \Q}{\dd \P}(X) = \frac{e^{-\mathcal{W}(X,0)}}{\Z}.
\end{equation}
We can then compute
\begin{equation}
    \frac{\Q_0(X_0)}{\P_0(X_0)} = \E\left[\frac{\dd \Q}{\dd \P}(X)\bigg|X_0\right] = \E\left[\frac{e^{-\mathcal{W}(X,0)}}{\Z}\bigg|X_0\right].
\end{equation}
Now, for a chosen $p_0 = \P_0$ we want that $\Q_0(X_0) = \P_0(X_0)$, which requires
\begin{equation}
    \Z = \E\left[{e^{-\mathcal{W}(X,0)}}\Big|X_0\right] = e^{-V(X_0,0)}.
\end{equation}
Clearly, the right-hand side depends on $X_0$. Hence, in general, $\Q_0(X_0) = \P_0(X_0)$ can only hold if $\Z$ depends on $X_0$. Conversely, if we wanted to have a global normalizing constant $\Z$, which is independent of $X_0$, we would need to tilt the initial marginal of $\Q$ as well, namely via
\begin{equation}
    \Q_0(X_0) = \P_0(X_0) \frac{\E[e^{-\mathcal{W}(X, 0)} | X_0]}{\Z} = \P_0(X_0) \frac{e^{-V(X_0, 0)}}{\Z}.
\end{equation}
However, the function $V(\cdot, 0)$ is typically not known in practice.

\section{Details on trust region SOC algorithms}
\label{app:trust_regions}

\subsection{Characterizing the solutions of the trust region optimization problem}
\label{app:solution_to_tr_opt}
\begin{proposition}[Characterizing the solutions of the trust region optimization problem]
    The solution $\P^{u_{i+1}}$ of the problem \eqref{eq:p_i_plus_one_p_i} is unique and it satisfies the following:
    \begin{itemize}
        \item If $
        \KL(\Q | \P^{u_i})
        \leq \varepsilon$, 
        then $\P^{u_{i+1}} = \Q$.
        \item If $
        \KL(\Q |\P^{u_i})
        \geq \varepsilon$, then $\KL(\P^{u_{i+1}} | \P^{u_i}) = \varepsilon$, i.e. $\P^{u_{i+1}}$ is also the unique solution of the problem
        \begin{equation} \label{eq:p_i_plus_one_p_i_equality}
            \argmin_{u \in \mathcal{U}} \KL\left(\P^u | \Q \right) \quad \text{s.t.} \quad \KL(\P^u | \P^{u_i}) = \varepsilon.
        \end{equation}
    \end{itemize}
\end{proposition}
\begin{proof}\vspace{-0.6em}
    To prove the first case, observe that $\Q$ is the only solution of the unconstrained problem $\argmin_{\P \in \mathcal{P}} \KL\left(\P | \Q \right)$, which means that it is also the unique solution of the problem \eqref{eq:p_i_plus_one_p_i} since it satisfies the constraint $\KL(\Q | \P^{u_i}) \leq \epsilon$. \highprio{JB: we need to argue that $\mathcal{U}$ is sufficiently expressive to contain the optimal control} To prove the second case, by the Karush-Kuhn-Tucker (KKT) conditions, we have that either $\lambda = 0$, or $\KL(\P^{u_{i+1}} | \P^{u_i}) = \varepsilon$. We assume that $\lambda = 0$ and $\KL(\P^{u_{i+1}} | \P^{u_i}) < \varepsilon$ to reach a contradiction, which will imply that $\KL(\P^{u_{i+1}} | \P^{u_i}) = \varepsilon$.
    The first-order optimality condition for the problem is as follows: for any perturbation $v$ of the control $u_{i+1}$, we have that
    \begin{align}
    \begin{split}
        0 &= \frac{\dd}{\dd \eta} \left( \KL\left(\P^{u_{i+1}+\eta v} | \Q \right) + \lambda\left( \KL(\P^{u_{i+1}+\eta v} | \P^{u_i} ) - \varepsilon \right) \right) \rvert_{\eta = 0} = \frac{\dd}{\dd \eta}  \KL\left(\P^{u_{i+1}+\eta v} | \Q \right) \rvert_{\eta = 0}, 
    \end{split}
    \end{align}
    which means that $u_{i+1}$ satisfies the first-order optimality condition for the relative entropy loss $u \mapsto \KL\left(\P^{u} | \Q \right)$. By \cite[Prop.~2]{domingoenrich2025adjoint}, the only control that satisfies the first-order optimality condition for the relative entropy loss is the optimal control $u^*$, which implies that $\P^{u_{i+1}} = \Q$, which yields a contradiction because $\varepsilon >  \KL(\P^{u_{i+1}} | \P^{u_i}) = 
    \KL\left(\Q | \P^{u_i} \right)
    \geq \varepsilon$. 
    
    Hence, we conclude that $\KL(\P^{u_{i+1}} | \P^{u_i}) = \varepsilon$. To show that the solution $\P^{u_{i+1}}$ is unique, we use that $\P \mapsto \KL(\P | \P^{u_i})$ is strictly convex, and that $\{ \P | \KL(\P | \P^{u_i}) \leq \varepsilon\}$ is a convex set because it is the sublevel set of a convex mapping.
\vspace{-0.1em}\end{proof}

\subsection{Implementation}
\label{app:implementation}
We provide a detailed version of~\Cref{alg:tr sampler} in~\Cref{alg:tr sampler app}. The hyperparameters and used repositories for the experiments on unnormalized densities, transition path sampling, and fine-tuning can be found in the respective sections in~\Cref{app:sampling,appendix:transitionpath,appendix:finetuning}.

\highprio{JB: answer this question here: For each experiment, does the paper provide sufficient information on the computer resources (type of compute workers, memory, time of execution) needed to reproduce the experiments?}

\begin{figure}[t]
\vspace{-0.5em}
\begin{algorithm}[H]
\caption{\small Trust Region SOC with buffer}
\label{alg:tr sampler app}
\small
\begin{algorithmic}
\Require Neural network $u_\theta$ with parameters $\theta$, target path measure $\Q$, buffer size $K$, time discretization $S=(s_j)_{j=0}^J \subset [0,T]$, 
number of gradient steps $M$ per trust region iteration, termination threshold $\delta$
\State Initialize $i=0$ and $\lambda_0 = \infty$
\For{$i=0,1,\dots$}
\State Define $u_i=u_\theta$ (detached)
\State Simulate $K$ trajectories $(X_s^{(k)})_{s\in S}$ of the SDE in~\eqref{eq:soc_intro} with Brownian motion $W_s^{(k)}$ and control $u_i$
\State Compute importance weights $w^{(k)} = \frac{\dd\Q}{\dd\P^{u_i}}(X^{(k)}) \propto \exp(-\mathcal{W}_i(X^{(k)},0))$ as in~\eqref{eq:rnd_adjacent}
\State Initialize buffer $\mathcal{B} = \big\{ (W^{(k)}, X^{(k)}, w^{(k)}) \big\}_{k=1}^K$ 
\State Compute multiplier $\lambda_i = \argmax_{\lambda \in \R^+} \ \mathcal{L}_{\mathrm{Dual}}^{(i)}(\lambda)$ as in~\eqref{eq:dual} using $\mathcal{B}$ and a $1$-dim. non-linear solver
\If{$\lambda_i \le \delta$}
\State \Return control $u_i$ with $\P^{u_i} \approx \Q$
\EndIf
\If{adjoint matching loss}
\State Compute annealing $\beta_{i+1} = 1-\prod_{j=0}^{i} \frac{\lambda_i}{1+\lambda_i}$ as in~\Cref{prop: Optimal change of measure as geometric annealing}
\State Compute lean adjoint states $a^{(k)}_s = a_{i+1}(X_s^{(k)}, s)$, $s \in S$, as in~\eqref{eq:lean_adjoint_state} and store in $\mathcal{B}$
\EndIf
\For{$m=1,\dots,M$}
\If{adjoint matching loss}
\State Estimate $\mathcal{L}(\theta) = \E_{(X,w,a) \sim \mathcal{B}, \ s\sim \mathrm{Unif}(S)} \big[ \|
\sigma^{\top} a_s - u_\theta(X_s,s) \|^2 w^{\frac{1}{1 + \lambda_i}} \big] $ as in~\eqref{eq:lean_adjoint_matching} 
\EndIf
\If{log-variance loss}
\State Estimate $\mathcal{L}(\theta) =  \V_{(W,X,w)\sim \mathcal{B}}\big[\sum_{j=1}^J \big( \tfrac{\| \Delta_j\|^2 (s_{j} - s_{j-1})}{2}  + \Delta_j \cdot (W_{s_j} - W_{s_{j-1}}) \big) + \frac{1}{1+\lambda_i}\log w\big]$ \\ \hspace{7em} with $\Delta_j = u_i(X_{s_j}, s_{j}) - u_\theta(X_{s_j}, s_{j})$ as in~\eqref{eq:trust_region_lv} 
\EndIf
\State Perform a gradient-descent step on $\mathcal{L}(\theta)$ 
\EndFor
\EndFor
\end{algorithmic}
\end{algorithm}
\vspace{-1.5em}
\end{figure}

\subsection{Variance of the importance weights and trust region bounds} 
\label{app:variance iws}

As mentioned in \Cref{rem: controlling the variance}, one motivation of the trust region constrain $\KL(\P^{u} | \P^{u_{i}}) \le \varepsilon$ defined in \eqref{eq: constrained optimization} is to keep the variance of the importance weights between two consecutive measures $\P^{u_{i}}$ and $\P^{u_{i+1}}$ small. This can be motivated by the inequality
\begin{subequations}
\label{eq: variance KL bounds computations}
\begin{align}
    \Var_{\P^{u_i}}\left( \frac{\mathrm d \P^{u_{i+1}}}{\mathrm d \P^{u_{i}}} \right) &= \E_{\P^{u_i}}\left[\left(\frac{\mathrm d \P^{u_{i+1}}}{\mathrm d \P^{u_{i}}}\right)^2 - 1 \right] =  \E_{\P^{u_{i+1}}}\left[\frac{\mathrm d \P^{u_{i+1}}}{\mathrm d \P^{u_{i}}} - 1 \right] \\
    &\ge \exp\left(\E_{\P^{u_{i+1}}}\left[\log \frac{\mathrm d \P^{u_{i+1}}}{\mathrm d \P^{u_{i}}}\right]\right) - 1 = \exp\left(\KL(\P^{u_{i+1}} | \P^{u_{i}})\right) - 1,
\end{align}
\end{subequations}
which follows by Jensen's inequality. While a lower bound on the variance is not straight forward for path space measures (cf. \cite{hartmann2024nonasymptotic}), we can consider the following heuristics. Let us assume that 
\begin{equation}
    \frac{\dd \P^{u_{i+1}}}{\dd \P^{u_i}} \approx 1
\end{equation}
$\P^{u_{i}}$- and $\P^{u_{i+1}}$-almost surely, which is reasonable if $\KL(\P^{u_{i+1}} | \P^{u_{i}}) \le \varepsilon$ with $\varepsilon \ll 1$. By a Taylor approximation it then holds
\begin{equation}
\label{eq: taylor approx}
    \left(\frac{\dd \P^{u_{i+1}}}{\dd \P^{u_i}}\right)^2 = \exp\left(2 \log\frac{\dd \P^{u_{i+1}}}{\dd \P^{u_i}}\right) \approx 1 + 2 \log \frac{\dd \P^{u_{i+1}}}{\dd \P^{u_i}}.
\end{equation}
Now, taking expectations w.r.t. $\P^{u_{i}} \approx \P^{u_{i+1}}$, respectively, using computations similar to \eqref{eq: variance KL bounds computations}, and assuming $\KL(\P^{u_{i+1}} | \P^{u_{i}}) = \varepsilon$, as argued in \Cref{app:solution_to_tr_opt}, yields
\begin{equation}
    \Var_{\P^{u_i}}\left( \frac{\mathrm d \P^{u_{i+1}}}{\mathrm d \P^{u_{i}}} \right) \approx 2 \varepsilon.
\end{equation}

\subsection{Lagragian formulation}
\label{app:lagragian}

Using the Girsanov theorem (see~\Cref{sec: useful identities}), we first note that we can write the Lagrangian as
\begin{align}
\label{eq:lagrangian appendix}
    \mathcal{L}_{\mathrm{TR}}^{(i)}(u, \lambda) & = 
     \E \left[ \int_0^T  \tfrac{1}{2} \|u(X^u_s,s)\|^2 \dd s + \mathcal{W}(X^u,0) + \log \Z(X^u_0)\right] + \lambda \left(\KL({\mathbb{P}}^u|{\mathbb{P}}^{u_i}) - \varepsilon \right)
    \\ & = \E\left[ \int_0^T \left( \tfrac{1}{2} \|u(X^u_s,s)\|^2 + \tfrac{\lambda}{2} \|u(X^u_s,s)-u_i(X^u_s,s)\|^2 \right) \dd s  +\mathcal{W}(X^u,0) + \log \Z(X^u_0) \right] -\lambda \varepsilon
    \\ & = 
    \E\left[ \int_0^T \left( \tfrac{1+\lambda}{2} \|u(X^u_s,s)-\tfrac{\lambda}{1+\lambda}u_i(X^u_s,s)\|^2  + f_i(X^u_s,s)\right) \dd s + g(X^u_T) + \log \Z(X^u_0) \right] - \lambda \varepsilon \\
    &= \mathcal{L}^{(i)}_{\mathrm{TRC}}(u, \lambda) - \lambda \varepsilon ,
\end{align}
where $\mathcal{L}^{(i)}_{\mathrm{TRC}}(u, \lambda)$ is defined as in~\eqref{eq: soc_tr}, $\lambda\in \R^+$ is the Lagrangian multiplier for the trust region constraint, and we abbreviate $f_i \coloneqq \frac{\lambda}{2(1+\lambda)}\|u_i\|^2 +  f$. 
For fixed $\lambda$, optimizing the Lagrangian $\mathcal{L}_{\mathrm{TR}}^{(i)}(u, \lambda)$ with respect to $u$ is again an SOC problem. As such, for given $u_i$ and $\lambda$, we can define the value function as 
\begin{equation}
    V^\lambda_{i+1}(x,t) = \inf_{u \in \mathcal{U}} \E\left[ \int_t^T \left( \tfrac{1+\lambda}{2} \|u(X^u_s,s)-\tfrac{\lambda}{1+\lambda}u_i(X^u_s,s)\|^2 + f_i(X^u_s,s) \right) \dd s + g(X^u_T) |X_t = x\right].
\end{equation}
The next proposition provides representations for the value function and the solution to the SOC problem.

\begin{proposition}[Optimality for trust region SOC problems]
\label{prop: tr_soc_optimality app}
For fixed $\lambda$, let us define by 
\begin{equation*}
    V^\lambda_{i+1}(x,t) \coloneqq \inf_{u \in \mathcal{U}} \E\left[ \int_0^T \left( \tfrac{1+\lambda}{2} \|u-\tfrac{\lambda}{1+\lambda}u_i\|^2 + \tfrac{\lambda}{2(1+\lambda)}\|u_i\|^2 +  f \right)(X^u_s,s) \, \dd s+ g(X^u_T) \Bigg| X_t=x \right]
\end{equation*}
the value function of the SOC problem $\inf_{u \in \mathcal{U}} \mathcal{L}^{(i)}_{\mathrm{TRC}}(u, \lambda)$ corresponding to~\eqref{eq: soc_tr} and by $u^\lambda_{i+1}$ its solution. Then it holds that
\begin{enumerate}[label=(\roman*)]
    \item \label{it:value_fn_app} \textit{(Estimator for value function)}  $V^{\lambda}_{i+1}(x,t)  = -(1+\lambda) \log \E\left[e^{-\tfrac{1}{1+\lambda}\mathcal{W}_i(X^{u_i},t)}\Big|X^{u_i}_t=x\right]$,
    \begin{equation*}
      \text{where} \quad  \mathcal{W}_i(X^{u_i},t) = \int_t^T \tfrac{1}{2}\|u_i(X^{u_i}_s,s)\|^2\dd s + \int_t^T u_i(X^{u_i}_s,s) \cdot \dd W_s + \mathcal{W}(X^{u_i},t). 
    \end{equation*}
    \item \label{it:solution_value_fn_app} (Connection between solution and value function) It holds $u^\lambda_{i+1} = \frac{\lambda}{1 + \lambda} u_i - \frac{1}{1 + \lambda} \sigma^\top \nabla V^\lambda_{i+1}$.
\end{enumerate}
Moreover, for $u_0=\mathbf{0}$ and the optimal Lagrange multiplier $\lambda_i$, let us define the value function 
\begin{equation*}
    \widetilde{V}_{i+1}(x,t) \coloneqq \inf_{u \in \mathcal{U}} \E\left[ \int_0^T \left( \tfrac{1}{2} \|u\|^2 +  \beta_{i+1} f \right)(X^u_s,s) \, \dd s+ \beta_{i+1} g(X^u_T) \Bigg| X_t=x \right]
\end{equation*}
of the SOC problem given by the optimal change of measure 
\begin{equation}
    \frac{\dd {\P}^{u_{i+1}}}{\dd{\mathbb{P}}}(X) = \frac{1}{\widetilde{\Z}_{i+1}(X_0)}\left(\frac{\dd \Q}{\dd{\mathbb{P}}}(X)\right)^{\beta_{i+1}} = \frac{e^{-\beta_{i+1}\mathcal{W}(X,0)}}{\widetilde{\Z}_{i+1}(X_0)}
\end{equation}
as in~\Cref{prop: Optimal change of measure as geometric annealing} and~\eqref{eq:opt_change_of_measure}, and $\widetilde{\Z}_{i+1}(X_0)$ as defined in \eqref{eq: unnormalized optimal iterates and log Z}. Then it holds that
\begin{enumerate}[label=(\roman*), start=3]
    \item \textit{(Estimator for value function)} \label{it:value_fn_app_rec_app} $\widetilde{V}_{i+1}(x,t) = -\log \E \left[e^{-\beta_{i+1}\mathcal{W}(X_t,t)}|X_t=x\right]$,
    \item \label{it:solution_value_fn_rec_app} (Connection between solution and value function) $u_{i+1} = u^{\lambda_i}_{i+1} = -\sigma^{\top}\nabla \widetilde{V}_{i+1}$.
\end{enumerate}

\end{proposition}
\begin{proof}
    For notational convenience, we abbreviate $\VV=V^\lambda_{i+1}$ in this proof. From the verification theorem (see, e.g.,~\cite[Theorem 3.5.2]{pham2009continuous}), we obtain that the value function is the solution to the HJB equation
\begin{subequations}
\begin{align}
    (\partial_t + L) \VV & = - \inf_{\alpha \in \R^d} \left\{f_i +
    \tfrac{1+\lambda}{2}\|\alpha - \tfrac{\lambda}{1+\lambda}u_i \|^2 + \sigma \alpha \cdot \nabla \VV
    \right\}
    \\ & = 
     - f_i - \inf_{\alpha \in \R^d} \left\{
    \tfrac{1+\lambda}{2}\|\alpha - \tfrac{\lambda}{1+\lambda}u_i \|^2 + \sigma \alpha \cdot \nabla \VV
    \right\},
    \quad \VV(\cdot, T) = g,
\end{align}
\end{subequations}
where the infimum is pointwise for every $(x,t)\in\R^d \times [0,T]$ and the optimal $\alpha^*$ defines the solution $u^*$. Solving for $\alpha$ yields $\alpha^* = \frac{\lambda}{1+\lambda}u_i - \frac{1}{1+\lambda} \sigma^\top \nabla \VV$, which proves~\Cref{it:solution_value_fn_app}.

Plugging this result back into the HJB equation, we obtain 
\begin{subequations}
\label{eq:hjb_closed_form appendix}
\begin{align}
    (\partial_t + L) \VV 
    & = 
    -f - \tfrac{\lambda}{2(1+\lambda)}\|u_i\|^2 
    - \tfrac{1}{2(1+\lambda)}\|\sigma^\top \nabla\VV\|^2 - \sigma \left(\tfrac{\lambda}{1+\lambda}u_i - \tfrac{1}{1+\lambda} \sigma^\top \nabla \VV\right) \cdot \nabla \VV 
    \\ & = 
    -f - \tfrac{\lambda}{2(1+\lambda)}\|u_i\|^2 
    + \tfrac{1}{2(1+\lambda)}\|\sigma^\top \nabla\VV\|^2 - \tfrac{\lambda}{1+\lambda} \sigma u_i \cdot  \nabla \VV 
    \\ & = 
        -f - \tfrac{\lambda}{2(1+\lambda)}\|u_i\|^2 
    + \tfrac{1}{2(1+\lambda)}\|\sigma^\top \nabla\VV\|^2 - \sigma u_i \cdot   \nabla \VV   +  \tfrac{1}{1+\lambda} \sigma u_i \cdot  \nabla \VV.
\end{align}
\end{subequations}
Now, we define the infinitesimal generator of the SDE
\begin{equation}
    \dd X^{u_i}_s = \left(b(X^{u_i}_s,s) + \sigma u_i(X^{u_i}_s,s)\right) \dd s + \sigma \dd W_s
\end{equation}
as
\begin{equation}
\label{eq: generator trust region sde appendix}
    \bar L \coloneqq \frac{1}{2}\sum_{i,j=1}^d (\sigma \sigma^{\top})_{ij} \partial_{x_i}\partial_{x_j} + \sum_{i=1}^d (b_i + (\sigma u_i)_i)\partial_{x_i} = L + \sum_{i=1}^d (\sigma u_i)_i \partial_{x_i}.
\end{equation}
Using \eqref{eq: generator trust region sde appendix}, we can rewrite \eqref{eq:hjb_closed_form appendix} as 
\begin{subequations}
 \label{eq:gen_bar appendix}
\begin{align}
    (\partial_t + \bar L) \VV  &=  -f - \tfrac{\lambda}{2(1+\lambda)}\|u_i\|^2 
    + \tfrac{1}{2(1+\lambda)}\|\sigma^\top \nabla\VV\|^2  + \sigma \tfrac{1}{1+\lambda}u_i  \cdot \nabla \VV \\
    &  =  -f - \tfrac{1}{2}\|u_i\|^2 
    + \tfrac{1}{2(1+\lambda)}\|u_i +\sigma^\top \nabla\VV\|^2 
\end{align}
\end{subequations}
By It\^o's formula (see~\Cref{sec: useful identities}), we have
\begin{equation}
    \label{eq:value_sde appendix}
    \dd \VV(X^{u_i}_s,s) = (\partial_s + \bar L) \VV(X^{u_i}_s,s) \mathrm d s + \sigma^\top \nabla \VV(X^{u_i}_s,s) \cdot \dd W_s.
\end{equation}
Plugging \eqref{eq:gen_bar appendix} into \eqref{eq:value_sde appendix} and defining $Y_s \coloneqq \VV(X^{u_i}_s, s)$ and $Z_s \coloneqq (-u_i -\sigma^\top \nabla \VV)(X^{u_i}_s, s)$, we obtain the pair of forward-backward SDEs (FBSDEs)
\begin{align}
\label{eq:fbsde1 appendix}
        \dd X^{u_i}_s & = \left(b(X^{u_i}_s,s) + \sigma u_i(X^{u_i}_s,s)\right) \dd s + \sigma(s) \dd W_s, \quad X^{u_i}_0 \sim p_0,
        \\
\label{eq:fbsde appendix}
        \dd Y_s & = \left(-f(X^{u_i}_s,s) - \tfrac{1}{2}\|u_i(X^{u_i}_s,s)\|^2+  \tfrac{1}{2(1+\lambda)} \|Z_s\|^2 \right)\dd s - (u_i(X^{u_i}_s,s) +Z_s) \cdot \dd W_s, 
\end{align}
with $Y_T = g(X^{u_i}_T)$. This shows that
\begin{align*}
    g(X^{u_i}_T) = Y_t - \int_t^T \left(f(X^{u_i}_s,s) + \tfrac{1}{2}\|u_i(X^{u_i}_s,s)\|^2 -  \tfrac{1}{2(1+\lambda)} \|Z_s\|^2 \right)\dd s - \int_t^T (u_i(X^{u_i}_s,s) + Z_s) \cdot \dd W_s,
\end{align*}
which can be rewritten as 
\begin{equation}
\label{eq:integrated_work}
    \mathcal{W}_i(X^{u_i},t) = Y_t + \int_t^T \tfrac{1}{2(1+\lambda)} \|Z_s\|^2 \dd s - \int_t^T Z_s \cdot  \dd W_s.
\end{equation}
Using the definition of $Y_t$, we can now write
\begin{align*}
    \E\left[e^{-\tfrac{1}{1+\lambda}\mathcal{W}_i(X^{u_i},t)}\Big|X^{u_i}_t=x\right] &= e^{-\tfrac{1}{1+\lambda}V(X^{u_i}_t, t)}\E\left[e^{\tfrac{1}{1+\lambda} \int_t^T  Z_s \cdot \dd W_s- \tfrac{1}{(1+\lambda)^2} \int_t^T \tfrac{1}{2} \|Z_s\|^2 \dd s}\Big|X^{u_i}_t=x\right] \\
    &= e^{-\tfrac{1}{1+\lambda}V(X^{u_i}_t, t)},
\end{align*}
\highprio{JB: A perhaps simpler proof is by showing that $e^{-\VV/(1+\lambda)}$ satisfies a Kolmogorov eq., using the FK formula for the drift $b+\lambda/(1+\lambda) \sigma u_i$, and then using Girsanov to the drift $b+\sigma u_i$.}
where we leveraged Novikov's theorem to show that the Doléans-Dade exponential is a martingale with vanishing expectation. This concludes the proof of~\Cref{it:value_fn_app}. The proof of~\Cref{it:value_fn_app_rec_app,it:solution_value_fn_rec_app} follows directly from~\Cref{thm:soc_optimality}.
\end{proof}

\section{Trust region SOC sequences and Fisher-Rao geometry}
\label{app:tr_limit}
For a fixed $\varepsilon$, suppose that we construct the sequence of controls $(u_{i+1})_{i \geq 0}$ as the solutions of the problem \eqref{eq: constrained optimization}. As shown in~\Cref{prop: Optimal change of measure as geometric annealing},
we have that  
\begin{equation} \label{eq: optimal com recursion general 2}
    \frac{\dd {\P}^{u_{i}}}{\dd{\mathbb{P}}} \propto \left(\frac{\dd \Q}{\dd{\mathbb{P}}}\right)^{
    \beta_i} 
     \left(\frac{\dd{\P}^{u_0}}{\dd {\mathbb{P}}}\right)^{1-
     \beta_i}, \quad \text{with} \quad \beta_i = 1-\prod_{j=0}^{i-1}\tfrac{\lambda_j}{1+\lambda_j}
\end{equation}
If we define the family $(\mathbb{Q}^{(\tau)})_{\tau \in [0,1]}$ such that
\begin{equation} \label{eq:Q_definition}
    \frac{\dd {\mathbb{Q}}^{(\tau)}}{\dd{\mathbb{P}}} \propto \left(\frac{\dd \Q}{\dd{\mathbb{P}}}\right)^{\tau} \left(\frac{\dd{\P}^{u_0}}{\dd {\mathbb{P}}}\right)^{1-\tau},
\end{equation}
we can write $\mathbb{P}^{u_{i}} = \mathbb{Q}^{(\beta_i)}$. Hence, we can regard the sequence $(\mathbb{P}^{u_{i}})_{i\geq 0}$ as a discretization of the family $(\mathbb{Q}^{(\tau)})_{\tau \in [0,1]}$. Next, we characterize this discretization more precisely using tools from information geometry.

\subsection{Basics on information geometry}
Let $\{p(x;\theta)\}_{\theta\in\Theta}$ be a parametric family of probability densities (or mass functions) on the sample space $\mathcal{X}$, and let $X$ be a random variable with distribution $p(x;\theta)$.
\begin{definition}[Fisher information matrix]
The \emph{Fisher information matrix} at $\theta$ is defined as
\[
\mathcal{I}(\theta)
\;=\;
\mathbb{E}_{X\sim p(\cdot;\theta)}\!\Bigl[
\nabla_\theta \log p(X;\theta)\,\bigl(\nabla_\theta \log p(X;\theta)\bigr)^\top
\Bigr]
\;=\;
-\,\mathbb{E}_{X\sim p(\cdot;\theta)}\!\Bigl[
\nabla^2_\theta \log p(X;\theta)
\Bigr],
\]
where $\nabla_\theta$ denotes the column gradient with respect to $\theta$, and $\nabla^2_\theta$ the Hessian.
\end{definition}
As an average of positive semi-definite matrices, $\mathcal{I}(\theta)$ is positive semi-definite, which makes it possible to define a geometric structure:
\begin{definition}[Statistical manifold]
Let $\{p(x;\theta)\}_{\theta\in\Theta}$ be a smooth parametric family of probability densities on $\mathcal X$, with parameter space~$\Theta\subseteq\mathbb R^d$.  Then $\Theta$ itself can be viewed as a $d$-dimensional differentiable manifold
\[
\mathcal M \;=\;\{\,p(\,\cdot\,;\theta):\theta\in\Theta\}\cong\Theta,
\]
called the \emph{statistical manifold} of the model.  Endow $\mathcal M$ with the Riemannian metric
\[
g_{ij}(\theta)
\;=\;
\mathcal I_{ij}(\theta)
\;=\;
\mathbb{E}_{X\sim p(\cdot;\theta)}
\Bigl[\partial_{i}\log p(X;\theta)\,\partial_{j}\log p(X;\theta)\Bigr],
\]
where $\partial_{i}=\frac{\partial}{\partial\theta_i}$.  This $g$ is known as the \emph{Fisher–Rao metric}, turning $(\mathcal M,g)$ into the canonical information‐geometric manifold of the model.
\end{definition}

Next, we review the definition of the length of a curve on a Riemannian manifold.

\begin{definition}[Length of a curve on a Riemannian manifold]
Let $(\mathcal{M},g)$ be a $d$-dimensional Riemannian manifold, and let $\gamma \colon [a,b]\;\longrightarrow\;\mathcal{M}$
be a piecewise smooth curve.  Choose local coordinates 
\(\theta=(\theta^1,\dots,\theta^d)\) on an open set \(\mathcal U\subset\mathcal M\) containing the image of \(\gamma\), so that 
\(\gamma(t)\mapsto \theta(t) = (\theta^1(t),\dots,\theta^d(t))\).  
Then the \emph{length} of \(\gamma\) is
\[
L(\gamma)
\;=\;
\int_a^b
\sqrt{g_{ij}\bigl(\theta(t)\bigr)\,\dot\theta^i(t)\,\dot\theta^j(t)\,}
\;dt,
\]
where 
\(\dot\theta^i(t)=\dfrac{d\theta^i}{dt}(t)\)
and we employ the Einstein summation convention on repeated indices \(i,j=1,\dots,d\).
\end{definition}

A geodesic between two points $\theta_1, \theta_2 \in \mathcal{M}$ is a piecewise smooth curve $\gamma \colon [a,b]\;\longrightarrow\;\mathcal{M}$ such that $\gamma(a) = \theta_1$, $\gamma(b) = \theta_2$ that minimizes the length functional $L$ locally. Any time reparameterization of a geodesic is also a geodesic, because the geodesic distance between $\theta_1, \theta_2$ is the infimum over the lengths of all geodesics (or all piecewise smooth curves) between $\theta_1, \theta_2$.

\begin{definition}[Fisher-Rao distance]
    The geodesic distance induced by the Fisher-Rao metric is known as the \emph{Fisher–Rao distance}.
\end{definition}

Lastly, we present another statement which connects the Kullback–Leibler divergence and the Fisher information matrix using a local expansion of the KL divergence.
\begin{proposition}[Second‐order expansion of $\mathrm{KL}$] \label{prop:2nd_order_expansion}
Let $\{p(x;\theta)\}_{\theta\in\Theta}$ be a smooth parametric family of densities, and fix $\theta\in\Theta$.  For a small increment $\delta\in\mathbb{R}^d$, consider
\[
\mathrm{KL}\bigl(p(\cdot;\theta+\delta)\,\big\|\,p(\cdot;\theta)\bigr)
\;=\;
\int_{\mathcal X}
p(x;\theta+\delta)\,
\log\frac{p(x;\theta+\delta)}{p(x;\theta)}
\,dx.
\]
Then one has the Taylor expansion
\[
\mathrm{KL}\bigl(p(\theta+\delta)\|p(\theta)\bigr)
=
\underbrace{0}_{\text{constant term}}
\;+\;
\underbrace{0}_{\text{linear term}}
\;+\;
\frac{1}{2}\,\delta^{i}\,\mathcal{I}_{ij}(\theta)\,\delta^{j}
\;+\;
O\bigl(\|\delta\|^3\bigr),
\]
where
\[
\mathcal{I}_{ij}(\theta)
\;=\;
\mathbb{E}_{X\sim p(\cdot;\theta)}
\bigl[\partial_{i}\log p(X;\theta)\,\partial_{j}\log p(X;\theta)\bigr]
\]
is the Fisher information matrix.  Equivalently,
\[
\left.\frac{\partial \mathrm{KL}}{\partial \delta^{i}}\right|_{\delta=0}
=0,
\qquad
\left.\frac{\partial^2 \mathrm{KL}}{\partial \delta^{i}\partial \delta^{j}}\right|_{\delta=0}
=\mathcal{I}_{ij}(\theta).
\]
\end{proposition}
\begin{proof}[Sketch]
Expand both $p(x;\theta+\delta)$ and $\log p(x;\theta+\delta)$ to second order in $\delta$, substitute into the integral, and use
\(\int p\,\partial_{i}\log p\, \mathrm dx=0\)
and
\(\int p\,\partial_{i}\partial_{j}\log p\, \mathrm dx=-\mathcal{I}_{ij}(\theta)\)
to verify cancellation of constant and linear terms, leaving the stated quadratic form.
\vspace{-0.1em}\end{proof}

\subsection{Fisher–Rao geometry of an exponential family}
\begin{definition}[The exponential‐family manifold]
Let 
\[
p(x;\theta)
\;=\;
\exp\bigl(\theta^i T_i(x)\;-\;A(\theta)\bigr)\,h(x),
\quad 
\theta=(\theta^1,\dots,\theta^d)\in\Theta\subseteq\mathbb{R}^d
\]
be a regular \(d\)-parameter exponential family on \(\mathcal X\).  The parameter space \(\Theta\) (equipped with the atlas coming from the coordinates \(\theta^i\)) is a \(d\)-dimensional differentiable manifold, which we identify with the statistical model
\[
\mathcal M
\;=\;
\bigl\{\,p(\,\cdot\,;\theta)\mid \theta\in\Theta\bigr\}.
\]
Its tangent space at \(\theta\) is $T_\theta\mathcal M \;\cong\;\mathbb{R}^d$,
with basis \(\{\partial/\partial\theta^i\}\).
\end{definition}

\begin{definition}[Fisher–Rao metric]
The Fisher–Rao metric on \(\mathcal M\) is the Riemannian metric whose components in the natural coordinate chart \(\theta\) are
\[
g_{ij}(\theta)
\;=\;
\mathbb{E}_{X\sim p(\cdot;\theta)}
\bigl[\partial_i\log p(X;\theta)\,\partial_j\log p(X;\theta)\bigr]
\;=\;
-\mathbb{E}_{X\sim p(\cdot;\theta)}
\bigl[\partial_{ij}\log p(X;\theta)\bigr]
\;=\;
\frac{\partial^2 A(\theta)}{\partial\theta^i\,\partial\theta^j}.
\]
Equivalently,
\(\,g(\theta)=\nabla^2 A(\theta)\), the Hessian of the log‐partition function.
\end{definition}
For general exponential families, the Fisher-Rao distance and the geodesics do not admit a closed form. Yet, one-dimensional families can be handled explicitly, because geodesics are trivial:
\begin{proposition}[One‐parameter exponential family] \label{prop:one_parameter_exponential}
If \(d=1\) then \(\theta\in(a,b)\subseteq\mathbb{R}\), and
\(\;g(\theta)=A''(\theta)\).  Hence
\[
\mathrm{FR}(\theta_1,\theta_2)
\;=\;
\Bigl|\!\int_{\theta_1}^{\theta_2}\!
\sqrt{A''(\theta)}\,\mathrm d\theta\Bigr|.
\]
\end{proposition}

\subsection{Fisher–Rao geometry of an exponential family of path measures}
We can view the family $(\mathbb{Q}^{(\tau)})_{\tau \in [0,1]}$ defined in \eqref{eq:Q_definition} as a one-parameter exponential family \cite{brekelmans2020all} by rewriting $\mathbb{Q}^{(\tau)}$ as
\begin{equation} \label{eq:RDN_Q_P0}
    \frac{\dd {\mathbb{Q}}^{(\tau)}}{\dd{\P^{u_0}}}  = \exp \left( \tau \bigg( 
    \log \frac{\dd \Q}{\dd{\P^{u_0}}}
    \bigg) - A(\tau)
    \right), 
\end{equation}
where the log-partition function \(A(\tau)\) is defined as
\begin{equation}
    A(\tau) = \log \mathbb{E}_{\P^{u_0}} \bigg[  \left(
    \frac{\dd \Q}{\dd{\P^{u_0}}}
    \right)^{\tau} \bigg].
\end{equation}
Equivalently, we can write it as an exponential family centered on an arbitrary $\tau \in [0,1]$:
\begin{equation}
    \frac{\dd {\mathbb{Q}}^{(\tau+\Delta \tau)}}{\dd{\mathbb{Q}}^{(\tau)}}  =\exp \left( \Delta \tau \bigg( 
    \log \frac{\dd \Q}{\dd{\P^{u_0}}}
    \bigg) - A_{\tau}(\Delta \tau) \right), 
\end{equation}
where
\begin{equation}
    A_{\tau}(\Delta \tau) := \log \mathbb{E}_{{\mathbb{Q}}^{(\tau)}} \bigg[  \left(
    \frac{\dd \Q}{\dd{\P^{u_0}}}
    \right)^{\Delta \tau} \bigg].
\end{equation}
\paragraph{Deriving an expression for the Fisher information.} Observe that by construction 
\begin{align}
\begin{split}
    A_{\tau}(\Delta \tau) \coloneqq \log \mathbb{E}_{{\P}^{u_0}} \bigg[  \left(
    \frac{\dd \Q}{\dd{\P^{u_0}}}
    \right)^{\Delta \tau} \frac{\dd {\mathbb{Q}}^{(\tau)}}{\dd{\P^{u_0}}} \bigg]  &= \log \mathbb{E}_{{\P}^{u_0}} \bigg[  \left(
    \frac{\dd \Q}{\dd{\P^{u_0}}}
    \right)^{\Delta \tau} \left(
    \frac{\dd \Q}{\dd{\P^{u_0}}}
    \right)^{\tau} \exp \big( - A(\tau) \big) \bigg] \\ &= A(\tau + \Delta \tau) - A(\tau),
\end{split}
\end{align}
which means that $A_{\tau}'(0) = A'(\tau)$ for all $\tau \in (0,1)$.
Thus, by Prop.~\ref{prop:one_parameter_exponential}, we conclude that the Fisher information matrix, which is a scalar because the manifold is one-dimensional, reads
\begin{equation}
    \mathcal{I}(\tau) = A''(\tau) = A_{\tau}''(0).
\end{equation}
Computing the first and second derivatives of $A_{\tau}$ is straight-forward: 
\begin{align}
\begin{split}
    A'_{\tau}(\Delta \tau) &= \frac{\mathbb{E}_{{\mathbb{Q}}^{(\tau)}} \bigg[  \log \left(
    \frac{\dd \Q}{\dd{\P^{u_0}}}
    \right) \left(
    \frac{\dd \Q}{\dd{\P^{u_0}}}
    \right)^{\Delta \tau} \bigg]}{\mathbb{E}_{{\mathbb{Q}}^{(\tau)}} \bigg[  \left(
    \frac{\dd \Q}{\dd{\P^{u_0}}}
    \right)^{\Delta \tau} \bigg]}, \\
    A''_{\tau}(0) &= \mathbb{E}_{{\mathbb{Q}}^{(\tau)}} \bigg[  \log \left(
    \frac{\dd \Q}{\dd{\P^{u_0}}}
    \right)^2 \bigg] - \mathbb{E}_{{\mathbb{Q}}^{(\tau)}} \bigg[  \log \left(
    \frac{\dd \Q}{\dd{\P^{u_0}}}
    \right) \bigg]^2,
\end{split}    
\end{align}
and this implies that 
\begin{align}
\begin{split} \label{eq:FI}
    \mathcal{I}(\tau) = \mathrm{Var}_{\mathbb{Q}^{(\tau)}} \bigg[ \log \bigg( \frac{\dd \Q}{\dd{\P^{u_0}}} \bigg) \bigg].
\end{split}
\end{align}
\paragraph{Connecting the trust region constraint to the Fisher information.} Applying Proposition \ref{prop:2nd_order_expansion}, we obtain that
\begin{equation}
    \mathrm{KL}\big({\mathbb{Q}}^{(\tau + \Delta \tau)}|{\mathbb{Q}}^{(\tau)} \big) = \frac{\Delta \tau^2}{2} \mathcal{I}(\tau) + O(\Delta \tau^3),
\end{equation}
When we set $\tau + \Delta \tau = \beta_{i+1}$, $\tau = \beta_i$, we have that
\begin{equation}
\label{eq: trust region KL expansion}
    \varepsilon = \mathrm{KL}\big({\P}^{u_{i+1}}|{\P}^{u_i} \big) = \frac{\Delta \tau^2}{2} \mathcal{I}(\tau) + O(\Delta \tau^3).
\end{equation}
Thus,
\begin{equation}
\label{eq: analytical annealing update}
    \Delta \tau = \sqrt{\frac{2 \varepsilon}{\mathcal{I}(\tau)}} + O(\Delta \tau^{3/2}),
\end{equation}
Moreover, the Fisher-Rao distance between $\P^{u_0}$ and $\mathbb{P}^{(i)}$, or rather, between $0$ and $\beta_i$,  
\begin{equation}
    \mathrm{FR}(0,\beta_i) = \int_0^{\beta_i} \sqrt{\mathcal{I}(\tau)} \dd \tau.
\end{equation}
Then, the difference between Fisher-Rao distances $\mathrm{FR}(0,\beta_{i+1})$ and $\mathrm{FR}(0,\beta_i)$ which is equal to the Fisher-Rao distance $\mathrm{FR}(\beta_i,\beta_{i+1})$ is
\begin{align}
\begin{split}
    &\mathrm{FR}(0,\beta_{i+1}) - \mathrm{FR}(0,\beta_i) = \mathrm{FR}(\beta_i,\beta_{i+1}) = \int_{\beta_i}^{\beta_{i+1}} \sqrt{\mathcal{I}(\tau)} \dd \tau \\ &= \big( \sqrt{\mathcal{I}(\beta_i)} + O(\beta_{i+1} - \beta_i) \big) (\beta_{i+1} - \beta_i) = \sqrt{\mathcal{I}(\beta_i)} \Delta \tau + O(\Delta \tau^2) \\ &= \sqrt{\mathcal{I}(\beta_i)} \sqrt{\frac{2 \varepsilon}{\mathcal{I}(\beta_i)}} + O(\Delta \tau^{3/2}) = \sqrt{2\varepsilon} + O(\Delta \tau^{3/2}).
\end{split}
\end{align}
In continuous time, we have a curve $\beta : \mathbb{R}^{>0} \to [0,1]$, and
\begin{align}
\begin{split}
    \frac{\dd}{\dd t} \mathrm{FR}(0,\beta(t)) = \sqrt{\mathcal{I}(\beta(t))} \beta'(t) =  \sqrt{\mathcal{I}(\beta(t))} \sqrt{\frac{2}{\mathcal{I}(\beta(t))}} = \sqrt{2} 
\end{split}    
\end{align}
Thus, we have shown the following result:
\begin{proposition}
Up to high order terms, the elements of sequence $(\mathbb{P}^{u_{i}})_{0 \leq i \leq I-1}$ are equispaced in the Fisher-Rao distance. The last term $\mathbb{P}^{u_{I}}$ is equal to the target distribution $\mathbb{Q}$.
\end{proposition}

\paragraph{A Monte Carlo estimate for the Fisher information.} By equation \eqref{eq:RDN_Q_P0}, we have that $\log \frac{\dd {\mathbb{Q}}^{(\tau)}}{\dd{\P^{u_0}}} = \tau \big( \log \frac{\dd \Q}{\dd{\P^{u_0}}} \big) - A(\tau)$. Hence, we can rewrite \eqref{eq:FI} as
\begin{equation}
    \mathcal{I}(\tau) = \frac{1}{\tau^2} \mathrm{Var}_{\mathbb{Q}^{(\tau)}} \bigg[ \log \bigg( \frac{\dd {\mathbb{Q}}^{(\tau)}}{\dd{\P^{u_0}}} \bigg) \bigg],
\end{equation}
which provides a way to estimate $\mathcal{I}(\tau)$, leveraging the Girsanov theorem to estimate $\log \big( \frac{\dd {\mathbb{Q}}^{(\tau)}}{\dd{\P^{u_0}}} \big) = \log \big( \frac{\dd {\mathbb{Q}}^{(\tau)}}{\dd{\mathbb{P}}} \big) - \log \big( \frac{\dd{\P^{u_0}}}{\dd{\mathbb{P}}} \big)$.

\begin{remark}[Analytical computation of annealing sequence]
    Using $\tau + \Delta \tau = \beta_{i+1}$, $\tau = \beta_i$ paired with
    \eqref{eq: analytical annealing update} and \eqref{eq:FI} we can analytically compute the annealing sequence $(\beta_i)_i$, up to high order terms, as
    \begin{equation}
        \beta_{i+1} = \beta_i + \sqrt{\frac{2 \varepsilon}{\mathcal{I}(\beta_i)}} + O((\beta_{i+1}-\beta_i)^{3/2})
        ,
    \end{equation}
    with 
    \begin{equation}
        \mathcal{I}(\beta_i) = \mathrm{Var}_{\P^{u_i}} \bigg[ \log \bigg( \frac{\dd \Q}{\dd{\P^{u_0}}} \bigg) \bigg],
    \end{equation}
    where we used that $\mathbb{Q}^{(\tau)} = \P^{u_i}$.
\end{remark}

\section{Trust region SOC losses}
In this section, we provide a non-exhaustive list of losses that can be readily applied to solve SOC problems within our trust region framework. More specifically, we aim for minimizing a divergence $D: \mathcal{P} \times \mathcal{P} \rightarrow \R^+$ between the path measures induced by the control $u_{i+1}$ and the learnable control $u$. For a comprehensive overview of SOC losses without trust regions, see \cite{domingo2024taxonomy}. 

\subsection{Log-variance loss}
\label{app:lv_tr}
 Here, we provide further details on the log-variance loss \cite{richter2020vargrad,nuesken2021solving,richter2023improved} within our trust region framework. The log-variance loss is defined as
\begin{equation}
    \mathcal{L}_{\mathrm{LV}}(u) = \V\left[
    \log \left(\frac{\dd \P^{u_{i+1}}}{\dd \P^{u}}(X^{w})  \right)
    \right]  \quad \text{with} \quad X^w \sim \P^w,
\end{equation}
where $X^w$ is defined as in \eqref{eq:soc_intro}, with $u$ replaced by $w \in \mathcal{U}$, referred to as the \textit{reference process}. Although the choice of $w$ is arbitrary, we discuss two particularly suitable options that facilitate sample reuse with replay buffers in combination with trust regions.

\paragraph{Using $w = u_i$ as reference control.} First, we replace the reference control $w$ with the control function of the previous iteration $u_i$. Thus, the log-variance loss becomes
\begin{equation}
    \mathcal{L}_{\mathrm{LV}}(u) = \V\left[
    \log \left(\frac{\dd \P^{u_{i+1}}}{\dd \P^{u}}(X^{u_i})  \right)
    \right]  = \V\left[
    \log \left(\frac{\dd \P^{u_{i+1}}}{\dd \P^{u_i}} (X^{u_i}) \frac{\dd\P^{u_{i}}}{\dd\P^{u}}(X^{u_i}) \right)
    \right].
\end{equation}
The Girsanov theorem (see~\Cref{sec: useful identities}) shows that
\begin{equation}
\label{eq:lv_appendix_girsanov}
     \log \left(\frac{\dd\P^{u_{i}}}{\dd\P^{u}}(X^{u_i})\right)  =\frac{1}{2} \int_0^T \| u_i(X^{u_i}_s,s)-u(X^{u_i}_s,s)\|^2\dd s + \int_0^T (u_i-u)(X^{u_i}_s,s) \cdot \dd W_s.
\end{equation}
Combining this result for $u=0$ with \Cref{prop: Optimal change of measure as geometric annealing}, we obtain 
\begin{subequations}
\label{eq:lv_appendix_rnd}
\begin{align}
    \frac{\dd \P^{u_{i+1}}}{\dd \P^{u_i}}(X^{u_i})  &\propto  \left(\frac{\dd \Q}{\dd \P}\frac{\dd \P}{\dd \P^{u_i}}(X^{u_i})\right)^{{\frac{1}{1+\lambda_i}}}  \\ & = e^{-\frac{1}{1+\lambda_i}\left( \int_0^T \frac{1}{2}\| u_i(X^{u_i}_s,s)\|^2\dd s + \int_0^T u_i(X^{u_i}_s,s) \cdot \dd W_s + \mathcal{W}(X^{ u_i},0)\right)}.
\end{align}
\end{subequations}
Noting that the variance is shift-invariant, \eqref{eq:lv_appendix_girsanov} and~\eqref{eq:lv_appendix_rnd} imply that
\begin{align}
\begin{split}
\mathcal{L}_{\mathrm{LV}}(u) = \V\Bigg[ &
-\frac{1}{1+\lambda_i}\left(\frac{1}{2} \int_0^T \| u_i(X^{u_i}_s,s)\|^2\dd s + \int_0^T u_i(X^{u_i}_s,s) \cdot \dd W_s + \mathcal{W}(X^{ u_i},0)\right)
\\ & +
\frac{1}{2} \int_0^T \| u_i(X^{u_i}_s,s)-u(X^{u_i}_s,s)\|^2\dd s + \int_0^T (u_i-u)(X^{u_i}_s,s) \cdot \dd W_s   
\Bigg],
\end{split}
\end{align}
which can be implemented by discretizing the integrals; see~\Cref{app:implementation}.

Please note that the loss reduces to 
\begin{equation}
    \mathcal{L}_{\mathrm{LV}}(u) = \V\left[
    \log \left(\frac{\dd \Q}{\P^{u_i}} (X^{u_i}) \frac{\dd\P^{u_{i}}}{\dd\P^{u}}(X^{u_i}) \right)
    \right] = \V\left[
    \log \left(\frac{\dd \Q}{\dd \P^{u}}(X^{u_i})  \right)
    \right]
\end{equation}
for $\lambda_i = 0$, which is how the loss is mostly used in the literature, where the variance is computed using the most recent control, see e.g. \cite{richter2023improved}.

\paragraph{Using $w=u_{i+1}$ as reference control.} Alternatively, when setting the reference control to the optimal control $u_{i+1}$, the log-variance loss is given by
\begin{subequations}
\label{eq: lv optimal ref}
\begin{align}
    \mathcal{L}_{\mathrm{LV}}(u) & = \V\left[
    \log \left(\frac{\dd \P^{u_{i+1}}}{\dd \P^{u}}(X^{u_{i+1}})  \right)
    \right] = 
    \\ & =
    \E\left[
    \log \left(\frac{\dd \P^{u_{i+1}}}{\dd \P^{u}}(X^{u_{i+1}})  \right)^2
    \right] - \left( \E\left[
    \log \left(\frac{\dd \P^{u_{i+1}}}{\dd \P^{u}}(X^{u_{i+1}})  \right)
    \right] \right)^2
        \\ & =
    \E\left[
    \log \left(\frac{\dd \P^{u_{i+1}}}{\dd \P^{u}}(X^{u_{i}})  \right)^2
    \frac{\dd \P^{u_{i+1}}}{\dd \P^{u_i}}(X^{u_{i}})\right] - \left( \E\left[
    \log \left(\frac{\dd \P^{u_{i+1}}}{\dd \P^{u}}(X^{u_{i}})  \right)
    \frac{\dd \P^{u_{i+1}}}{\dd \P^{u_i}}(X^{u_{i}})\right]\ \right)^2
\end{align}
    
\end{subequations}
which can be computed using \eqref{eq:lv_appendix_girsanov} and \eqref{eq:lv_appendix_rnd}. Hence, in contrast to using $w=u_i$, \eqref{eq: lv optimal ref} additionally incorporates the smoothed importance weights ${\dd \P^{u_{i+1}}}/{\dd \P^{u_i}}$.

\subsection{Moment loss}
The \textit{moment loss} was introduced in \cite{hartmann2019variational} and is defined as 
\begin{align}
    \mathcal{L}_{\mathrm{moment}}(u) & =\E\left[\log \left(\frac{\dd \P^{u_{i+1}}}{\dd \P^{u}}(X^{w})  \right)^2\right] \quad \text{with} \quad X^w \sim \P^w,
\end{align}
where $w\in\mathcal{U}$ is again an arbitrary reference control; see \ref{app:lv_tr}. We  distinguish again between $u = u_i$ and $w = u_{i+1}$.

\paragraph{Using $w = u_i$ as reference control.} In this case, $\mathcal{L}_{\mathrm{moment}}$ becomes 
\begin{align}
    \mathcal{L}_{\mathrm{moment}}(u) & =\E\left[\log \left(\frac{\dd \P^{u_{i+1}}}{\dd \P^{u}}(X^{u_i})  \right)^2\right] = \E\left[ 
    \log \left(\frac{\dd \P^{u_{i+1}}}{\dd \P^{u_i}} (X^{u_i}) \frac{\dd\P^{u_{i}}}{\dd\P^{u}}(X^{u_i}) \right)^2
    \right],
\end{align}
with 
\begin{equation}
\label{eq: moment app Z}
    \frac{\dd \P^{u_{i+1}}}{\dd \P^{u_i}}(X^{u_i}) 
    =\frac{e^{-\tfrac{1}{1+\lambda_i}\mathcal{W}_i(X^{u_i},0)}}{\Z_{i+1}} \quad \text{with} \quad \Z_{i+1} = \E\left[ e^{-\tfrac{1}{1+\lambda_i}\mathcal{W}_i(X^{u_i},0)}\right].
\end{equation}
Contrary to the log-variance loss, the moment loss is not shift-invariant, thus requiring $\Z_{i+1}$ which is commonly not available. As such, \cite{hartmann2019variational} proposes to treat $\Z_{i+1}$ as a learnable parameter. 
Using~\eqref{eq:lv_appendix_girsanov} and~\eqref{eq:lv_appendix_rnd} imply that
\begin{align}
\begin{split}
\mathcal{L}_{\mathrm{moment}}&(u,\Z_{i+1}) = \E\Bigg[ 
\Bigg(-\frac{1}{1+\lambda_i}\left(\frac{1}{2} \int_0^T \| u_i(X^{u_i}_s,s)\|^2\dd s + \int_0^T u_i(X^{u_i}_s,s) \cdot \dd W_s + \mathcal{W}(X^{ u_i},0)\right)
\\ & +
 \frac{1}{2} \int_0^T\| u_i(X^{u_i}_s,s)-u(X^{u_i}_s,s)\|^2\dd s + \int_0^T (u_i-u)(X^{u_i}_s,s) \cdot \dd W_s -\log \Z_{i+1}\Bigg)^2   
\Bigg],
\end{split}
\end{align}
which is optimized as $\min_{u\in\mathcal{U}, \ \Z_{i+1} \in \R} \mathcal{L}_{\mathrm{moment}}(u,\Z_{i+1})$.

\paragraph{Using $w = u_{i+1}$ as reference control.} Using $w = u_{i+1}$ yields 
\begin{align}
    \mathcal{L}_{\mathrm{moment}}(u,\Z_{i+1}) & =\E\left[\log \left(\frac{\dd \P^{u_{i+1}}}{\dd \P^{u}}(X^{u_{i+1}})  \right)^2\right] 
    \\ & = \E\left[ 
    \log \left(\frac{\dd \P^{u_{i+1}}}{\dd \P^{u_i}} (X^{u_i}) \frac{\dd\P^{u_{i}}}{\dd\P^{u}}(X^{u_i}) \right)^2\frac{\dd \P^{u_{i+1}}}{\dd \P^{u_i}} (X^{u_i}) 
    \right],
\end{align}
where ${\dd \P^{u_{i+1}}}/{\dd \P^{u_i}}$ depends on $\Z_{i+1}$, see \eqref{eq: moment app Z}.
Hence, the difference between using $w=u_i$ and $w=u_{i+1}$ lies in the additional importance weights ${\dd \P^{u_{i+1}}}/{\dd \P^{u_i}}$.

\subsection{Cross-entropy loss}
The cross entropy loss is defined as the forward KL divergence between $u_{i+1}$ and $u$, i.e.,
\begin{subequations}
\begin{align}
    \mathcal{L}_{\mathrm{CE}}(u) & = 
    D_{\mathrm{KL}}\left(\P^{u_{i+1}}|{\P}^{u}\right) 
    = 
    \E\left[\log \frac{\dd \P^{u_{i+1}}}{\dd{\P}^{u}}(X^{u_{i+1}})\right]
    \\ & 
    = 
    \E\left[\log \left(\frac{\dd \P^{u_{i+1}}}{\dd \P^{u_i}} (X^{u_{i+1}}) \frac{\dd\P^{u_{i}}}{\dd\P^{u}}(X^{u_{i+1}}) \right)\right]
    \\ & = 
    \E\left[\log \left(\frac{\dd \P^{u_{i+1}}}{\dd \P^{u_i}} (X^{u_i}) \frac{\dd\P^{u_{i}}}{\dd\P^{u}}(X^{u_i}) \right)\frac{\dd \P^{u_{i+1}}}{\dd \P^{u_i}} (X^{u_i})\right].
\end{align}
\end{subequations}
Using~\eqref{eq:lv_appendix_girsanov} and~\eqref{eq:lv_appendix_rnd} implies that
\begin{align}
\begin{split}
&\mathcal{L}_{\mathrm{CE}}(u) = \E\Bigg[ \Bigg(
-\frac{1}{1+\lambda_i}\left( \frac{1}{2} \int_0^T \| u_i(X^{u_i}_s,s)\|^2\dd s + \int_0^T u_i(X^{u_i}_s,s) \cdot \dd W_s + \mathcal{W}(X^{ u_i},0)\right)
\\ & +
\frac{1}{2} \int_0^T \| u_i(X^{u_i}_s,s)-u(X^{u_i}_s,s)\|^2\dd s + \int_0^T (u_i-u)(X^{u_i}_s,s) \cdot \dd W_s \Bigg)  \frac{\dd \P^{u_{i+1}}}{\dd \P^{u_i}} (X^{u_i})
\Bigg] - \log \Z,
\end{split}
\end{align}
with importance weights ${\dd \P^{u_{i+1}}}/{\dd \P^{u_i}}$.

\subsection{Stochastic optimal control matching via adjoint method }
\label{app:tr_adjoint_matching}
Here, we provide further details on the trust region version of the stochastic optimal control matching (SOCM) loss introduced in \cite{domingoenrich2024stochastic}.
We start from the cross-entropy loss, i.e., the forward KL divergence between $u_{i+1}$ and $u$, that is,
\begin{equation}
    \mathcal{L}_{\mathrm{CE}}(u) = 
    D_{\mathrm{KL}}\left(\P^{u_{i+1}}|{\P}^{u}\right) = 
    \E\left[\log \frac{\dd \P^{u_{i+1}}}{\dd{\P}^{u}}(X^{u_{i+1}})\right].
\end{equation}
Using Girsanov's theorem (see~\Cref{sec: useful identities}), the cross-entropy loss can be written as 
\begin{subequations}
\begin{align}
    \mathcal{L}_{\mathrm{CE}}(u) & = 
    \E\left[\frac{1}{2}\int_0^T \|u_{i+1}(X^{u_{i+1}}_s,s) - u(X^{u_{i+1}}_s,s)\|^2 \dd s\right]
    \\ & =
   \E\left[\frac{1}{2}\int_0^T \|u_{i+1}(X^{u_i}_s,s) - u(X^{u_i}_s,s)\|^2 \dd s \frac{\dd \P^{u_{i+1}}}{\dd \P^{u_i}} \right].
\end{align}
\end{subequations}

Using the expression for the optimal control, $u_{i+1} = \frac{\lambda_i}{1 + \lambda_i}u_i - \frac{1}{1 + \lambda_i}\nabla V_{i+1}$, see \Cref{prop: tr_soc_optimality}, yields
\begin{align}
    {\mathcal{L}}_{\mathrm{CE}}(u) & =
    \E\left[\frac{1}{2}\int_0^T \|\tfrac{\lambda_i}{1+\lambda_i}u_i(X_s^{u_i},s) - \tfrac{1}{1+\lambda_i} \sigma^{\top} \nabla V_{i+1}(X_s^{u_i},s) - u(X_s^{u_i},s)\|^2 \dd s \frac{\dd \P^{u_{i+1}}}{\dd \P^{u_i}} \right] 
\end{align}
with
\begin{equation}
\label{eq:gamma_socm}
      \nabla_x V_{i+1}(x,t) = -(1+\lambda_i) \frac{\nabla_x\E\left[e^{-\tfrac{1}{1+\lambda_i}\mathcal{W}_i(X^{u_i},t)}\Big|X^{u_i}_t=x\right]}{\E\left[e^{-\tfrac{1}{1+\lambda_i}\mathcal{W}_i(X^{u_i},t)}\Big|X^{u_i}_t=x\right]}.
\end{equation}
We use the adjoint method \cite[see Lemma 5]{domingoenrich2025adjoint} to evaluate the conditional expectation \eqref{eq:gamma_socm}\footnote{Note that there exist other methods for computing derivatives of functionals of stochastic processes. We refer the interested reader to  \cite{domingoenrich2024stochastic}.}, giving
\begin{align}
    & \nabla_x \E\left[e^{-\tfrac{1}{1+\lambda_i} \mathcal{W}_i(X^{u_i},t)}\Big|X^{u_i}_t=x\right] 
    = \E\left[\widetilde{a}(t,u_i, X^{u_i})e^{-\tfrac{1}{1+\lambda_i}\mathcal{W}_i(X^{u_i},t)}\Big|X^{u_i}_t=x\right]
\end{align}
where the adjoint state $\widetilde{a}(t,u_i, X^{u_i})$ satisfies the ordinary differential equation (ODE)
\highprio{JB: shouldn't it be a VJP instead of a JVP?}
\begin{align}
    \frac{\dd}{\dd s} \widetilde{a}(s, u_i, X_s^{u_i}) = - \Big[ 
    (\nabla  (b(X_s^{u_i},s) & + \sigma u_i(X_s^{u_i}, s))^{\top} \widetilde{a}(u_i, X_s^{u_i}, s) \\ & + \tfrac{1}{1+\lambda_i}\nabla (f(X_s^{u_i},s) +\tfrac{1}{2} \|u_i(X_s^{u_i},s)\|^2 )
    \Big]
\end{align}
with $\widetilde{a}(T,u_i, X_T^{u_i}) = \tfrac{1}{1+\lambda_i}\nabla g(X_T)$. Using the argument from \cite[Theorem 1]{domingoenrich2024stochastic}, replacing the path-wise reparameterization trick with the adjoint method, we arrive at the trust region version of the stochastic optimal control loss given by
\begin{align}
    {\mathcal{L}}_{\mathrm{SOCM}}(u) & = 
    \E\left[\frac{1}{2}\int_0^T \|\tfrac{\lambda_i}{1+\lambda_i}u_i(X^{u_i},s) -  \sigma^{\top} \widetilde{a}(u_i, X_s^{u_i}, s) - u(X^{u_i},s)\|^2 \dd s \frac{\dd \P^{u_{i+1}}}{\dd \P^{u_i}} \right]  + K
\end{align}
for some $K$ independent of $ u$. However, the adjoint state contains the Jacobian  $\nabla u_i$ and the derivative $\nabla \|u_i\|$, which can be expensive in practice. In what follows, we rewrite the objective such that we can get rid of these terms. 

\subsection{Stochastic optimal control matching via lean adjoint method }
\label{app:tr_lean_adjoint_matching}
Starting again from the cross-entropy loss, we now employ the alternative expression for the optimal control as stated in \Cref{it:solution_value_fn_app}. This yields the objective
\begin{align}
    {\mathcal{L}}_{\mathrm{CE}}(u) & =
    \E\left[\frac{1}{2}\int_0^T \|-\sigma^{\top}\nabla\widetilde{V}_{i+1}(X_s,s)  - u(X_s,s)\|^2 \dd s \frac{\dd \P^{u_{i+1}}}{\dd \P} \right],
\end{align}
where the gradient of the smoothed value function is given by
\begin{equation}
    \nabla_x \widetilde{V}_{i+1}(x,t) = - \frac{\nabla_x 
    \E \left[e^{-\beta_{i+1}\mathcal{W}(X_t,0)}|X_t=x\right]}{\E \left[e^{-\beta_{i+1}\mathcal{W}(X_t,0)}|X_t=x\right]}.
\end{equation}
We evaluate the conditional expectation using the adjoint method:
\begin{equation}
\nabla_x 
    \E \left[e^{-\beta_{i+1}\mathcal{W}(X_t,0)}|X_t=x\right] =  
    \E \left[a_{i+1}(X_s, s)e^{-\beta_{i+1}\mathcal{W}(X_t,0)}|X_t=x\right],
\end{equation}
where $a_{i+1}(X_s, s)$ denotes the lean adjoint state~\cite{domingoenrich2025adjoint}, which satisfies the backward differential equation
\begin{equation}
\label{eq:lean_adjoint_ode}
    \frac{\dd}{\dd s} a_{i+1}(X_s, s) = - \left[
    \left(\nabla  b(X_s,s\right)^{\top} a_{i+1}(X_s, s) + \beta_{i+1}\nabla f(X_s,s)
    \right]
\end{equation}
with terminal condition $a_{i+1}(X_T, T) = \beta_{i+1} \nabla g(X_T)$.
Following the derivations in~\cite{domingoenrich2024stochastic}, we arrive at the objective
\begin{equation}
\label{eq: socm lean}
   \mathcal{L}_{\mathrm{SOCM}}(u) =  \E\left[\tfrac{1}{2}\int_0^T  \| \sigma^{\top}a_{i+1}(X_s,s) - u(X_s,s)\|^2 \dd s\frac{\dd {\P}^{u_{i+1}}}{\dd{\mathbb{P}}}(X)\right].
\end{equation}
Finally, performing a change of measure to the previous control $u_i$ gives the expression:
\begin{equation}
   \mathcal{L}_{\mathrm{SOCM}}(u) =  \E\left[\tfrac{1}{2}\int_0^T  \| \sigma^{\top}a_{i+1}(X^{u_i}_s,s) - u(X^{u_i}_s,s)\|^2 \dd s\frac{\dd \P^{u_{i+1}}}{\dd \P^{u_i}}(X^{u_i})\right].
\end{equation}
We remark that the adjoint ODE in~\eqref{eq:lean_adjoint_ode} can be solved as
\begin{equation}
    \label{eq: adjoint zero f}
    a_{i+1}(X_s, s) = \beta_{i+1} \exp\left(\int_s^T \nabla b(X_t,t)^\top \dd t\right)  \nabla g(X_T)
\end{equation}
if $f=0$ and $\nabla b(X_t,t)\nabla b(X_s,s) = \nabla b(X_s,s) \nabla b(X_t,t)$ for all $s,t \in [0,T]$ (i.e., the matrices at different times commute). This allows us to solve the adjoint ODE exactly for our applications of sampling from unnormalized densities; see~\Cref{app:sampling}.

\paragraph{Extensions for diffusion-based sampling.} Consider the case where $f = 0$ and $b(x,t) = b_1(t)x$ with $b_1 \in C([0,T],\R)$, which holds in certain settings for diffusion-based sampling \cite{zhang2021path, vargas2023denoising}. In this case \eqref{eq: adjoint zero f} becomes
\begin{equation}
    a_{i+1}(X_T, s) = \beta_{i+1} \gamma(s)  \nabla g(X_T) \quad \text{with} \quad \gamma(s) \coloneqq \exp\left(\int_s^T b_1(t) \dd t\right).
\end{equation}
The SOCM loss in \eqref{eq: socm lean} therefore reads 
\begin{equation}
   \mathcal{L}_{\mathrm{SOCM}}(u) =  \E_{\P^{u_{i+1}}}\left[\tfrac{1}{2}\int_0^T  \| \beta_{i+1} \gamma(s)\sigma^{\top}  \nabla g(X^{u_{i+1}}_T) - u(X^{u_{i+1}}_s,s)\|^2 \dd s\right].
\end{equation}
From \Cref{prop: tr_soc_optimality app}~\Cref{it:solution_value_fn_app} it directly follows that
\begin{equation}
        \frac{\dd \P^{u_{i+1}}}{\dd \P^{u_i}}(X) = \frac{\dd \P^{u_{i+1}}}{\dd \P} (X)\frac{\dd \P}{\dd \P^{u_i}}(X) 
        \propto e^{-(\beta_{i+1} - \beta_{i}) g(X_T)}
\end{equation}
for $u_0 = \mathbf{0}$. Thus, the SOCM loss can be rewritten as 
\begin{subequations}
\begin{align}
   \mathcal{L}_{\mathrm{SOCM}}(u) 
   & =
   \E_{\P^{u_{i}}}\left[\tfrac{1}{2}\int_0^T  \| \beta_{i+1} \gamma(s)\sigma^{\top}  \nabla g(X^{u_{i}}_T) - u(X^{u_{i}}_s,s)\|^2 \dd s \frac{\dd \P^{u_{i+1}}}{\dd \P^{u_i}}(X^{u_i})\right]
   \\ & \propto
   \E_{\P^{u_{i}}}\left[\tfrac{1}{2}\int_0^T  \| \beta_{i+1} \gamma(s)\sigma^{\top}  \nabla g(X^{u_{i}}_T) - u(X^{u_{i}}_s,s)\|^2 \dd s \ e^{-(\beta_{i+1} - \beta_{i}) g(X^{u_{i}}_T)}\right].
\end{align}
\end{subequations}
Lastly, using that
$
\P^{u_{i}} = \P^{u_{i}}_{\cdot|T} \P^{u_{i}}_{T} =  \P_{\cdot|T} \P^{u_{i}}_{T},
$
we arrive at 
\begin{subequations}
\begin{align}
   \mathcal{L}_{\mathrm{SOCM}}(u) 
   & \propto
   \E_{\P^{u_{i}}_T\P_{\cdot|T}}\left[\tfrac{1}{2}\int_0^T  \| \beta_{i+1} \gamma(s)\sigma^{\top}  \nabla g(X^{u_{i}}_T) - u(X^{u_{i}}_s,s)\|^2 \dd s \ e^{-(\beta_{i+1} - \beta_{i}) g(X^{u_{i}}_T)}\right]
   \\ & =
   \int_0^T\E_{\P^{u_{i}}_T\P_{s|T}}\left[\tfrac{1}{2}  \| \beta_{i+1} \gamma(s)\sigma^{\top}  \nabla g(X^{u_{i}}_T) - u(X^{u_{i}}_s,s)\|^2  \ e^{-(\beta_{i+1} - \beta_{i}) g(X^{u_{i}}_T)}\right] \dd s,
\end{align}
\end{subequations}
where we marginalized out all $t \in [0,T)$ except for $t=s$. Note that for certain $b_1$, $\P_{s|T}$ can be sampled directly, removing the necessity for storing intermediate samples in the buffer.

\section{Trust regions for probability measures}
\label{app:trust_region_densities}

Our goal is to sample from a probability density of the form
\begin{equation}
    p_{\mathrm{target}}(x) = \frac{\rho_\mathrm{target}(x)}{\Z}, \quad \text{with} \quad \Z = \int \rho_\mathrm{target}(x) \dd x,
\end{equation}
where we can evaluate $\rho_\mathrm{target}$ but typically do not have access to samples from $p_{\mathrm{target}}$. To tackle this problem, one can again formulate this problem as a variational problem by minimizing a divergence between some $q$ and the target density $p_{\mathrm{target}}$. We can again incorporate an additional trust region constraint, that is, an upper bound on the change of the variational distribution $q$ within a single update step.  Formally, we are trying to solve the following problem:
\begin{equation}
\label{eq: trust region KL}
    q_{i+1} = \argmin_q \ \KL(q\|p_{\mathrm{target}}) \quad \text{s.t.} \quad \KL(q\|q_{i}) \leq \varepsilon, \quad \int \dd q = 1,
\end{equation}
where $q_{i}$ is the variational distribution from the previous iteration. We again tackle the constrained optimization problem in \eqref{eq: trust region KL} using Lagrangian multipliers. The Lagrangian is given by 
\begin{align}
    \mathcal{L}^{(i)}_{\mathrm{TR}}(q,\lambda,\omega) &= \KL(q\|p_{\mathrm{target}}) + \lambda \left(\KL(q\|q_{i}) - \varepsilon\right) + \omega\left(\int \dd q - 1\right) 
\end{align}
with Lagrangian multipliers $\lambda,\omega$. Taking the functional derivative $\delta \mathcal{L}^{(i)}_{\mathrm{TR}}(q,\lambda,\omega) /\delta q$ and setting it to zero admits a closed-form solution for the optimal density $q_{i+1}$ as the geometric average between the old distribution and the (unnormalized) optimal distribution, that is,\footnote{Note the dependence of $\mathcal{L}^{(i)}_{\mathrm{TR}}$ on $\omega$ vanishes as $q_{i+1}$ satisfies the normalization constraint.}
\begin{equation}
    q_{i+1}(\lambda) = \argmin_q \ \mathcal{L}^{(i)}_{\mathrm{TR}}(q,\lambda) = \frac{q_{i}^{\frac{\lambda}{1+\lambda}} \rho_{\mathrm{target}}^{\frac{1}{1+\lambda}}}{\Z_i(\lambda)}, \quad \text{with}\quad \Z_i(\lambda) = \int\dd  q_{i}^{\frac{\lambda}{1+\lambda}} \rho_{\mathrm{target}}^{\frac{1}{1+\lambda}} .
\end{equation}
Plugging the optimal distribution back into the Lagrangian yields the dual function
\begin{equation}
\mathcal{L}_{\mathrm{Dual}}^{(i)}(\lambda)
 =  \mathcal{L}^{(i)}_{\mathrm{TR}}(q_{i+1}(\lambda),\lambda) = -(1+\lambda)\log \Z_i(\lambda) - \lambda \varepsilon.
\end{equation}
Note that we can use any non-linear optimizer for solving for the optimal Lagrangian multiplier by maximizing the dual function, i.e.,
\begin{equation}
    \lambda_i = \argmax_{\lambda \in \R^+} \ \mathcal{L}_{\mathrm{Dual}}^{(i)}(\lambda).
\end{equation}

\section{Diffusion-based sampling}
\label{app:sampling}

We consider the task of sampling from densities of the form
\begin{equation}
    p_\mathrm{target} = \frac{\rho_\mathrm{target}}{\Z} \quad \text{with} \quad \Z \coloneqq \int_{\R^d} \rho_\mathrm{target}(x) \mathrm d x,
\end{equation}
where $\rho_\mathrm{target} \in C(\mathbb{R}^d, \mathbb{R}_{\ge 0})$ can be evaluated pointwise, but the normalizing constant $\Z$ is typically intractable. 

Here, we approach the sampling problem by using denoising diffusion-based sampling based on the work of \cite{vargas2023denoising} (see \cite{berner2022optimal,richter2023improved} for a generalization). To that end, we consider a controlled ergodic Ornstein-Uhlenbeck (OU) process $X= (X_s)_{s\in[0,T]}$, i.e.,
\begin{align}
\label{eq: dds controlled}
\dd X^u_s = \left(-\zeta(s)X^u_s + u(X^u_s,s)\right)\, \dd s +  \eta \sqrt{2 \zeta(s)} \, {\dd} W_s, && X_0 \sim p_0,
\end{align}
with noise schedule $\zeta \in C([0, T], \R)$, $p_0(x) = \mathcal{N}(0, \eta^2I)$ and corresponding path measure $\P^u$. The target path space measure $\Q$ is induced by an uncontrolled ergodic Ornstein-Uhlenbeck (OU) process, starting from the target $p_{\mathrm{target}}$ and running backward in time, that is,
\begin{align}
\label{eq: dds uncontrolled}
\dd X_s = \zeta(s)X_s\, \dd s +  \eta \sqrt{2 \zeta(s)} \, {\dd} W_s, && X_T \sim p_{\mathrm{target}},
\end{align}
which fulfills $\Q_0 \approx p_0$ for a suitable choice of $\zeta$.
For integration, we follow \cite{vargas2023denoising} and use an exponential integrator. Lastly, it can be shown that the optimal control fulfills
\begin{equation}
      u^*(x, s) = \eta \sqrt{2 \zeta(s)}\nabla_x \log \frac{\Q_s}{\P_s}(x),
\end{equation}
which is later used to analytically compute the optimal control for Gaussian mixture model target densities, see e.g. \cite{vargas2023denoising}. Please note that $\P_s = \mathcal{N}(0, \eta^2I)$ for all $s \in [0,T]$ as the uncontrolled SDE is initialized at its equilibrium distribution.

\subsection{Experimental setup} \label{appendix:sampling_setup}

Here, we provide further details on our experimental setup.

\paragraph{General setting.} 
The codebase used in this work was developed from scratch but is loosely inspired by \href{https://github.com/facebookresearch/SOC-matching}{\texttt{github.com/facebookresearch/SOC-matching}}.
All experiments are conducted using the Jax library \cite{bradbury2021jax} and are run on a single 40GB NVIDIA A40 GPU. Our default experimental setup, unless specified otherwise, is as follows: We use the Adam optimizer \cite{kingma2014adam} with a learning rate of $5 \times 10^{-4}$ and gradient clipping with a value of $1$. We utilized 50 discretization steps using exponential integrators. The control function $u$ is parameterized as a fully-connected 6-layer neural network with 256 neurons and GELU activations \cite{hendrycks2016gaussian}. Time embedding is achieved via Fourier features \cite{tancik2020fourier}. For all experiments, we used a time horizon of $T=1$.

The control is parameterized as 
\begin{equation}
    u^{\theta}(x,t) = f^{\theta}_1(x,t) + f^{\theta}_2(t) \frac{x}{\eta^2},
\end{equation}
and for experiments using Langevin preconditioning (LP), it is parameterized as
\begin{equation}
    u_{\mathrm{LP}}^{\theta}(x,t) = f^{\theta}_1(x,t) + f^{\theta}_2(t) \left(\frac{x}{\eta^2} + \nabla_x \log \rho_{\mathrm{target}}(x)\right),
\end{equation}
where $f_1^\theta$ and $f_2^\theta$ are neural networks parameterized by $\theta$.

For non-trust methods, we train for $60k$ gradient steps with a batch size of 2000, amounting to a total of $120M$ target evaluations. In contrast, trust region methods use a buffer of length 50k refreshed 150 times during training, resulting in a total of $60k \times 150 = 7.5M$ target evaluations. To optimize for the next control $u_{i+1}$, we perform $400$ gradient steps on the replay buffer using randomly sampled batches of size 2000. All experiments use a trust region bound of $\varepsilon = 0.1$. The dual function is optimized using a line search method.

For the \textit{Many Well} target, we set the standard deviation of the prior distribution to 1 and to 2.5 for the Gaussian mixture target. For the randomization of the mixing weights, we uniformly sample positive values that are normalized and rescaled such that the ratio between the maximum mixing weight and the minimum is 3. The diffusivity is scheduled according to $\zeta(t) = (C_{\text{max}} - C_{\text{min}}) \cos^2\left(\frac{t \pi}{2T}\right) + C_{\text{min}}$ with $C_{\text{min}} = 0.01$ and $C_{\text{max}} = 10$.

\paragraph{Evaluation protocol and model selection.}
We follow the evaluation protocol of prior work \cite{blessing2024beyond} and evaluate all performance criteria 100 times during training, using 2000 samples for each evaluation. We apply a running average with a window of 5 evaluations to smooth out short-term fluctuations and obtain more robust results within a single run. We conducted each experiment using four different random seeds and averaged the best results for each run.

\paragraph{Benchmark problem details.} 
The \textit{Many Well} target involves a $d$-dimensional \emph{double well} potential, corresponding to the (unnormalized) density
\[
\rho_{\mathrm{target}}(x) = \exp\left(-\sum_{i=1}^m(x_i^2 - \delta)^2 - \frac{1}{2}\sum_{i=m+1}^{d} x_i^2\right),
\]
with $m \in \mathbb{N}$ representing the number of combined double wells (resulting in $2^m$ modes), and a separation parameter $\delta \in (0, \infty)$ (see also~\cite{wu2020stochastic}). In our experiments, we set $m=5$ leading to $2^m = 32$ modes. The separation parameter is set to $\delta=4$. Since $\rho_{\mathrm{target}}$ factorizes across dimensions, we can compute a reference solution for $\log \mathcal{Z}$ via numerical integration, as described in \cite{midgley2022flow}.

Moreover, we consider a Gaussian mixture model (GMM) target of the form
\begin{equation}
p_{\mathrm{target}}(x) = \sum_{k=1}^K \pi_k \mathcal{N}(x \mid \mu_k, \Sigma_k),
\end{equation}
where $\mu_k \in \mathbb{R}^d$, $\Sigma_k \in \mathbb{R}^{d \times d}$, $\pi_k \geq 0$, and $\sum_{k=1}^K \pi_k = 1$.
To compute the optimal control $u^*$, we exploit the fact that the optimal marginal path measures $\Q_t(x)$ can be derived analytically \cite{noble2024learned},
\begin{equation}
\label{eq: opt gmm path measure}
    \Q_t(x) = \sum_{k=1}^K \pi_k\mathcal{N}\left(x\,\Big|\,\mu_k e^{-\int_t^T \zeta(s)  \dd s} , \Sigma_k e^{-2 \int_t^T \zeta(s) \dd s} + \eta^2\int_t^T 2 \zeta(s) e^{-2 \int_t^s \zeta(u) \dd u} \dd s
    \right)
\end{equation}
and used this for computing the optimal control $u^*$.
Finally, to compute the total variation distance, we leverage the known true mixing weights $\pi_k$ and define the mode partitions $S_k \subset \mathbb{R}^d$ as
\begin{equation}
    S_k = \{x \in \R^d | \argmax_j \ \pi_j \mathcal{N}(x|\mu_j, \Sigma_j)=k\}.
\end{equation}

\subsection{Evaluation criteria}
\label{sec:eval_criteria}
Here, we provide further information on how our evaluation criteria are computed.

\paragraph{Control $L^2$ error.} Assuming access to the optimal control $u^*$, we can compute the $L^2$ error between the optimal and the learned control, i.e.,
\begin{equation}
    \text{control} \ L^2 \ \text{error} \coloneqq \E \left[\frac{1}{2} \int_0^T \|u^* - u\|^2(X^{u^*},s) \dd s\right],
\end{equation}
where $X^{u^*}$ is obtained by simulating the controlled process with $u^*$, and compute the error using a Monte Carlo estimate. Note that this quantity is equivalent to the forward Kullback-Leibler divergence
\begin{equation}
    \KL\left(\Q|\P^u\right) = \E \left[\log \frac{\dd \Q}{\dd \P^u}(X^{u^*})\right].
\end{equation}
Via Girsanov's theorem (see \Cref{sec: useful identities}) we have that
\begin{equation}
\label{eq: rnd Q to u}
    \frac{\dd \Q}{\dd \P^u}(X^{u^*}) = \int_0^T (u^* - u)(X^{u^*},s)\cdot \dd W_s + \frac{1}{2} \int_0^T \|u^* - u\|^2(X^{u^*},s) \dd s.
\end{equation}
The desired equivalence follows from the fact that, under mild regularity assumptions, the stochastic integral in \eqref{eq: rnd Q to u} is a martingale and has vanishing expectation.

\paragraph{Log-normalizing constant.}
By definition, the log-normalizing constant is given by
\begin{equation}
    \Z(X_0) = \E\left[e^{-\mathcal{W}(X,0)}\Big|X_0\right].
\end{equation}
Applying a change of measure to the controlled process yields
\begin{equation}
    \Z(X_0) = \E\left[e^{-\mathcal{W}(X^u,0)}\frac{\dd \P}{\dd \P^u}(X^u)\Big|X_0\right] = \E\left[e^{-\int_0^T \tfrac{1}{2}\|u(X^{u}_s,s)\|^2\dd s - \int_0^T u(X^{u}_s,s) \cdot \dd W_s-\mathcal{W}(X^u,0)}\Big|X_0\right],
\end{equation}
which can be estimated via Monte Carlo using samples from the current control $u$.

\paragraph{Sinkhorn distance.}
We estimate the Sinkhorn distance $\mathcal{W}^2_\gamma$~~\cite{cuturi2013sinkhorn}, an entropy-regularized optimal transport distance, between model and target samples using the JAX-based \texttt{ott} library~~\cite{cuturi2022optimal}.

\paragraph{Total variation distance.}
Inspired by recent work \cite{blessing2024beyond, grenioux2025improving}, we assume access to ground truth mixing weights $\pi_k$, $k \in \{1, \ldots, K\}$, along with a partition $\{S_1, \ldots, S_K\}$ of $\mathbb{R}^d$, where each region $S_k \subset \mathbb{R}^d$ corresponds to the $k$-th mode of the target distribution. We estimate the empirical mixing weights using
\begin{equation}
    \widehat{\pi}_k = \frac{\E\left[\mathbbm{1}_{S_k}(X^u_T)\right]}{\sum_{k'=1}^K \E\left[\mathbbm{1}_{S_{k'}}(X^u_T)\right]}.
\end{equation}
Using these estimates, we compute the total variation distance (TVD) between the empirical and true mode weights as
\begin{equation}
\text{TVD} = \sum_{k=1}^K \left| \pi_k - \widehat{\pi}_k \right|.
\end{equation}
Details on how the ground truth mixing weights and the corresponding mode regions $S_k$ are defined can be found in the descriptions of the target densities.

\subsection{Additional experiments}
\label{app:diffusion_exp}
Here, we provide results for additional numerical experiments. 

\paragraph{Gaussian Mixture 40 (GMM40).} 
We further evaluate the performance of trust-region-based losses by comparing them to existing SOC losses on the well-established GMM40 benchmark \cite{midgley2022flow}. In this task, the target distribution is a Gaussian mixture model with 40 components, where the means are uniformly sampled from the interval $[-40, 40]$, and each component has an initial variance of 1 . We set the prior’s standard deviation to $\eta = 30$. The results, presented in \Cref{fig:gmm40}, show that only two losses, Cross-Entropy (CE) and trust region with log-variance (TR-LV), can consistently learn all 40 modes. Notably, TR-LV achieves this with approximately ten times fewer target evaluations than CE.

\paragraph{Control $L^2$ error vs. target evaluations.}
We extend the results presented in \Cref{sec:application sampling} for the GMM benchmark by providing a detailed analysis of the control $L^2$ error as a function of the number of target evaluations across varying problem dimensionalities $d$. For $d = 2$, all SOC losses achieve low control error. However, at $d = 10$, some methods begin to exhibit elevated control error due to mode collapse. As the dimensionality increases further, only trust-region-based losses consistently maintain low control error. While these methods show partial mode collapse for $d \geq 150$, we anticipate that this issue can be mitigated by refining the control function architecture or by employing larger buffer and batch sizes. Importantly, trust region methods also require significantly fewer target evaluations -- a key advantage in many real-world applications where evaluations are costly.

\paragraph{Influence of trust region bounds.} 
We further investigate the effect of different trust region bound values $\varepsilon$ on the GMM target using TR-LV. The results are presented in \Cref{fig: different epsilon}. The left figure shows that smaller trust region bounds significantly improve performance: $\varepsilon = 0.01$ yields up to an order of magnitude lower control error compared to $\varepsilon = 1$. Additionally, smaller $\varepsilon$ values help stabilize training, as evidenced by the reduced standard deviation across random seeds. In contrast, training with $\varepsilon = 1$ becomes unstable. However, this improved stability comes at the cost of slower convergence -- smaller bounds require more training iterations to effectively anneal from the prior to the target path measure, as illustrated in the middle figure. Finally, the right figure shows that the empirically observed smoothed effective sample size (ESS) aligns well with its Taylor series approximation, $\operatorname{ESS} = \left(\Var\left(\frac{\dd \mathbb{P}^{(i+1)}}{\dd \mathbb{P}^{(i)}}\right)+1\right)^{-1} \approx \frac{1}{2\varepsilon + 1}$ for small values of $\varepsilon$; see see \Cref{app:variance iws} for further details.

\begin{figure*}[t!]
        \centering
        \begin{minipage}[t!]{\textwidth}
            \centering
            \begin{minipage}[t!]{0.245\textwidth}
            \includegraphics[width=\textwidth]{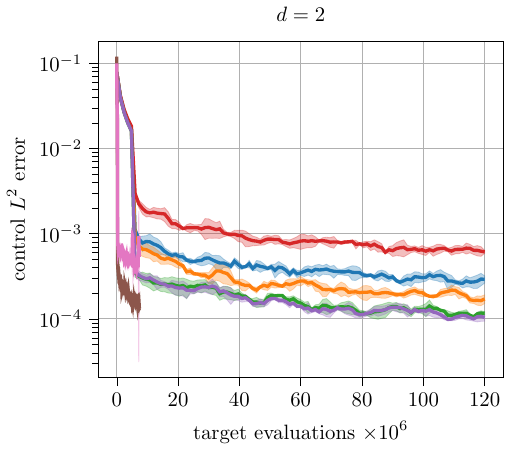}
            \end{minipage}
            \begin{minipage}[t!]{0.245\textwidth}
            \includegraphics[width=\textwidth]{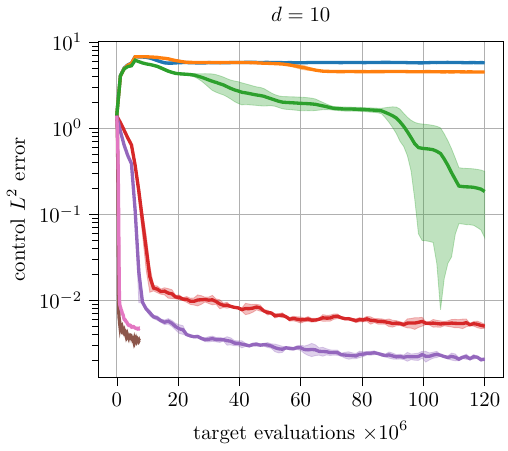}
            \end{minipage}
            \begin{minipage}[t!]{0.245\textwidth}
            \includegraphics[width=\textwidth]{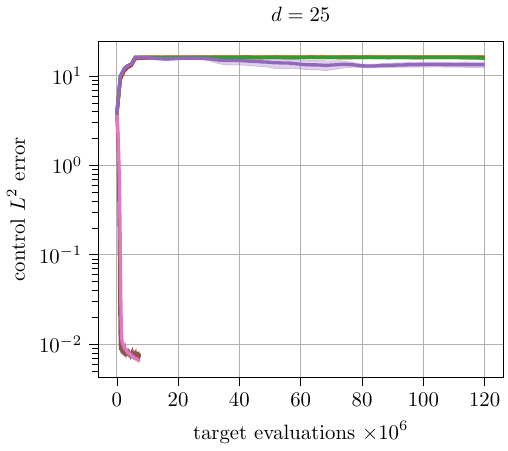}
            \end{minipage}
            \begin{minipage}[t!]{0.245\textwidth}
            \includegraphics[width=\textwidth]{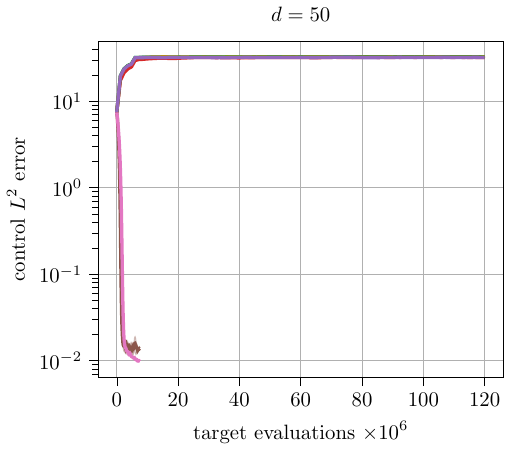}
            \end{minipage}
            \centering
            \begin{minipage}[t!]{0.245\textwidth}
            \includegraphics[width=\textwidth]{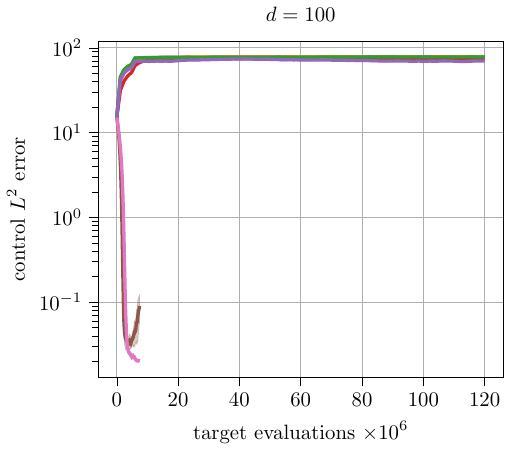}
            \end{minipage}
            \begin{minipage}[t!]{0.245\textwidth}
            \includegraphics[width=\textwidth]{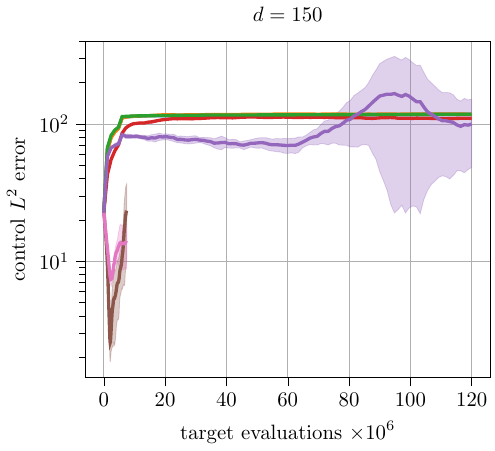}
            \end{minipage}
            \begin{minipage}[t!]{0.245\textwidth}
            \includegraphics[width=\textwidth]{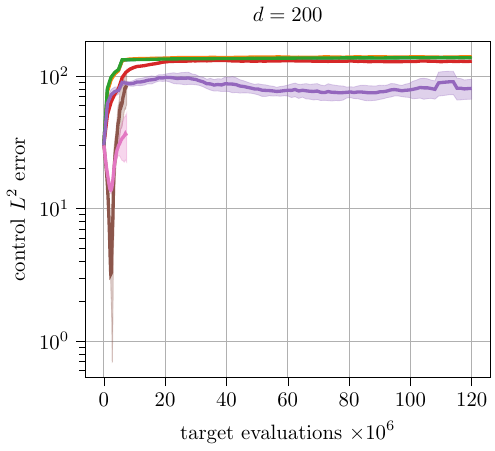}
            \end{minipage}
            \begin{minipage}[t!]{0.245\textwidth}
            \includegraphics[width=\textwidth]{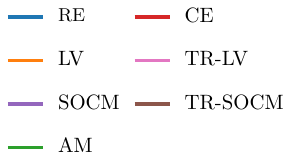}
            \end{minipage}
        \end{minipage}
        \caption[ ]
        {
Control $L^2$ error as a function of the number of target evaluations for the GMM target across varying problem dimensionalities $d$. All results are averaged across four
random seeds.
        }
        \label{fig:gmm_ctrl_err}
    \end{figure*}

\begin{figure*}[t!]
        \centering
        \begin{minipage}[t!]{\textwidth}
            \centering
            \begin{minipage}[t!]{0.245\textwidth}
            \includegraphics[width=\textwidth]{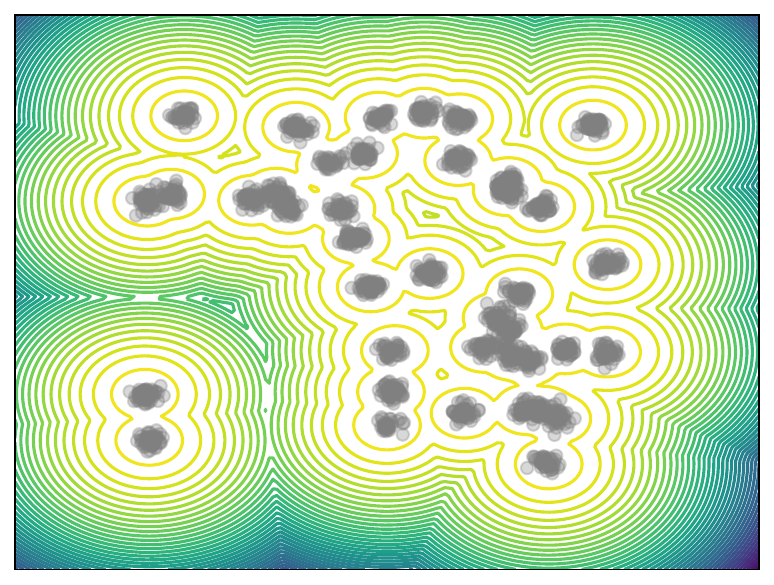}
            \end{minipage}
            \begin{minipage}[t!]{0.245\textwidth}
            \includegraphics[width=\textwidth]{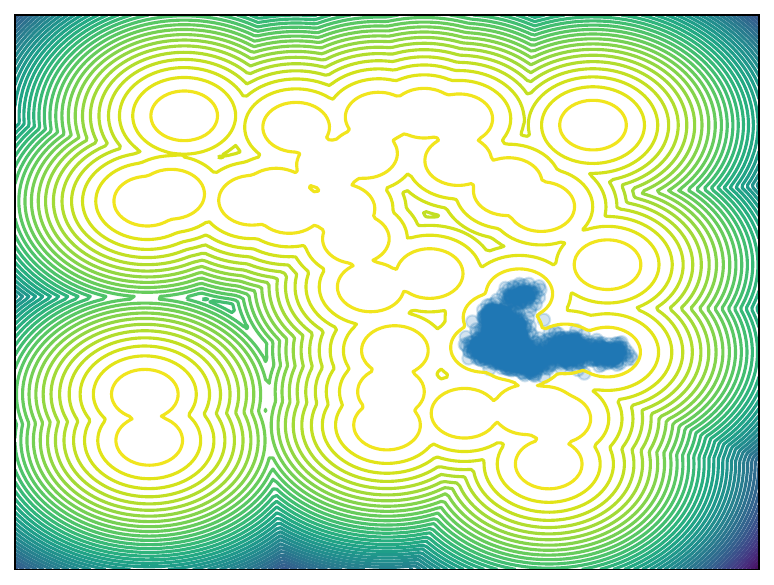}
            \end{minipage}
            \begin{minipage}[t!]{0.245\textwidth}
            \includegraphics[width=\textwidth]{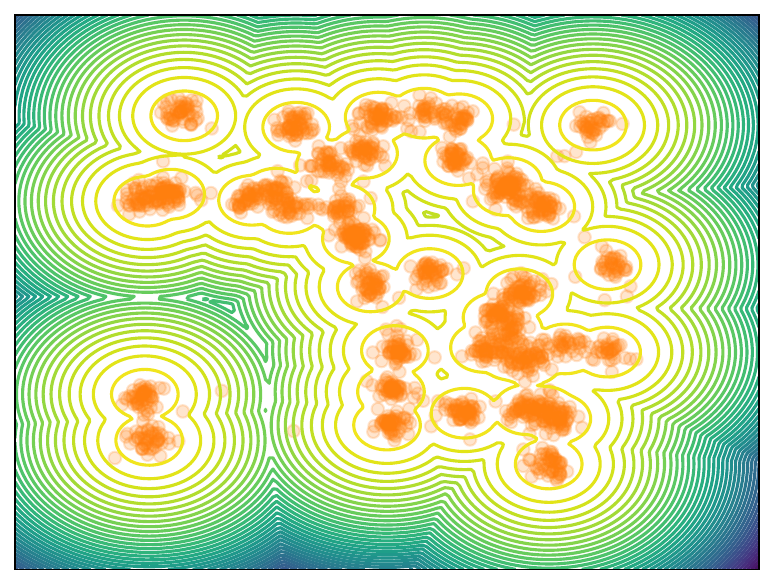}
            \end{minipage}
            \begin{minipage}[t!]{0.245\textwidth}
            \includegraphics[width=\textwidth]{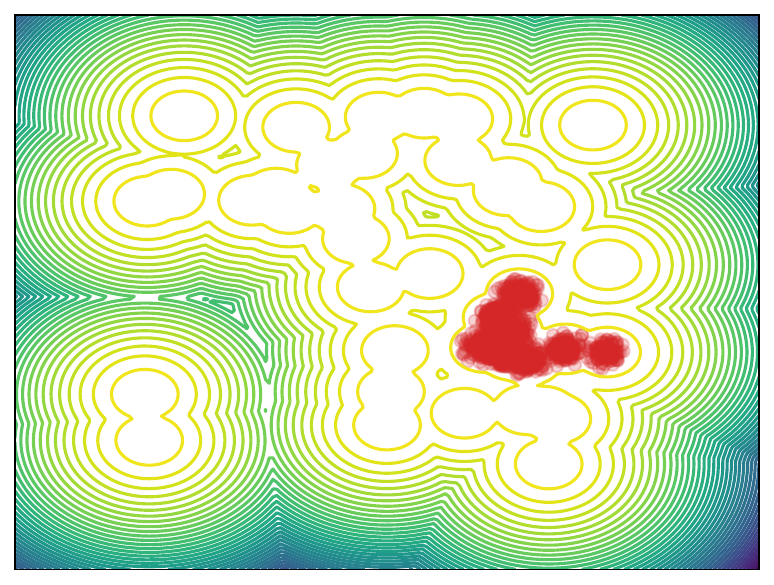}
            \end{minipage}
            \centering
            \begin{minipage}[t!]{0.245\textwidth}
            \includegraphics[width=\textwidth]{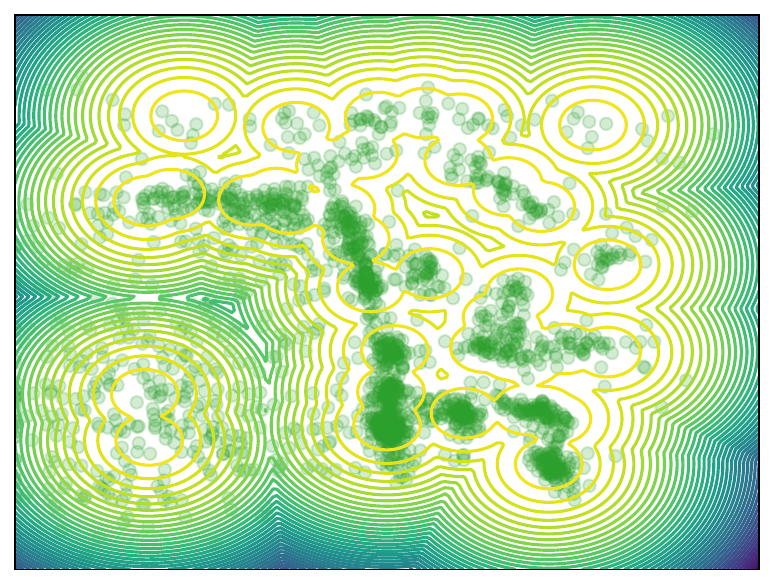}
            \end{minipage}
            \begin{minipage}[t!]{0.245\textwidth}
            \includegraphics[width=\textwidth]{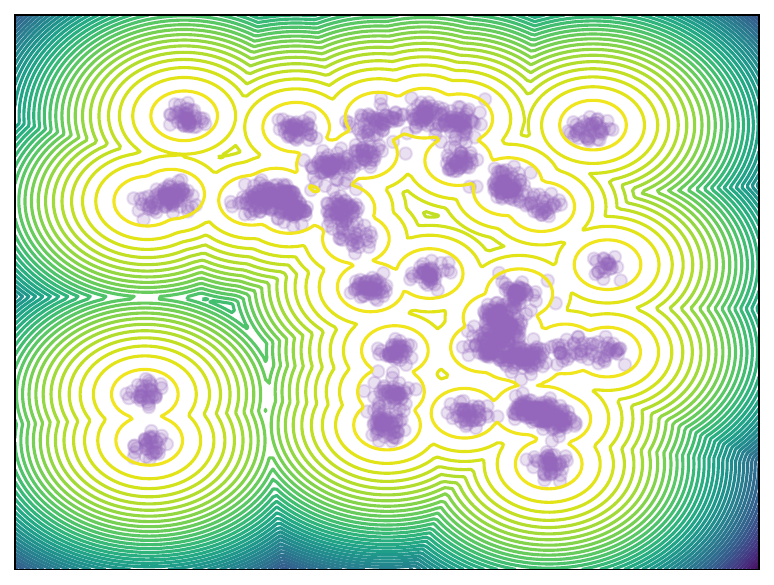}
            \end{minipage}
            \begin{minipage}[t!]{0.245\textwidth}
            \includegraphics[width=\textwidth]{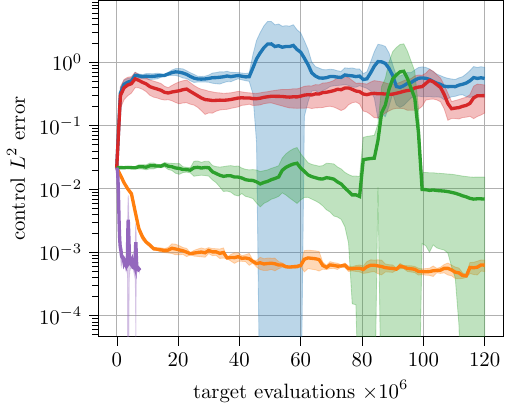}
            \end{minipage}
            \begin{minipage}[t!]{0.245\textwidth}
            \includegraphics[width=\textwidth]{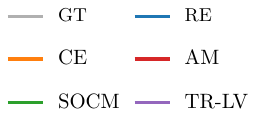}
            \end{minipage}

        \end{minipage}
        \caption[ ]
        {
        Qualitative and quantitative results for the GMM40 target. The qualitative plots demonstrate that only the CE (orange) and TR-LV (purple) losses successfully capture all 40 modes of the ground truth (GT, grey) distribution. This is further supported by the low $L^2$ control error observed for these two methods.
        Results are averaged across four random seeds and are not reported for the log-variance loss due to numerical instabilities.
        }
\label{fig:gmm40}
\end{figure*}

\vspace{1cm}
\begin{figure*}[t!]
        \centering
        \begin{minipage}[t!]{\textwidth}
            \centering
            \begin{subfigure}[t]{0.49\textwidth}
                \centering
                \begin{minipage}[t!]{\textwidth}
                \includegraphics[width=\textwidth]{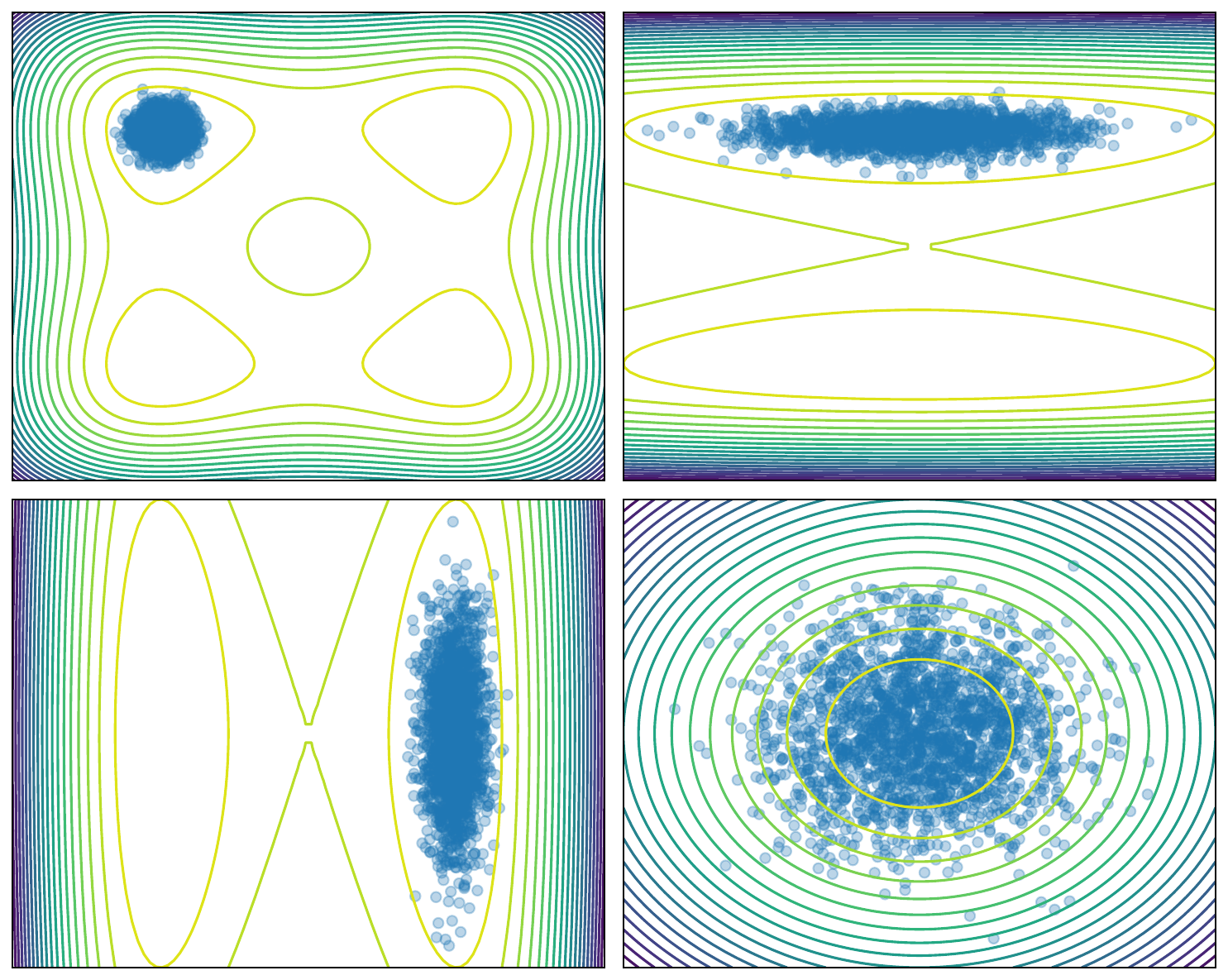}
                \end{minipage}
                \caption{Relative entropy (RE)}
            \end{subfigure}
            \begin{subfigure}[t]{0.49\textwidth}
                \centering
                \begin{minipage}[t!]{\textwidth}
                \includegraphics[width=\textwidth]{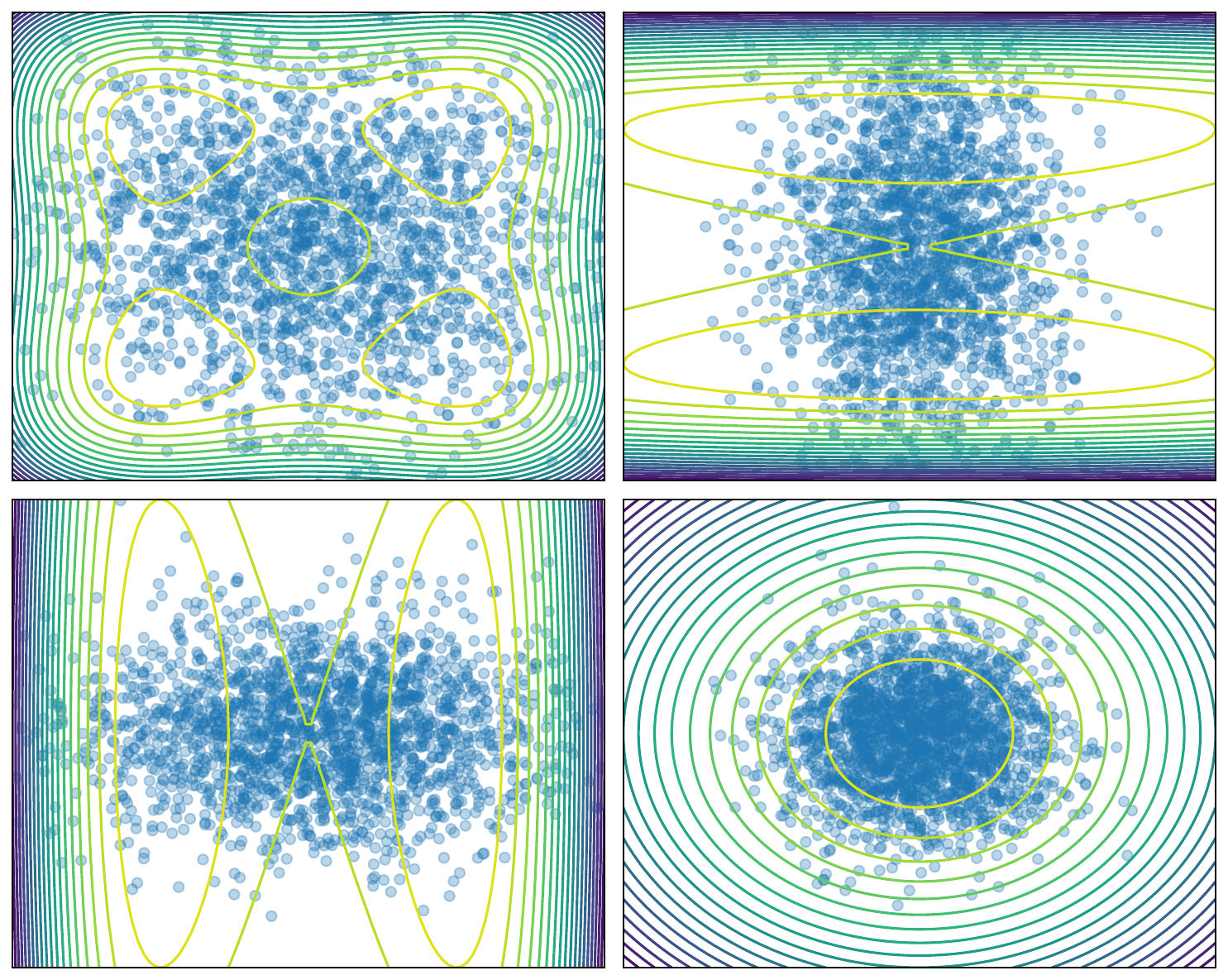}
                \end{minipage}
                \caption{Cross entropy (CE)}
            \end{subfigure}
            \begin{subfigure}[t]{0.49\textwidth}
                \centering
                \begin{minipage}[t!]{\textwidth}
                \includegraphics[width=\textwidth]{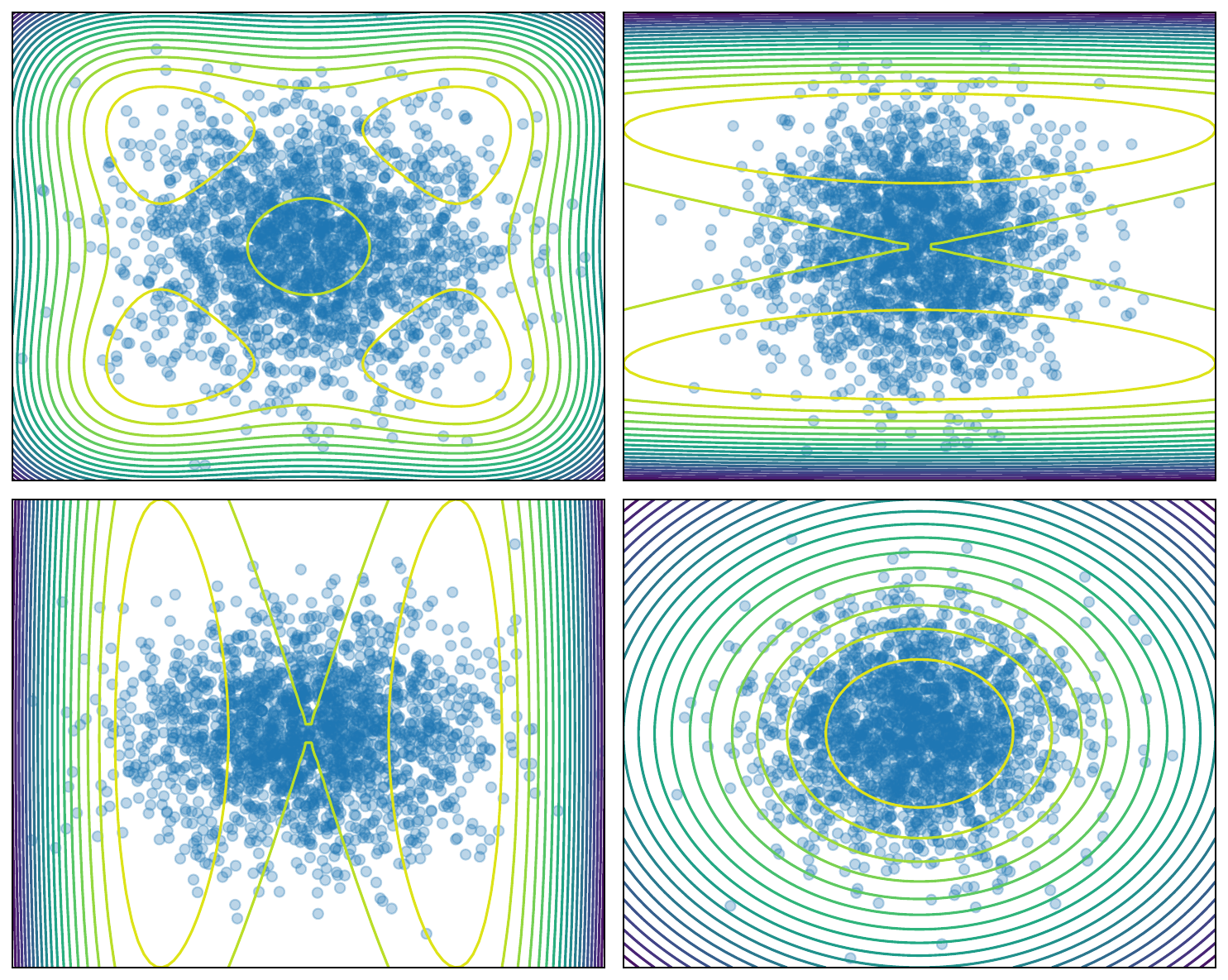}
                \end{minipage}
                \caption{Log-variance (LV)}
            \end{subfigure}
            \begin{subfigure}[t]{0.49\textwidth}
                \centering
                \begin{minipage}[t!]{\textwidth}
                \includegraphics[width=\textwidth]{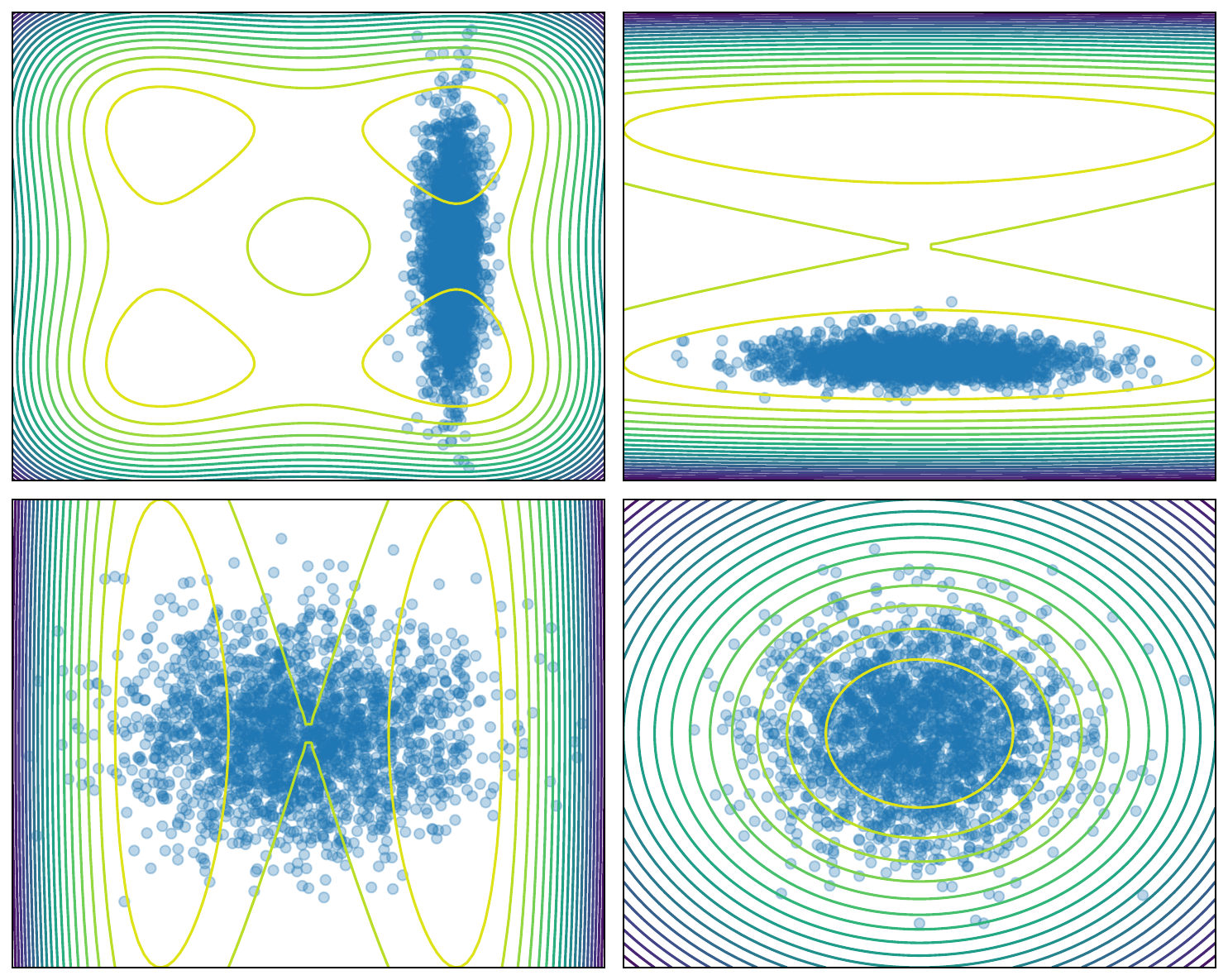}
                \end{minipage}
                \caption{Adjoint matching (AM)}
            \end{subfigure}
            \begin{subfigure}[t]{0.49\textwidth}
                \centering
                \begin{minipage}[t!]{\textwidth}
                \includegraphics[width=\textwidth]{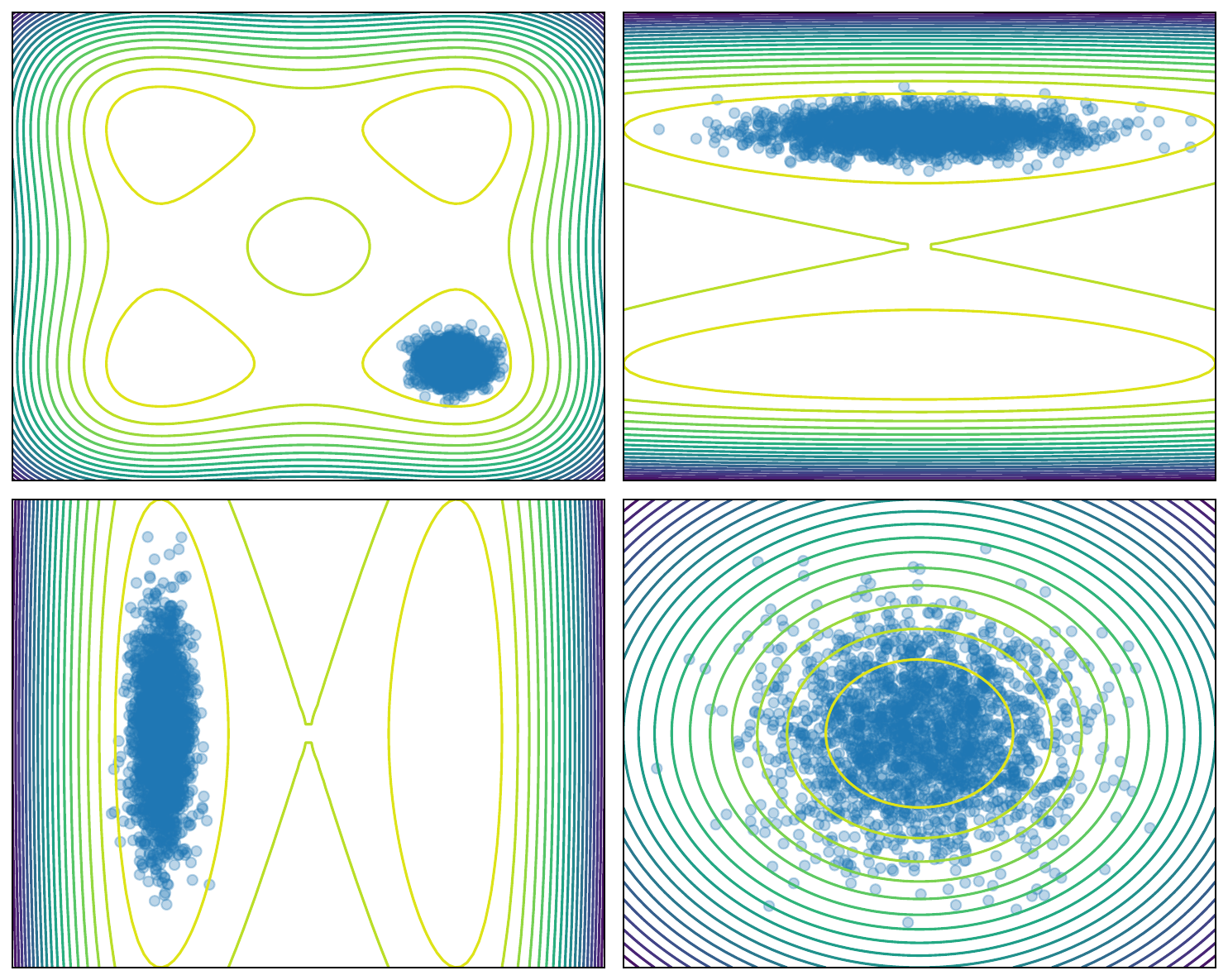}
                \end{minipage}
                \caption{SOC Matching (SOCM)}
            \end{subfigure}
            \begin{subfigure}[t]{0.49\textwidth}
                \centering
                \begin{minipage}[t!]{\textwidth}
                \includegraphics[width=\textwidth]{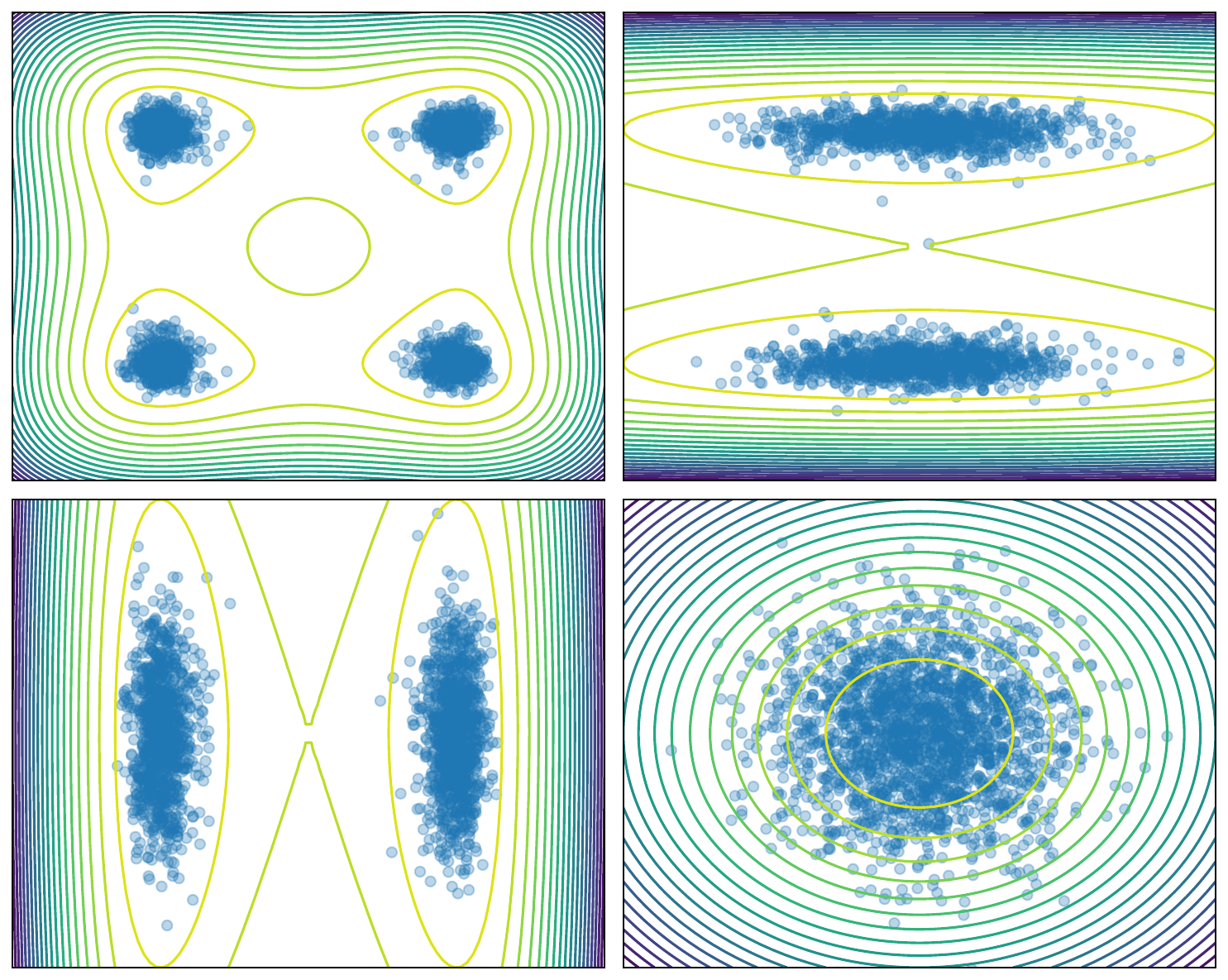}
                \end{minipage}
                \caption{Trust region log-variance (TR-LV)}
            \end{subfigure}
        \end{minipage}
        \vspace{0.1cm}
        \caption[ ]
        {
        Qualitative results for the \textit{Many Well} target with $d=200$. Level plots depict the ground truth density for pairs of marginal distributions, while blue dots represent samples generated by models trained using the respective loss functions (indicated in the sub-captions). Among all methods, only the trust-region-based log-variance loss successfully avoids mode collapse and convergence issues. Interestingly, although the cross-entropy loss achieves the second-lowest estimation error for $\log \mathcal{Z}$ (see \Cref{fig:combined_manywell}), the qualitative results suggest that the model fails to adequately capture the target distribution -- likely due to the high variance of the importance weights. All visualizations are generated using the same random seed for consistency.
        }
\label{fig:manywell}
\end{figure*}

\begin{figure*}[t!]
        \centering
        \begin{minipage}[t!]{\textwidth}
            \begin{minipage}[t!]{0.9\textwidth}
            \centering
            \includegraphics[width=0.8\textwidth]{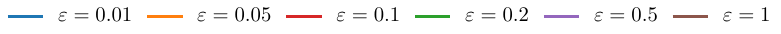}
            \end{minipage}
            \centering
            \begin{minipage}[t!]{0.32\textwidth}
            \includegraphics[width=\textwidth]{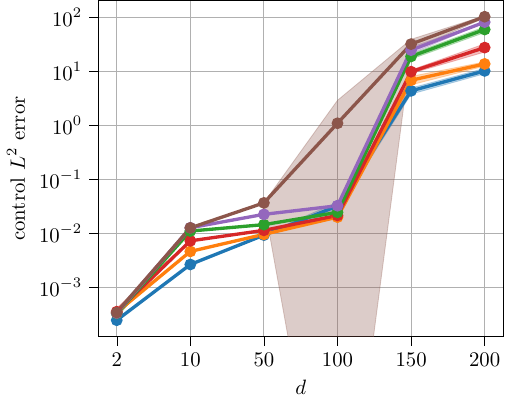}
            \end{minipage}
            \begin{minipage}[t!]{0.32\textwidth}
            \includegraphics[width=\textwidth]{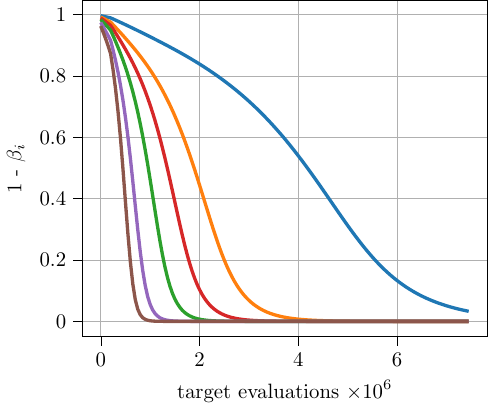}
            \end{minipage}
            \begin{minipage}[t!]{0.32\textwidth}
            \includegraphics[width=\textwidth]{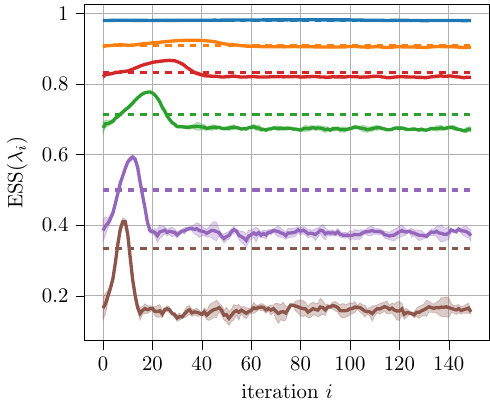}
            \end{minipage}
        \end{minipage}
        \caption[ ]
        {Influence of different trust region bound values $\varepsilon$ on the GMM target for TR-LV. The left figure considers varying problem dimensionalities $d$ whereas the middle and right figure report results for $d=100$.
        The figure on the right shows the empirically observed smoothed effective sample size (ESS) and its approximation via Taylor series approximation, i.e., $\operatorname{ESS} = \left(\Var\left(\frac{\dd \P^{u_{i+1}}}{\dd \P^{u_i}}\right)+1\right)^{-1} \approx \frac{1}{2\varepsilon+1}$, with solid and dashed lines, respectively. All results are averaged across four random seeds.
        }
        \label{fig: different epsilon}
    \end{figure*}

\section{Transition path sampling} \label{appendix:transitionpath}

\subsection{Experimental setup}

We build upon the codebase provided by TPS-DPS~\cite{seong2024transition} (\href{https://github.com/kiyoung98/tps-dps}{\texttt{github.com/kiyoung98/tps-dps}}). Our experimental setups also follow~\cite{seong2024transition} to ensure a fair comparison. The dual function is optimized using Brent's method \cite{brent1971algorithm}.

\textbf{MD simulation setup.} We run molecular dynamics simulation on the OpenMM platform. Both simulations are run at temperature 300K. For Alanine Dipeptide, we use the `amber99sbildn.xml' forcefield with a VVVR integrator to simulate in vaccum. Each timestep is set as 1 femtosecond. Each path sampled is of length 1,000. For Chignolin, we use the `protein.ff14SBonlysc.xml' forcefield with implicit solvant model `gbn2.xml' with a VVVR integrator. Each timestep is set as 1 femtosecond. Each path sampled is of length 5,000.

\textbf{Target hit.} For Alanine Dipeptide, target hit is defined over the two dihedral angles $\phi$ and $\psi$ and a distance radius within 0.75\AA. For Chignolin, a long MD simulation is pre-loaded with Time-lagged independent component analysis (TICA) to select the first two dimensions that capture most variance. The region is then defined over the two dimensions with a radius of 0.75.

\textbf{Training process.} Annealing is applied from 600K to 300K. A replay buffer is used with buffer size 1,000 and 200 for Alanine Dipeptide and Chignolin, respectively, and training over buffer per iteration is 1,000 times.

\textbf{Hyperparameters.} The trust region constraint is set to $\varepsilon = 0.01$ for Alanine Dipeptide and $\varepsilon = 0.2$ for Chignolin. Batch size for both systems is set to 16, Alanine Dipeptide is trained for 2000 iterations, while Chignolin is trained for 50 iterations.

\textbf{Computing resources.} Each experiment is run on a single 80GB NVIDIA H100 GPU.

\subsection{Additional experimental result discussion}
We discuss our results in comparison to~\cite{seong2024transition}. First of all, we evaluate three seed average as we notice the high variance nature of the transition path sampling problem--running several times can have huge variance in results (also evidenced in~\Cref{fig:tps}). We can also observe the trust region constraint helps to stabilize the training significantly and thus have much smaller variance across three runs. Notably, for Alanine Dipeptide, both methods start with zero hitting percentage, while in Chignolin, in the beginning both methods already have some trajectories that hit the target, trust region constraint is already effective in improving the efficiency. We use almost the exact same setup as in~\cite{seong2024transition} with the only difference being the batch size for Chignolin is 16 instead of 4. We do not tune the model as our goal is to show the trust region constraint improves the training stability and thus the efficiency and accuracy in terms of number of energy calls.

\section{Fine-tuning of diffusion models} \label{appendix:finetuning}
We take the adjoint matching (AM) implementation in \href{https://github.com/microsoft/soc-fine-tuning-sd}{\texttt{github.com/microsoft/soc-fine-tuning-sd}} as our baseline, and we modify it to implement TR-SOCM. 

\paragraph{Fine-tuning experimental details.}
We generate images using classifier-free guidance, with guidance scale 7.5. We use 50 inference timesteps to sample the trajectories during fine-tuning, and the evaluation samples are also generated at 50 inference timesteps.

We fine-tune using the default hyperparameters in the repo: we use AdamW, using learning rate \num{3e-6}, beta 1 set to 0.9, beta 2 set to 0.95, and weight decay 0. We use an effective batch size of 500 trajectories and 4 model backpropagations per trajectory. For the TR-SOCM loss, we use a trust-region bound $\varepsilon=0.1$, a buffer size of 100, and 10 passes per buffer. The dual function is optimized using Brent's method \cite{brent1971algorithm}.

We use the 10000 fine-tuning prompts taken from the repository for \cite{xu2023imagereward}, and the 100 validation prompts from the same repository (see \url{https://github.com/THUDM/ImageReward}). The two prompts used in Figure \ref{fig:images_diff_finetuning} are ``masterpiece, best quality, realistic photograph, 8k, high detailed vintage motorcycle parked on a wet cobblestone street at dusk, neon reflections, shallow depth of field'' and 
``close up photo of anthropomorphic fox animal dressed in white shirt, fox animal, glasses''.

\section{Classical SOC problems} 
\label{appendix:classicalSOC}
Here, we consider classical SOC problems, for which the optimal control can be computed analytically. These problems have been widely used in recent studies to compare different loss functions \cite{nuesken2021solving,domingoenrich2024stochastic,domingo2024taxonomy}. Here, we leverage them to showcase that importance sampling works in high dimensions when using trust-region-based losses. To that end, we consider the comparison between the SOCM loss and its trust-region-based counterpart. 

\subsection{Experimental setup} The experimental setup follows the setup used for diffusion-based sampling, as explained in \Cref{appendix:sampling_setup}, including control function architecture, hyperparameter evaluation protocol, and model selection. 

For discretizing the SDE, we leverage the Euler-Maruyama scheme, i.e.,
\begin{equation}
    \widehat{X}_{n+1} = \widehat{X}_n +\left( b + \sigma u \right)(\widehat{X}_n,n \Delta t) \Delta t + \sigma(n) \sqrt{\Delta t}\xi_n, \quad \xi_n \sim \mathcal{N}(0,I).
\end{equation}
Since the considered benchmark problems admit analytical solutions for the optimal control $u^*$, we consider the $L^2$ error between the learned and the optimal control for evaluating the models as explained in \Cref{appendix:sampling_setup}.

\subsection{Benchmark problem details}
We consider two problems taken from \cite{nuesken2021solving}, the \textit{Quadratic Ornstein-Uhlenbeck (OU) easy} and \textit{Quadratic Ornstein-Uhlenbeck (OU) hard}. For convenience, we briefly introduce them again here.

\paragraph{Quadratic Ornstein-Uhlenbeck (OU)} The choices for the functions of the control problem are
\begin{align}
    b(x,t) = A x, \quad f(x,t) = x^{\top} P x, \quad 
    g(x) = x^{\top} Q x, \quad \sigma(t) = \sigma_0,
\end{align}
where $Q$ is a positive definite matrix. Control problems of this form are better known as linear quadratic regulator (LQR) and they admit a closed form solution \cite{van2007stochastic}. The optimal control is given by
\begin{align}
    u^*(x,t) = - 2 \sigma_0^{\top} F(t) x,
\end{align}
where $F(t)$ is the solution of the Ricatti equation
\begin{align}
    \frac{\dd F(t)}{\dd t} + A^{\top} F(t) + F(t) A - 2 F(t) \sigma_0 \sigma_0^{\top} F(t) + P = 0
\end{align}
with the final condition $F(T) = Q$.
Within the Quadratic OU class, we consider two settings:
\begin{itemize}
    \item Easy: We set $A = 0.2 I$, $P = 0.2 I$, $Q = 0.1 I$, $\sigma_0 = I$, $\lambda = 1$, $T=1$, $x_{\mathrm{init}} \sim 0.5 \mathcal{N}(0,I)$.
    \item Hard: We set $A = I$, $P = I$, $Q = 0.5 I$, $\sigma_0 = I$, $\lambda = 1$, $T=1$, $x_{\mathrm{init}} \sim 0.5 \mathcal{N}(0,I)$.
\end{itemize}

\vspace{1cm}
\begin{figure*}[t!]
        \centering
        \begin{minipage}[t!]{\textwidth}
            \centering
            \begin{subfigure}[t]{\textwidth}
                \centering
                \begin{minipage}[t!]{0.245\textwidth}
                \includegraphics[width=\textwidth]{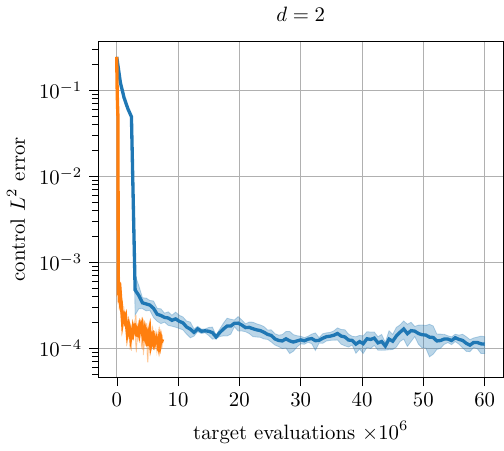}
                \end{minipage}
                \begin{minipage}[t!]{0.245\textwidth}
                \includegraphics[width=\textwidth]{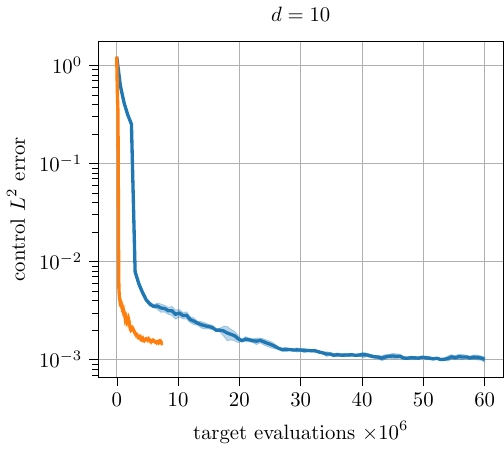}
                \end{minipage}
                \begin{minipage}[t!]{0.245\textwidth}
                \includegraphics[width=\textwidth]{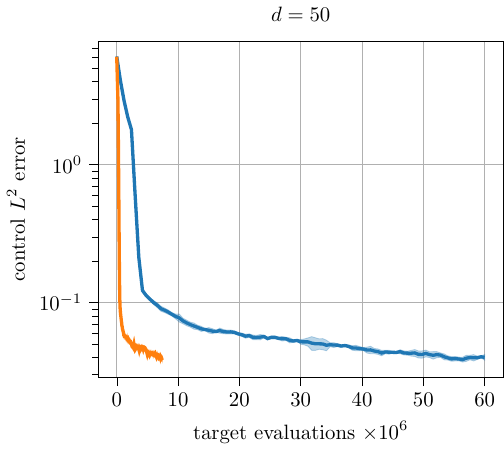}
                \end{minipage}
                \begin{minipage}[t!]{0.245\textwidth}
                \includegraphics[width=\textwidth]{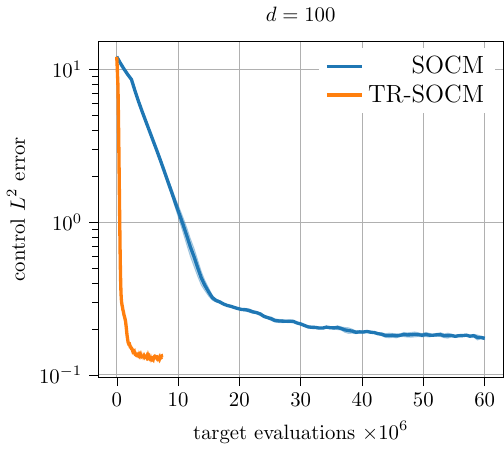}
                \end{minipage}
                \caption{Quadratic Ornstein-Uhlenbeck process (easy)}
            \end{subfigure}
            \centering
            \begin{subfigure}[t]{\textwidth}
                \centering
                \begin{minipage}[t!]{0.245\textwidth}
                \includegraphics[width=\textwidth]{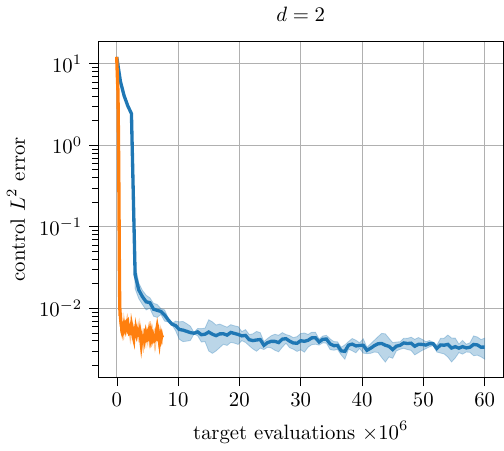}
                \end{minipage}
                \begin{minipage}[t!]{0.245\textwidth}
                \includegraphics[width=\textwidth]{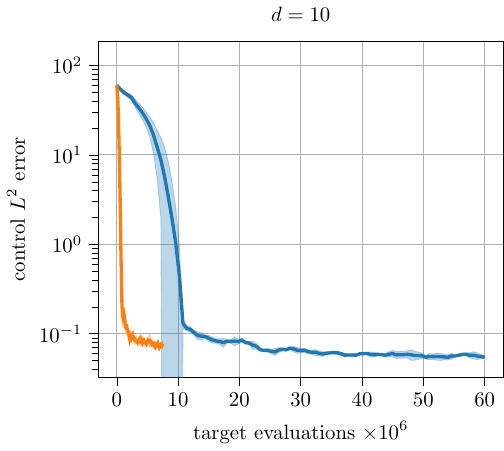}
                \end{minipage}
                \begin{minipage}[t!]{0.245\textwidth}
                \includegraphics[width=\textwidth]{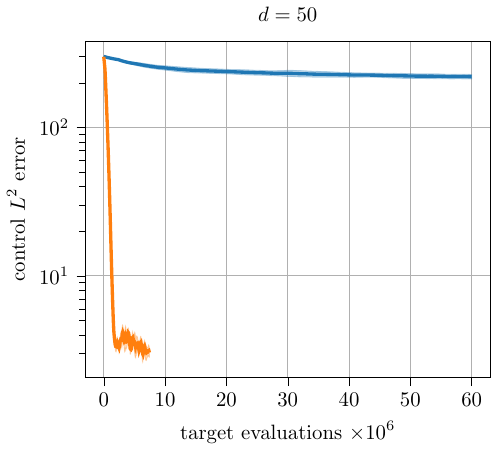}
                \end{minipage}
                \begin{minipage}[t!]{0.245\textwidth}
                \includegraphics[width=\textwidth]{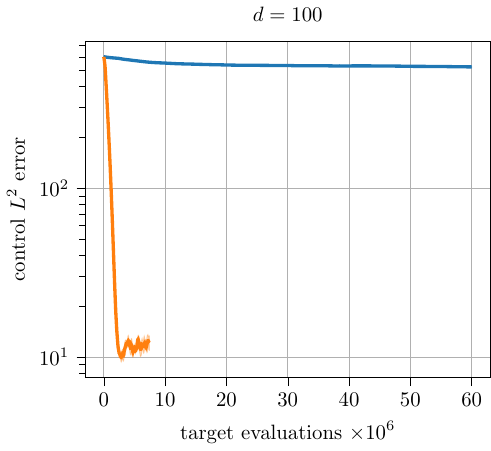}
                \end{minipage}
                \caption{Quadratic Ornstein-Uhlenbeck process (hard)}
            \end{subfigure}
        \end{minipage}
        \vspace{0.1cm}
        \caption[ ]
        {
        Control $L^2$ error as a function of the number of target evaluations for the quadratic OU problem across varying problem dimensionalities $d$. All results are averaged across four random seeds.
        }
\label{fig:quad_ou}
\end{figure*}

\subsection{Results}
We compare the performance of SOCM and its trust-region-based variant (TR-SOCM) on the quadratic Ornstein–Uhlenbeck (OU) problem across varying problem dimensionalities $d$. Both approaches rely on importance sampling, which is known to be challenging in high-dimensional settings. This experiment highlights the role of trust regions in scaling to such regimes. Results are presented in \Cref{fig:quad_ou}.

In low-dimensional settings ($d \leq 10$), both methods perform comparably, although TR-SOCM exhibits significantly better sample efficiency. As the dimensionality increases ($d \geq 50$), the performance of SOCM deteriorates markedly, while TR-SOCM continues to achieve low control error. For the more challenging variant of the quadratic OU problem, SOCM fails to meaningfully improve upon its initialization, whereas TR-SOCM demonstrates consistent error reduction.

These results suggest that trust regions are particularly beneficial in high-dimensional and difficult problem settings, where they provide stability and improved performance.

\end{document}